\documentclass{article}
\usepackage[utf8]{inputenc}
\usepackage{authblk}
\usepackage[margin=1in]{geometry}
\usepackage[titletoc,title]{appendix}
\usepackage[dvipsnames,table,xcdraw]{xcolor}
\usepackage{color}
\usepackage{setspace}
\usepackage{amsmath}

\usepackage{amsmath,amsfonts,amssymb,amsthm,mathtools,bm}
\usepackage{algorithm2e}
\SetKwComment{Comment}{/* }{ */}
\SetKwProg{Init}{Initialize}{}{}
\RestyleAlgo{ruled}

\usepackage{graphicx,float}
\usepackage{subcaption}  
\usepackage{booktabs}
\usepackage{booktabs,multirow}

\usepackage[colorlinks=true, pdfborder={0 0 0}, linkcolor=red]{hyperref}

\usepackage[style=authoryear,date=year,maxcitenames=2,maxbibnames=20,giveninits,url=false,backend=biber,uniquelist=false,uniquename=init,dashed=false]{biblatex}
\bibliography{main}
\usepackage{wasysym}

\allowdisplaybreaks
\newcommand{\Bmid}{\mathrel{\Big|}}

\def\One{\bm{1}}

\def\III{\text I}

\def\A{\mathcal A}
\def\B{\mathcal B}

\def\D{\mathcal D}
\def\F{\mathcal F}

\def\M{\mathcal M}
\def\N{\mathcal N}
\def\P{\mathcal P}

\def\BB{\mathbb B}
\def\EE{\mathbb E}
\def\NN{\mathbb N}
\def\PP{\mathbb P}
\def\RR{\mathbb R}
\def\XX{\mathbb X}
\def\ZZ{\mathbb Z}

\def\bbeta{\boldsymbol \beta}

\def\Diag{\mathrm{Diag}}
\def\gd{\mathrm{gd}}
\def\sgd{\mathrm{sgd}}

\newtheorem{definition}{Definition}
\newtheorem{assumption}{Assumption}
\newtheorem{theorem}{Theorem}

\newtheorem{corollary}{Corollary}
\newtheorem{lemma}{Lemma}
\newtheorem{remark}{Remark}

\begin{document}

\title{\textbf{Asymptotics of Stochastic Gradient Descent with Dropout Regularization in Linear Models}}
\author[1]{Jiaqi Li\thanks{Corresponding author. Email: jqli@uchicago.edu}}
\author[2]{Johannes Schmidt-Hieber}
\author[1]{Wei Biao Wu}

\affil[1]{Department of Statistics, University of Chicago}
\affil[2]{Department of Applied Mathematics, University of Twente}

\maketitle

\begin{abstract}
    This paper proposes an asymptotic theory for online inference of the stochastic gradient descent (SGD) iterates with dropout regularization in linear regression. Specifically, we establish the geometric-moment contraction (GMC) for constant step-size SGD dropout iterates to show the existence of a unique stationary distribution of the dropout recursive function. By the GMC property, we provide quenched central limit theorems (CLT) for the difference between dropout and $\ell^2$-regularized iterates, regardless of initialization. The CLT for the difference between the Ruppert-Polyak averaged SGD (ASGD) with dropout and $\ell^2$-regularized iterates is also presented. Based on these asymptotic normality results, we further introduce an online estimator for the long-run covariance matrix of ASGD dropout to facilitate inference in a recursive manner with efficiency in computational time and memory. The numerical experiments demonstrate that for sufficiently large samples, the proposed confidence intervals for ASGD with dropout nearly achieve the nominal coverage probability. 
\end{abstract}

\vspace{0.5cm}
\noindent\textbf{\textit{Keywords:}} stochastic gradient descent, dropout regularization, $\ell^2$-regularization, online inference, quenched central limit theorems

\newpage

\section{Introduction}

Dropout regularization is a popular method in deep learning (\cite{hinton_improving_2012,krizhevsky_imagenet_2012,srivastava_dropout_2014}). During each training iteration, each hidden unit is randomly masked with probability $1-p$. This ensures that a hidden unit cannot rely on the presence of another hidden unit. Dropout therefore provides an incentive for different units to act more independently and avoids co-adaptation, which means that different units do the same.

There is a rich literature contributing to the theoretical understanding of dropout regularization. As pointed out in \textcite{srivastava_dropout_2014}, the core idea of dropout is to artificially introduce stochasticity to the training process, preventing the model from learning statistical noise in the data. Starting with the connection of dropout 
and $\ell^2$-regularization that appeared already in the original dropout article 
\textcite{srivastava_dropout_2014}, numerous works investigated the statistical properties of dropout by marginalizing the loss functions over dropout noises and linking them with explicit regularization (\cite{arora_dropout_2021,baldi_understanding_2013,cavazza_dropout_2018,mcallester_pac-bayesian_2013,mianjy_dropout_2019,mianjy_implicit_2018,senen-cerda_asymptotic_2022,srivastava_dropout_2014,wager_dropout_2013}). The empirical study in \textcite{wei_implicit_2020} concluded that adding dropout noise to gradient descent also introduces implicit effects, which cannot be characterized by connections between the gradients of marginalized loss functions and explicit regularizers. For the linear regression model and fixed learning rates, \textcite{clara_dropout_2023} proved that the implicit effect of dropout adds noise to the iterates and that for a large class of design matrices, this implicit noise does not vanish in the limit.

Though the convergence theory of dropout in fixed design and full gradients has been widely investigated, an analysis of dropout with random design or sequential observations is still lacking, not to mention online statistical inference. To bridge this gap, we provide a theoretical framework for dropout applied to stochastic gradient descent (SGD). In particular, we establish the geometric-moment contraction (GMC) for the SGD dropout iterates for a range of constant learning rates $\alpha$. We provide two useful and sharp moment inequalities to prove the $q$-th moment convergence of SGD dropout for any $q>1$.

Besides the convergence and error bounds of SGD dropout, statistical inference of SGD-based estimators is also gaining attention (\cite{fang_scalable_2019,fang_online_2019,liang_statistical_2019,su_higrad_2023,zhong_online_2023}). Instead of focusing on point estimators using dropout regularization, we quantify  the uncertainty of the estimates through their confidence intervals or confidence regions (\cite{chen_statistical_2020,zhu_online_2023}). Nevertheless, it is challenging to derive asymptotic normality for SGD dropout or its variants, such as averaged SGD (ASGD) (\cite{ruppert_efficient_1988,polyak_acceleration_1992}). The reason is that the initialization makes the SGD iterates non-stationary. In this paper, we leverage the GMC property of SGD dropout and show quenched central limit theorems (CLT) for both SGD and ASGD dropout estimates. Additionally, we propose an online estimator for the long-run covariance matrix of ASGD dropout to facilitate the online inference.

\textit{Contributions.} This study employs powerful techniques from time series analysis to derive a general asymptotic theory for the SGD iterates with dropout regularization. Specifically, the key contributions can be summarized as follows. 
\begin{itemize}
    \item[(1)] We establish the geometric-moment contraction (GMC) of the non-stationary SGD dropout iterates, whose recursion can be viewed as a vector auto-regressive process (VAR). The possible range of learning rates that ensures GMC can be related to the condition number of the design matrix with dropout.
    \item[(2)] The GMC property guarantees the existence of a unique stationary distribution of the SGD iterates with dropout, and leads to the $L^q$-convergence, the asymptotic normality, and the Gaussian approximation rate of the SGD dropout estimates and their Ruppert-Polyak averaged version.
    \item[(3)] We derive a new moment inequality in Lemma~\ref{lemma_moment_inequality}, proving that for any two random vectors $\bm{x},\bm{y}$ of the same length, the $q$-th moment $\EE\|\bm{x}+\bm{y}\|_2^q$ can have a sharp bound in terms of $\EE\|\bm{x}\|_2^q$,  $\EE\|\bm{y}\|_2^q$ and $\EE(\bm{x}^{\top}\bm{y})$, without the condition $\EE[\bm{y}\mid\bm{x}]\overset{\mathrm{a.s.}}{=}0$ required in previous results (\cite{rio_moment_2009}). The derived moment inequality is also applicable to many other $L^q$-convergence problems in machine learning.
    \item[(4)] An online statistical inference method is introduced to construct joint confidence intervals for averaged SGD dropout iterates. The coverage probability is shown to be asymptotically accurate in theory and simulation studies.
\end{itemize} 

The rest of the paper is organized as follows. We introduce the dropout regularization in gradient descent in Section~\ref{sec_dropout}. Followed by Section~\ref{sec_gd}, we establish the geometric-moment contraction for dropout in gradient descent and provide the asymptotic normality. In Section~\ref{sec_sgd}, we generalize the theory to stochastic gradient descent. In Section~\ref{sec_online_inference}, we provide an online inference algorithm for the ASGD dropout with theoretical guarantees. Finally, we present simulation studies in Section~\ref{sec_simulation}. All the technical proofs are postponed to the Appendix.

\subsection{Background}
\textbf{Dropout regularization.} After its introduction by \textcite{hinton_improving_2012,srivastava_dropout_2014}, dropout regularization was found to be closely related to $\ell^2$-regularization in linear regression and generalized linear models. See also \textcite{baldi_understanding_2013,mcallester_pac-bayesian_2013}. \textcite{wager_dropout_2013} extended this connection to more general injected forms of noise, showing that dropout induces an $\ell^2$-penalty after rescaling the data by the estimated inverse diagonal Fisher information. In neural networks with a single hidden layer, dropout noise marginalization leads to a nuclear norm regularization, as studied in matrix factorization (\cite{cavazza_dropout_2018}), linear neural networks (\cite{mianjy_implicit_2018}), deep linear neural networks (\cite{mianjy_dropout_2019}) and shallow ReLU-activated networks (\cite{arora_dropout_2021}). Moreover, \textcite{gal_theoretically_2016} showed that dropout can be interpreted as a variational approximation to the posterior of a Bayesian neural network. \textcite{gal_dropout_2016} applied this new variational inference based dropout technique in recurrent neural networks (RNN) and long-short term memory (LSTM) models. Additional research has explored the impact of dropout on convolutional neural networks (\cite{wu_towards_2015}) and generalization properties via Rademacher complexity bounds (\cite{arora_dropout_2021,gao_dropout_2016,wan_regularization_2013,zhai_adaptive_2018}). 
Dropout has been successfully applied in various domains, including image classification (\cite{krizhevsky_imagenet_2012}), handwriting recognition (\cite{pham_autodropout_2021}) and heart sound classification (\cite{kay_dropconnected_2016}).

\noindent\textbf{Stochastic gradient descent.} 
To learn from huge datasets, stochastic gradient descent (SGD) (\cite{robbins_stochastic_1951,kiefer_stochastic_1952}) is a computationally attractive variant of the gradient descent method. While dropout and SGD have been studied separately, only little theory has been developed so far for SGD training with dropout regularization. \textcite{mianjy_convergence_2020} showed the necessary number of SGD iterations to achieve suboptimality in ReLU shallow neural networks for classification tasks, which is independent of the dropout probability due to a strict condition on data structures. \textcite{senen-cerda_almost_2023} extended this to more generic results without assuming any specific data structures, focusing instead on reaching stationarity in non-convex functions using dropout-like SGD. Furthermore, \textcite{senen-cerda_asymptotic_2022} analyzed the gradient flow of dropout in shallow linear networks and studied the asymptotic convergence rate of dropout by marginalizing the dropout noise in a shallow network. However, a theoretical convergence analysis or inference theory of SGD dropout iterates without marginalization has not been explored yet in the literature.

\subsection{Notation}
We denote column vectors in $\RR^d$ by lowercase bold letters, that is, $\bm{x}:=(x_1, \ldots, x_d)^{\top}$ and write $\|\bm{x}\|_2:=(\bm{x}^{\top} \bm{x})^{1/2}$ for the Euclidean norm. The expectation and covariance of random vectors are respectively denoted by $\EE[\cdot]$ and $\mathrm{Cov}(\cdot)$. For two positive number sequences $(a_n)$ and $(b_n)$, we say $a_n=O(b_n)$ (resp. $a_n\asymp b_n$) if there exists $c>0$ such that $a_n/b_n\le c$ (resp. $1/c\le a_n/b_n\le c$) for all large $n$, and 
say $a_n=o(b_n)$ if $a_n/b_n\rightarrow0$ as $n\rightarrow\infty$. Let $(x_n)$ and $(y_n)$ be two sequences of random variables. Write $x_n=O_{\PP}(y_n)$ if for $\forall \epsilon>0$, there exists $c>0$ such that $\PP(|x_n/y_n|\le c)>1-\epsilon$ for all large $n$, and say $x_n=o_{\PP}(y_n)$ if $x_n/y_n\rightarrow 0$ in probability as $n\rightarrow\infty$.

We denote matrices by uppercase letters. The $d \times d$ identity matrix is symbolized by $I_d$. Given matrices $A$ and $B$ of compatible dimension, their matrix product is denoted by juxtaposition. Write $A^{\top}$ for the transpose of $A$ and define $\mathbb{A}:=A^{\top} A$. When $A$ and $B$ are of the same dimension, the Hadamard product $A \odot B$ is given by element-wise multiplication $(A \odot B)_{i j}=A_{i j} B_{i j}$. For any $A \in \RR^{d \times d}$, let $\Diag(A):=I_d \odot A$ denote the diagonal matrix with the same main diagonal as $A$. Given $p \in(0,1)$, define the matrices
$$
\begin{aligned}
\overline{A} & :=A-\Diag(A), \\
A_p & :=p A+(1-p) \Diag(A) .
\end{aligned}
$$
In particular, $A_p=p \overline{A}+\Diag(A)$, so $A_p$ results from re-scaling the off-diagonal entries of $A$ by $p$.
For a matrix $A$, the operator norm induced by the Euclidean norm $\|\cdot\|_2$ is the spectral norm and will always be written without sub-script, that is, 
 $\|A\|:=\|A\|_{\mathrm{op}}$.

\section{Dropout Regularization}\label{sec_dropout}

The stochasticity of dropout makes it challenging to analyze the asymptotic properties of dropout in stochastic gradient descent. To address the complex stochastic structure, we investigate in Section~\ref{sec_dropout} the dropout regularization in gradient descent, and then generalize it to stochastic gradient descent in Section~\ref{subsec_dropout_sgd}.

Consider a linear regression model with fixed design matrix $X \in \RR^{n \times d}$ and outcome $\bm{y} \in \RR^n$, that is,
\begin{equation}
    \label{eq_model_gd}
    \bm{y}=X \bbeta^*+\bm{\epsilon},
\end{equation}
with unknown regression vector $\bbeta^* \in \RR^d$, and random noise $\bm{\epsilon} \in \RR^n$. The task is to recover $\bbeta^*$ from the observed data $(\bm{y},X)$. 
% Here in model (\ref{eq_model_gd}), the noise vector $\bm{\epsilon}$ is assumed to admit a density with respect to the Lebesgue measure on $\RR^n$ {\JL [maybe not required?]}. 
Moreover, we suppose that $\EE[\bm{\epsilon}]=\bm{0}$ and $\mathrm{Cov}(\bm{\epsilon})=I_n$. We highlight that the noise distribution of $\bm{\epsilon}$ is often explicitly modeled as multivariate normal, but this is not necessary for this analysis. We also assume that the design matrix $X$ has no zero columns. Because of that we also say that model (\ref{eq_model_gd}) is in reduced form. We can always bring the model into reduced form, since zero columns and the corresponding regression coefficients have no effect on the outcome $\bm{y}$ and can thus be eliminated from the model. 

We consider the least-squares criterion $\tfrac 12 \|\bm{y}-X\bbeta\|_2^2$ for the estimation of $\bbeta^*.$ For the minimization, we adopt a constant learning-rate gradient descent algorithm with random dropouts in each iteration. Following the seminal work on dropout by \textcite{srivastava_dropout_2014}, we call a $d\times d$ random diagonal matrix $D$ a $p$-dropout matrix if its diagonal entries satisfy $D_{ii} \overset{\mathrm{i.i.d.}}{\sim} \mathrm{Bernoulli}(p)$, with some retaining probability $p \in(0,1)$. On average, $D$ has $p d$ diagonal entries equal to 1 and $(1-p) d$ diagonal entries equal to 0. For simplicity, the dependence of $D$ on $p$ will only be stated if unclear from the context. For a sequence of independent and identically distributed (i.i.d.) dropout matrices $D_k,$ $k=1,2,\ldots,$ and some constant learning rate $\alpha>0$, the $k$-th step gradient descent iterate with dropout takes the form
\begin{align}
    \label{eq_dropout_gd_origin}
    \tilde\bbeta_k(\alpha) & = \tilde\bbeta_{k-1}(\alpha) - \alpha\nabla_{\tilde\bbeta_{k-1}}\frac{1}{2}\big\|\bm{y} - XD_k\tilde\bbeta_{k-1}(\alpha)\big\|_2^2  = \tilde\bbeta_{k-1}(\alpha)+\alpha D_k X^{\top} \big(\bm{y}-X D_k \tilde\bbeta_{k-1}(\alpha)\big).
\end{align}
When there is no ambiguity, we omit the dependence on $\alpha$, writing $\tilde\bbeta_k$ instead of $\tilde\bbeta_k(\alpha)$.

Marginalizing the noise, when the iteration number $k$ grows to infinity, the recursion in (\ref{eq_dropout_gd_origin}) shall eventually minimize the $\ell^2$-regularized least-squares loss by solving
\begin{equation}
    \label{eq_goal_gd}
    \tilde\bbeta:=\arg\min_{\bbeta\in\RR^d} \ \EE\Big[\frac{1}{2}\big\|\bm{y} - XD\bbeta\big\|_2^2\Bmid \bm{y},X\Big],
\end{equation}
where the expectation is taken only over the stochasticity in the dropout matrix $D$. Thus, the randomness in $\tilde\bbeta$ comes from the random noise $\bm{\epsilon}$ in (\ref{eq_model_gd}). In fact, the random vector $\tilde\bbeta$ has a closed form expression. To see this, we denote the Gram matrix by $\XX=X^{\top}X$ and recall 
\begin{equation}
    \label{eq_XX_p}
    \overline{\XX}=\XX-\Diag(\XX), \quad \XX_p=p\XX+(1-p)\Diag(\XX).
\end{equation}
Note that $D^2=D$, $\Diag(\XX)=\XX_p-p \overline{\XX}$, and that diagonal matrices always commute. Since the fixed design matrix $X$ is assumed to be in reduced form with $\min_iX_{ii}>0$, one can show that solving the gradient for the minimizer $\tilde\bbeta$ in (\ref{eq_goal_gd}) (\cite{srivastava_dropout_2014, clara_dropout_2023}) leads to the closed form expression
\begin{equation}
    \label{eq_l2_closed_form_gd}
    \tilde\bbeta= p\Big(p^2\XX+p(1-p)\mathrm{Diag}(\XX)\Big)^{-1}X^{\top}\bm{y}= \XX_p^{-1}X^{\top}\bm{y}.
\end{equation}
If the columns of $X$ are orthogonal, then $\XX$ is a diagonal matrix, $\XX_p=\XX$ and $\tilde\bbeta$ coincides with the classical least-squares estimator $\XX^{-1}X^{\top}\bm{y}$. We refer to Section~\ref{subsec_dropout_sgd} for a counterpart of $\tilde\bbeta$ using stochastic gradient.

A crucial argument in the analysis of the dropout iterate $\tilde\bbeta_k$ is to rewrite the dropout update formula as
\begin{align}
    \label{eq_dropout_gd}
    \begin{split}
    \tilde\bbeta_k-\tilde\bbeta & = \underbrace{(I_d-\alpha D_k \XX D_k)}_{=:A_k(\alpha)} (\tilde\bbeta_{k-1}-\tilde\bbeta  )+\underbrace{\alpha D_k \overline{\XX} (p I_d-D_k) \tilde\bbeta}_{=:\bm{b}_k(\alpha)}.
    \end{split}
\end{align}
For a derivation of \eqref{eq_dropout_gd}, see Section 4.1 in \textcite{clara_dropout_2023}. Throughout the rest of this paper, we will exchangeably write $A_k=A_k(\alpha)$ and $\bm{b}_k=\bm{b}_k(\alpha)$ when no confusion should be caused. 

\section{Asymptotic Properties of Dropout in GD}\label{sec_gd}

To study the asymptotic properties of gradient descent with dropout, we first establish the geometric-moment contraction for the GD dropout sequence. Subsequently, we derive the quenched central limit theorems for both iterative dropout estimates and their Ruppert-Polyak averaged variants. Furthermore, we provide the quenched invariance principle for the Ruppert-Polyak averaged dropout with the optimal Gaussian approximation rate.

\subsection{Geometric-Moment Contraction (GMC)}\label{subsec_gmc}

First, we extend the geometric-moment contraction in \textcite{wu_limit_2004} to the cases where the inputs of iterated random functions are i.i.d.\ random matrices.

\begin{definition}[Geometric-moment contraction]\label{def_gmc}
For i.i.d.\ $d\times d$ random matrices $\Psi_i,\Psi_j'$, $i,j\in\ZZ$, consider a stationary causal process
\begin{equation}
    \bm{\theta}_k=g(\Psi_k,\ldots,\Psi_1,\Psi_0,\Psi_{-1},\ldots), \quad k\in\ZZ,
\end{equation}
for a measurable function $g(\cdot)$ such that the $d$-dimensional random vector $\bm{\theta}_k$ has a finite $q$-th moment $\EE\|\bm{\theta}_k\|_2^q<\infty$, for some $q\ge1$. We say that $\bm{\theta}_k$ is \textit{geometric-moment contracting} if there exists some constant $r_q\in(0,1)$ such that
\begin{equation}
    \big(\EE\|\bm{\theta}_k-\bm{\theta}_k'\|_2^q\big)^{1/q} = O(r_q^k), \quad \text{for all} \ k=1,2,\ldots,
\end{equation}
where $\bm{\theta}_k'=g(\Psi_k,\ldots,\Psi_1,\Psi_0',\Psi_{-1}',\ldots)$ is a coupled version of $\bm{\theta}_k$ with $\Psi_i$, $i\le 0$, replaced by i.i.d.\ copies $\Psi_i'$.
\end{definition}

In general, an iterated random function satisfies the geometric-moment contraction property under regularity conditions on convexity and stochastic Lipschitz continuity, see Section~\ref{subsec_gmc_mat} in the Appendix for details. Here, we focus on the contraction property with $\Psi_k = D_k$, the $k$-th dropout matrix. Setting $f_D(\bm{u}):=\bm{u}+\alpha D X^{\top}(\bm{y}-XD\bm{u}),$ we can rewrite the recursion of the dropout gradient descent iterate $\tilde\bbeta_k(\alpha)$ in~(\ref{eq_dropout_gd_origin}) as
\begin{align}
    \label{eq_dropout_function_gd}
    \tilde\bbeta_k(\alpha) 
    = \tilde\bbeta_{k-1}(\alpha)+\alpha D_kX^{\top}\big(\bm{y}-XD_k\tilde\bbeta_{k-1}(\alpha)\big) =: f_{D_k}\big(\tilde\bbeta_{k-1}(\alpha)\big).
\end{align}
We shall show that, under quite general conditions on the constant learning rate $\alpha>0$, this process satisfies the geometric-moment contraction in Definition~\ref{def_gmc}, and converges weakly to a unique stationary distribution $\pi_{\alpha}$ on $\RR^d$, that is, for any continuous function $h\in\mathcal{C}(\RR^d)$ with $\|h\|_{\infty}<\infty$, $\EE\big[h\big(\tilde\bbeta_k(\alpha)\big)\big] \rightarrow \int h(\bm{u})\pi_{\alpha}(d\bm{u})$ as $k\rightarrow\infty.$ We then write $\tilde\bbeta_k(\alpha)\Rightarrow\pi_{\alpha}$. Set
\begin{align}
    \label{eq_gd_r}
    r_{\alpha,q}
    :=\bigg(\sup_{\bm{v}\in\RR^d:\|\bm{v}\|_2=1}\EE\Big\|\big(I_d-\alpha D_1\XX D_1\big)\bm{v}\Big\|_2^q
    \bigg)^{1/q}.
\end{align}
In particular, for $q=2,$ we can rewrite the squared norm and obtain $r_{\alpha,2}^2=\lambda_{\max}(\EE(I_d-\alpha D_1\XX D_1)^2)$ with $\lambda_{\max}(\cdot)$ the largest eigenvalue.

\begin{lemma}[Learning-rate range in GD dropout]\label{lemma_gmc_cond}
    If $q>1$ and $\alpha\|\XX\|<2$, then, $r_{\alpha,q}<1.$
\end{lemma}

We only assume that the design matrix $X$ has no zero column. Still $\XX=X^\top X$ can be singular. Interestingly, dropout ensures that even for singular $\XX,$ we have $r_{\alpha,q}<1.$ Without dropout, $\bm{v}$ could be chosen as an eigenvector of $\XX$ with corresponding eigenvalue zero. Then $\|(I_d-\alpha \XX)\bm{v}\|_2=\|\bm{v}\|_2,$ implying that $\sup_{\bm{v}\in\RR^d:\|\bm{v}\|_2=1}\EE\|\big(I_d-\alpha \XX\big)\bm{v}\|_2^q\geq 1.$

\begin{theorem}[Geometric-moment contraction of GD dropout]\label{thm_gmc_gd}
Let $q>1$ and choose a positive learning rate $\alpha$ satisfying $\alpha\|\XX\|<2$. For two dropout sequences $\tilde\bbeta_k(\alpha), \tilde\bbeta_k'(\alpha),$ $k=0,1,\ldots$, generated by the recursion (\ref{eq_dropout_gd}) with the same dropout matrices but possibly different initial vectors $\tilde\bbeta_0,\,\tilde\bbeta_0'$, we have
\begin{equation}
    \label{eq_gd_gmc}
    \big(\EE\|\tilde\bbeta_k(\alpha) - \tilde\bbeta_k'(\alpha)\|_2^q \big)^{1/q} \le r_{\alpha,q}^k\|\tilde\bbeta_0 - \tilde\bbeta_0'\|_2.
\end{equation} 
Moreover, there exists a unique stationary distribution $\pi_{\alpha}$ which does not depend on the initialization $\tilde\bbeta_0$, such that $\tilde\bbeta_k(\alpha)\Rightarrow\pi_{\alpha}$ as $k\rightarrow\infty$.
\end{theorem}

As mentioned before, $r_{\alpha,2}^2=\lambda_{\max}(\EE(I_d-\alpha D_1\XX D_1)^2)<1.$ A special case of Theorem~\ref{thm_gmc_gd} is thus $\EE\|\tilde\bbeta_k(\alpha) - \tilde\bbeta_k'(\alpha)\|_2^2 \le (\lambda_{\max}(\EE(I_d-\alpha D_1\XX D_1)^2))^{k}\|\tilde\bbeta_0 - \tilde\bbeta_0'\|_2^2.$ Theorem~\ref{thm_gmc_gd} indicates that although the GD dropout sequence $\{\tilde\bbeta_k\}_{k\in\NN}$ is non-stationary due to the initialization, it is asymptotically stationary and approaches the unique stationary distribution $\pi_{\alpha}$ at an exponential rate. Such geometric-moment contraction result is fundamental to establish a central limit theorem for the iterates.

Another consequence of Theorem~\ref{thm_gmc_gd} is that if $\bbeta$ is drawn from the stationary distribution $\pi_\alpha$ and $D$ is an independently sampled dropout matrix, then also $f_D(\bbeta)\sim \pi_\alpha.$ This also means that if the initialization $\tilde\bbeta_0^{\circ}$ is sampled from the stationary distribution $\pi_{\alpha}$, then, the marginal distribution of any of the GD dropout iterates $\tilde\bbeta_k^{\circ}$ will follow this stationary distribution as well. 

We can also define the GD dropout iterates $\tilde\bbeta_k^{\circ}$ for negative integers $k$ by considering i.i.d.\ dropout matrices $D_k$ for all integers $k\in\ZZ$ and observing that the limit
\begin{align}
    \label{eq_function_h_gd}
    \tilde\bbeta_k^{\circ} := \lim_{m\rightarrow\infty}f_{D_k}\circ f_{D_{k-1}} \circ\cdots\circ f_{D_{k-m}}(\bbeta) =:h_{\alpha}(D_k,D_{k-1},\ldots),
\end{align}
exists almost surely and does not depend on $\bbeta.$ Then, $\tilde\bbeta_k^{\circ}=f_{D_k}\big(\tilde\bbeta_{k-1}^{\circ}\big)$ also holds for negative integers and the geometric-moment contraction in Definition~\ref{def_gmc} is satisfied for $\tilde\bbeta_k^{\circ}$, that is,
\begin{equation}
    \label{eq:gmc}
    \Big(\EE \big\| h_{\alpha}(D_k, D_{k-1}, \ldots, D_1, D_0, D_{-1}, \ldots) - 
    h_{\alpha}(D_k, D_{k-1}, \ldots, D_1, D_0', D_{-1}', \ldots)\big\|_2^q\Big)^{1/q} = O(r_{\alpha,q}^k),
\end{equation}
for some $q\ge1,$  $r_{\alpha,q} \in(0,1)$ defined in~(\ref{eq_gd_r}), and i.i.d.\ dropout matrices $D_k, D_\ell'$, $k, \ell \in \mathbb Z$.

\subsection{Iterative Dropout Schemes}\label{subsec_iter_dropout}

Equation (\ref{eq_dropout_gd}) rewrites the GD dropout iterates into $\tilde\bbeta_k(\alpha)-\tilde\bbeta = A_k(\alpha)(\tilde\bbeta_{k-1}(\alpha)-\tilde\bbeta)+\bm{b}_k(\alpha).$ If the initial vector $\tilde\bbeta_0^{\circ}$ is sampled from the stationary distribution $\pi_\alpha,$ we also have
\begin{align}
    \label{eq_iter_dropout_stationary0}
    \tilde\bbeta_k^{\circ}(\alpha)-\tilde\bbeta = A_k(\alpha)(\tilde\bbeta_{k-1}^{\circ}(\alpha)-\tilde\bbeta) + \bm{b}_k(\alpha),
\end{align}
and for any $k=0,1,\ldots$, $\tilde\bbeta_k^{\circ}(\alpha)\sim\pi_{\alpha}$. We can see that $\{\tilde\bbeta_k^{\circ}(\alpha)\}_{k\in\NN}$ is a stationary vector autoregressive process (VAR) with random coefficients. While $(A_k(\alpha),\bm{b}_k(\alpha))$ are i.i.d., $A_k(\alpha)$ and $\bm{b}_k(\alpha)$ are dependent. This poses challenges to prove asymptotic normality of the dropout iterates. An intermediate recursion is obtained by replacing $A_k(\alpha)=I_d - \alpha D_k\XX D_k$ by its expectation $\EE[A_k(\alpha)]=I_d - \alpha p\XX_p$. This gives the recursion
\begin{equation}
    \label{eq_affine_dropout}
    \bbeta_k^{\dagger}(\alpha) - \tilde\bbeta = (I_d-\alpha p\XX_p)\big(\bbeta_{k-1}^{\dagger}(\alpha)-\tilde\bbeta\big) + \bm{b}_k(\alpha),
\end{equation}
with initial vector $\bbeta_0^{\dagger}=\tilde\bbeta_0^{\circ} \sim \pi_{\alpha}$. The proof then derives the asymptotic normality for $\bbeta_k^{\dagger}(\alpha)$, and shows that the difference between $\tilde\bbeta_k(\alpha)$ and $\bbeta_k^{\dagger}(\alpha)$ is negligible, in the sense that for $q\ge2$, $(\EE\|\tilde\bbeta_k(\alpha)-\bbeta_k^{\dagger}(\alpha)\|_2^q)^{1/q}=O\big(\alpha + r_{\alpha,q}^k\|\tilde\bbeta_0-\tilde\bbeta_0^{\circ}\|_2\big),$ where the first part is due to the affine approximation in Lemma~\ref{lemma_affine_approx} and the second part results from the GMC property in Theorem~\ref{thm_gmc_gd}.

\begin{lemma}[Affine approximation]\label{lemma_affine_approx}
If $\alpha \in (0,2/\|\XX\|),$ then the difference sequence $\bm{\delta}_k(\alpha) = \tilde\bbeta_k^{\circ}(\alpha) - \bbeta_k^{\dagger}(\alpha)$ satisfies $\EE[\bm{\delta}_k(\alpha)] = 0$ and for any $q\ge2$, $\max_k \big(\EE\|\bm{\delta}_k(\alpha)\|_2^q\big)^{1/q} = O(\alpha)$.
\end{lemma}

\begin{lemma}[Moment convergence of iterative GD dropout]\label{lemma_q_moment_gd}
    Let $q\ge2$. For the stationary GD dropout sequence $\{\tilde\bbeta_k^{\circ}(\alpha)\}_{k\in\NN}$ defined in (\ref{eq_dropout_gd}), if $\alpha \in (0,2/\|\XX\|)$, we have
    \begin{equation}
        \max_k \big(\EE\|\tilde\bbeta_k^{\circ}(\alpha)-\tilde\bbeta\|_2^q\big)^{1/q} = O(\sqrt{\alpha}).
    \end{equation}
\end{lemma}

\begin{theorem}[Quenched CLT of iterative GD dropout]\label{thm_clt_iter_gd}
    Consider the iterative dropout sequence $\{\tilde\bbeta_k(\alpha)\}_{k\in\NN}$ in (\ref{eq_dropout_gd}) and the $\ell^2$-regularized estimator $\tilde\bbeta$ in (\ref{eq_l2_closed_form_gd}). Assume that the constant learning rate $\alpha$ satisfies $\alpha \in (0,2/\|\XX\|)$, and suppose that for every $l=1,\ldots,d$, there exists $m\neq l$ such that $\XX_{lm}\neq0$. Then, for any $k\in\NN$, we have
    \begin{equation}
        \frac{\tilde\bbeta_k(\alpha)-\tilde\bbeta}{\sqrt{\alpha}} \Rightarrow \N(0,\Xi(0)),\quad \text{as }\alpha\rightarrow0,
    \end{equation}
    where $\Xi(0)=\lim_{\alpha \downarrow 0} \Xi(\alpha)$, and $\Xi(\alpha)\in\RR^{d\times d}$ denotes the covariance matrix of the stationary affine sequence $\{\bbeta_k^{\dagger}(\alpha)-\tilde\bbeta\}_{k\in\NN}$ defined in (\ref{eq_affine_dropout}), that is,
    \begin{equation}
        \label{eq_longrun_iter_gd}
        \Xi(\alpha) := \mathrm{Cov}\Big(\frac{\bbeta_1^{\dagger}(\alpha) - \tilde\bbeta}{\sqrt{\alpha}}\Big) = \frac{\EE[(\bbeta_1^{\dagger}(\alpha) - \tilde\bbeta)(\bbeta_1^{\dagger}(\alpha)-\tilde\bbeta)^{\top}]}{\alpha}.
    \end{equation}
\end{theorem}
One can derive more explicit expressions of $\Xi(0)$. Reshaping a $d\times s$ matrix $U=(\bm{u}_1,\ldots,\bm{u}_s)$ with $d$-dimensional column vectors $\bm{u}_1,\ldots,\bm{u}_s$ into a $ds$-dimensional column vector gives $\mathrm{vec}(U):=(\bm{u}_1^{\top},\ldots,\bm{u}_s^{\top})^{\top}$. Moreover, for any two matrices $A\in\RR^{p\times q}$ and $B\in\RR^{m\times n}$, the Kronecker product $A\otimes B$ is the $pm \times qn$ block matrix, with each block given by $(A\otimes B)_{ij}=A_{ij}B$. Following Theorem 1 in \textcite{pflug_stochastic_1986}, and assuming that $\Xi(\alpha)$ is differentiable with respect to $\alpha$, $\Xi(\alpha)$ becomes the solution of a classical Lyapunov equation $$\Xi(\alpha)(p\XX_p) + (p\XX_p)\Xi(\alpha) = S,$$
that is,
\begin{align}
    \Xi(\alpha) = V_0 + \alpha B_p,
\end{align}
where the $d\times d$ matrices $S, V_0$ and $B_p$ are respectively defined as
\begin{align}
    & S= \frac 1{\alpha^2}\mathrm{Cov}\big(\bm{b}_1(\alpha)\big)
    =\mathrm{Cov}\big(D_1\overline{\XX}(pI_d-D_1)\tilde\bbeta\big), \label{eq_S}\\
    & \mathrm{vec}(V_0) = (I_d\otimes p\XX_p + p\XX_p\otimes I_d)^{-1}\cdot \mathrm{vec}(S), \label{eq_V0} \\
    & \mathrm{vec}(B_p) = (I_d\otimes p\XX_p + p\XX_p\otimes I_d)^{-1}\cdot \mathrm{vec}(p^2\XX_pV_0\XX_p).
\end{align} 
By definition, the matrix $S$ is independent of $\alpha$ and $\overline{\XX}\neq0$ since there exist non-zero diagonal and off-diagonal elements by assumptions in Theorem~\ref{thm_clt_iter_gd}. Let $S_0=\EE[\tilde\bbeta\tilde\bbeta^{\top}].$ By the proof of Theorem~\ref{thm_clt_iter_gd}, we can express $S$ in terms of $p$, $X$ and $\tilde\bbeta$ as follows,
\begin{align}
    S & = p^3(\overline{\XX}S_0\overline{\XX})_p -2p\Big(p\overline{\XX}_p(S_0\overline{\XX})_p + p^2(1-p)\mathrm{Diag}(\overline{\XX}S_0\overline{\XX})\Big) \nonumber \\
    & \quad +p\overline{\XX}_p(S_0)_p\overline{\XX}_p + p^2(1-p)\Big(\mathrm{Diag}(\overline{\XX}(S_0)_p\overline{\XX}) + 2\overline{\XX}_p\mathrm{Diag}(\overline{S_0}\overline{\XX}) + (1-p)\overline{\XX}\odot\overline{S_0}^{\top}\odot\overline{\XX}\Big).
\end{align} 
One can see that, $\Xi(0) = \lim_{\alpha\downarrow0}\Xi(\alpha) = V_0,$ and in particular, for small $p$, $\mathrm{vec}(V_0)$ can be approximated by $(I_d\otimes \XX_p + \XX_p\otimes I_d)^{-1}\cdot \mathrm{vec}(\overline{\XX}_p\EE[\tilde\bbeta\tilde\bbeta^{\top}]_p\overline{\XX}_p)$.

\subsection{Dropout with Ruppert-Polyak Averaging}\label{subsec_ave_dropout}

To reduce the variance of the gradient descent iterates $\tilde\bbeta_k(\alpha)$ introduced by the random dropout matrix $D_k$, we now consider the averaged GD dropout (AGD) iterate
\begin{equation}
    \label{eq_ave_dropout_gd}
    \bar\bbeta_n^{\gd}(\alpha) = \frac{1}{n}\sum_{k=1}^n\tilde\bbeta_k(\alpha),
\end{equation}
following the averaging scheme in \textcite{ruppert_efficient_1988,polyak_acceleration_1992}. We derive the asymptotic normality of $\bar\bbeta_n^{\gd}(\alpha)$ in the following theorem.

\begin{theorem}[Quenched CLT of averaged GD dropout]\label{thm_clt_ave_gd}
For the constant learning rate $\alpha \in (0,2/\|\XX\|)$ and any fixed initial vector $\tilde\bbeta_0$, the averaged GD dropout sequence satisfies
\begin{equation}
    {\sqrt{n}\big(\bar\bbeta_n^{\gd}(\alpha) - \tilde\bbeta\big)}
      \Rightarrow \N\big(0,\Sigma(\alpha)\big), 
\end{equation}
with $\Sigma(\alpha)=\sum_{k=-\infty}^\infty \mathrm{Cov}(\tilde\bbeta_0^{\circ}(\alpha), \tilde\bbeta_k^{\circ}(\alpha))$ the long-run covariance matrix of the stationary process $\tilde\bbeta_k^{\circ}(\alpha)\sim\pi_{\alpha}$.
\end{theorem}

One can choose a few learning rates, say $\alpha_1,\ldots,\alpha_s,$ and run gradient descent for each of these learning rates in parallel by computing $\tilde\bbeta_k(\alpha_1),\ldots,\tilde\bbeta_k(\alpha_s)$ for $k=1,2,\ldots$. An example is federated learning where data are distributed across different clients (\cite{dean_large_2012,karimireddy_scaffold_2020,zinkevich_parallelized_2010}). Additionally, if we consider the unknown parameter $\bbeta^*$ in a general model instead of the linear regression in (\ref{eq_model_gd}), then the (stochastic) gradient descent algorithm with a constant learning rate $\alpha$ may not converge to the global minimum $\bbeta^*$, but oscillate around $\bbeta^*$ with the magnitude $O(\sqrt{\alpha})$ (\cite{dieuleveut_bridging_2020,pflug_stochastic_1986}). In this case, one can adopt extrapolation techniques (\cite{allmeier_computing_2024,huo_bias_2023,yu_analysis_2021}) to reduce the bias in $\bar\bbeta_k^{\gd}(\alpha)$ by using the results from parallel runs for different learning rates. 

\begin{corollary}[Quenched CLT of parallel averaged GD dropout]\label{cor_clt_agd_parallel}
    Let $s\ge1$. Consider constant learning rates $\alpha_1,\ldots,\alpha_s \in (0,2/\|\XX\|).$ Then, for any initial vectors $\tilde\bbeta_0(\alpha_1),\ldots,\tilde\bbeta_0(\alpha_s)$,
    \begin{align}
        \sqrt{n}\cdot\mathrm{vec}\big(\bar\bbeta_n^{\gd}(\alpha_1)-\tilde\bbeta,\ldots,\bar\bbeta_n^{\gd}(\alpha_s)-\tilde\bbeta\big) \Rightarrow \N(0,\Sigma^{\mathrm{vec}}),
    \end{align}
    with $\mathrm{vec}(\bm{u}_1,\ldots,\bm{u}_s)=(\bm{u}_1^{\top},\ldots,\bm{u}_s^{\top})^{\top}\in\RR^{ds}$ for $d$-dimensional vectors $\bm{u}_1,\ldots,\bm{u}_s$, and the long-run covariance matrix $\Sigma^{\mathrm{vec}}=\sum_{k=-\infty}^{\infty}\mathrm{Cov}\big(\mathrm{vec}\big(\tilde\bbeta_0^{\circ}(\alpha_1),\ldots,\tilde\bbeta_0^{\circ}(\alpha_s)\big),\mathrm{vec}\big(\tilde\bbeta_k^{\circ}(\alpha_1),\ldots,\tilde\bbeta_k^{\circ}(\alpha_s)\big)\big)$.
\end{corollary}

\begin{assumption}[Finite moment of gradients with dropout]
    \label{asm_fclt_agd}
    Let $q>2$. Assume that the $q$-th moment of the gradient in (\ref{eq_dropout_gd_origin}) exists at the true parameter $\bbeta^*$ in model (\ref{eq_model_sgd}), that is,
    $$\Big(\EE_{\bm{y}}\EE_D\Big\|\nabla_{\bbeta^*}\frac{1}{2}\big\|\bm{y}-XD\bbeta^*\big\|_2^2\Big\|_2^q\Big)^{1/q} = \Big(\EE\big\|DX^{\top}\big(\bm{y}-X D \bbeta^*\big)\big\|_2^q\Big)^{1/q}<\infty.$$
\end{assumption}
For the central limit theorems (cf. Theorem~\ref{thm_clt_ave_gd}), Assumption~\ref{asm_fclt_agd} is only required to hold for $q=2.$ Since we have already assumed that $\mathrm{Cov}(\bm{\epsilon})=I_n$, we did not additionally impose any moment condition in Theorem~\ref{thm_clt_ave_gd}. However, if one aims for a stronger Gaussian approximation result, such as the rate for the Komlós–Major–Tusnády (KMT) approximation (\cite{komlos_approximation_1975,komlos_approximation_1976,berkes_kmt_2014}), $q>2$ moments are necessary. In the quenched invariance principle below, we show that one can achieve the optimal Gaussian approximation rate $o_{\PP}(n^{1/q})$ for the averaged GD dropout process.

\begin{theorem}[Quenched invariance principle of averaged GD dropout]\label{thm_fclt_agd}
    Suppose that Assumption~\ref{asm_fclt_agd} holds and the constant learning rate satisfies $\alpha\in(0,2/\|\XX\|)$. Define a partial sum process $(S_i^{\circ}(\alpha))_{1\le i\le n}$ with
    \begin{equation}
        \label{eq_agd_partial_sum_stationary}
        S_i^{\circ}(\alpha) = \sum_{k=1}^i(\tilde\bbeta_k^{\circ}(\alpha) - \tilde\bbeta).
    \end{equation}
    Then, there exists a (richer) probability space $(\Omega^{\star},\A^{\star},\PP^{\star})$ on which we can define $d$-dimensional random vectors $\tilde\bbeta_k^{\star}$, the associated partial sum process $S_i^{\star}(\alpha)=\sum_{k=1}^i(\tilde\bbeta_k^{\star}(\alpha) - \tilde\bbeta)$, and a Gaussian process $G_i^{\star}=\sum_{k=1}^i\bm{z}_k^{\star}$, with independent Gaussian random vectors $\bm{z}_k^{\star}\sim \N(0,I_d)$, such that $(S_i^{\circ})_{1\le i\le n}\overset{\D}{=}(S_i^{\star})_{1\le i\le n}$ and
    \begin{equation}
        \max_{1\le i\le n}\big\|S_i^{\star} - \Sigma^{1/2}(\alpha)G_i^{\star}\big\|_2 = o_{\PP}(n^{1/q}), \quad \text{in }(\Omega^{\star},\A^{\star},\PP^{\star}),
    \end{equation}
    where $\Sigma(\alpha)$ is the long-run covariance matrix defined in Theorem \ref{thm_clt_ave_gd}. In addition, this approximation holds for all $(S_i^{\tilde\bbeta_0}(\alpha))_{1\le i\le n}$ given any arbitrary initial vector $\tilde\bbeta_0\in\RR^d$, where
    \begin{equation}
        S_i^{\tilde\bbeta_0}(\alpha) = \sum_{k=1}^i(\tilde\bbeta_k(\alpha) - \tilde\bbeta).
    \end{equation}
\end{theorem}
Theorem~\ref{thm_fclt_agd} shows that one can approximate the averaged GD dropout sequence by Brownian motions. Specifically, for any fixed initial vector $\tilde\bbeta_0\in\RR^d$, the partial sum process converges in the Euclidean norm, uniformly over $u,$
\begin{equation}\label{ip}
    \big\{n^{-1/2}S_{\lfloor n u\rfloor}^{\tilde\bbeta_0}(\alpha), \,\, 0 \le u \le 1\big\}
      \Rightarrow \big\{  \Sigma^{1/2}(\alpha) \BB(u), \,\, 0 \le u \le 1\big\},
\end{equation}
where $\lfloor t\rfloor=\max\{i\in\ZZ: i\le t\}$, and $\BB(u)$ is the standard $d$-dimensional Brownian motion, that is, it can be represented as a $d$-dimensional vector of independent standard Brownian motions. According to the arguments in \textcite{karmakar_optimal_2020}, the KMT approximation rate $o_{\PP}(n^{1/q})$ is optimal for fixed-dimension time series. Since we can view the GD dropout sequence as a VAR(1) process, the approximation rate in Theorem~\ref{thm_fclt_agd} is optimal for the partial sum process $(S_i^{\tilde\bbeta_0}(\alpha))_{1\le i\le n}$.

\section{Generalization to Stochastic Gradient Descent}\label{sec_sgd}

In the previous section, we considered a fixed design matrix and (full) gradient descent with dropout. Computing the gradient over the entire dataset can be computationally expensive, especially with large datasets. We now investigate stochastic gradient descent with dropout regularization.

\subsection{Dropout Regularization in SGD}\label{subsec_dropout_sgd}

Consider i.i.d.\ covariate vectors $\bm{x}_k\in\RR^d$, $k=1,2,\ldots$, from some distribution $\Pi$, and the realizations $y_k|\bm{x}_k$ from a linear regression model
\begin{equation}
    \label{eq_model_sgd}
    y_k = \bm{x}_k^{\top}\bbeta^* + \epsilon_k,
\end{equation}
with unknown regression vector $\bbeta^*\in\RR^d.$ We assume that the model is in reduced form, which here means that $\min_i(\EE[\bm{x}_1\bm{x}_1^{\top}])_{ii}>0$. In addition, we assume that the i.i.d.\ random noises $\epsilon_k$ satisfy $\EE[\epsilon_k]=0$ and $\mathrm{Var}(\epsilon_k)=1$. In this paper, we focus on the classical case where the SGD computes the gradient based on an individual observation $(y_k,\bm{x}_k)$. For constant learning rate $\alpha$ and initialization $\breve\bbeta_0(\alpha),$ the $k$-th step SGD iterate with Bernoulli dropout is
\begin{align}
    \label{eq_dropout_sgd_origin}
    \breve\bbeta_k(\alpha) & = \breve\bbeta_{k-1}(\alpha) - \alpha \nabla_{\breve\bbeta_{k-1}}\frac{1}{2} \big(y_k - \bm{x}_k^{\top}D_k\breve\bbeta_{k-1}(\alpha)\big)^2\nonumber \\
    & = \breve\bbeta_{k-1}(\alpha) + \alpha D_k\bm{x}_k\big(y_k-\bm{x}_k^{\top}D_k\breve\bbeta_{k-1}(\alpha)\big).
\end{align}
This is a sequential estimation, or online learning scheme, as computing $\breve\bbeta_k(\alpha)$ from $\breve\bbeta_{k-1}(\alpha)$ only requires the $k$-th sample $(y_k,\bm{x}_k).$ To study the contraction property of the SGD dropout iterates $\breve\bbeta_k(\alpha)$, we express the recursion (\ref{eq_dropout_sgd_origin}) by an iterated random function $\breve f: \RR^{d\times d}\times\RR\times\RR^d\times\RR^d\mapsto\RR^d$ with $$f_{D,(y,\bm{x})}(\bm{u})=\bm{u} + \alpha D\bm{x}(y-\bm{x}D\bm{u}),$$ 
that is,
\begin{align}
    \label{eq_dropout_function_sgd}
    \breve\bbeta_k(\alpha) & = \breve\bbeta_{k-1}(\alpha) + \alpha D_k\bm{x}_k\big(y_k-\bm{x}_k^{\top}D_k\breve\bbeta_{k-1}(\alpha)\big) \nonumber \\
    & =: \breve f_{D_k,(y_k,\bm{x}_k)}(\breve\bbeta_{k-1}(\alpha)).
\end{align}
We shall show that this iterated random function $\breve f$ is geometrically contracting, and therefore, there exists a unique stationary distribution $\breve\pi_{\alpha}$ such that $\breve\bbeta_k(\alpha)\Rightarrow\breve\pi_{\alpha}$, where $\Rightarrow$ denotes the convergence in distribution.

From now on, let $(y,\bm{x})$ be a sample with the same distribution as $(y_k,\bm{x}_k).$ By marginalizing over all randomness, we can view the SGD dropout in (\ref{eq_dropout_sgd_origin}) as a minimizer of the $\ell^2$-regularized least-squares loss
\begin{equation}
    \label{eq_goal_sgd}
    \breve\bbeta:=\arg\min_{\bbeta\in\RR^d}\EE_{(y,\bm{x})}\EE_D\Big[\frac{1}{2}(y - \bm{x}^{\top}D\bbeta)^2\Big].
\end{equation}
Here, the expectation is take over both the random sample $(y,\bm{x})$ and the dropout matrix $D$. Throughout the rest of the paper, we shall write $\EE[\cdot]=\EE_{(y,\bm{x})}\EE_D[\cdot]$ when no confusion should be caused.

Denote the $d\times d$ Gram matrix by $\XX_k = \bm{x}_k\bm{x}_k^{\top}$, and define
\begin{equation}
    \label{eq_barX_X_kp}
    \overline{\XX}_k = \XX_k - \Diag(\XX_k), \quad \XX_{k,p} = p\XX_k + (1-p)\Diag(\XX_k).
\end{equation}
By Lemma~\ref{lemma_sgd_l2_closed_form} in the Appendix, we have a closed form solution for $\breve\bbeta$ as follows
\begin{equation*}
    \breve\bbeta = p\Big(p^2\EE[\bm{x}_1\bm{x}_1^{\top}] + p(1-p)\mathrm{Diag}(\EE[\bm{x}_1\bm{x}_1^{\top}])\Big)^{-1}\EE[y_1\bm{x}_1] = (\EE[\XX_{1,p}])^{-1}\EE[y_1\bm{x}_1],
\end{equation*}
and thus, we obtain the relationship $\EE[\XX_{1,p}]\breve\bbeta = \EE[y_1\bm{x}_1].$ To study the SGD with dropout, we now focus on the difference process $\{\breve\bbeta_k(\alpha)-\breve\bbeta\}_{k\in\NN}$. As in the case of gradient descent, this process can be written in autoregressive form,
\begin{align}
    \label{eq_dropout_sgd}
    \breve\bbeta_k(\alpha) - \breve\bbeta & = \underbrace{\big(I_d - \alpha D_k\XX_kD_k\big)}_{=: \breve A_k(\alpha)}(\breve\bbeta_{k-1}(\alpha) - \breve\bbeta) + \underbrace{\alpha D_k\bm{x}_k(y_k - \bm{x}_k^{\top}D_k\breve\bbeta)}_{=:\breve{\bm{b}}_k(\alpha)}.
\end{align}

\subsection{GMC of Dropout in SGD}\label{subsec_gmc_sgd}

Establishing the geometric-moment contraction (GMC) property to the stochastic gradient descent iterates with dropout is non-trivial as the randomness of $\breve\bbeta_k(\alpha)$ not only comes from the dropout matrix $D_k$, but also the random design vectors $\bm{x}_k$. Recall that $\XX_{k,p}$ is defined in~(\ref{eq_barX_X_kp}) and by Lemma~\ref{lemma_clara}(ii), $\EE[D\XX_kD]=p\EE[\XX_{k,p}].$ 

\begin{lemma}[Learning-rate range in SGD dropout]\label{lemma_gmc_sgd_cond}
Assume that $\mu_q(\bm{v})=(\EE\|D_k\XX_kD_k\bm{v}\|_2^q)^{1/q}<\infty$ for some $q\ge2$ and some unit vector $\bm{v}\in\RR^d$. If the learning rate $\alpha>0$ satisfies
\begin{align}
    \label{eq_sgd_condition}
    \frac{\alpha(q-1)}{2} \sup_{\bm{v}\in\RR^d,\,\|\bm{v}\|_2=1}\frac{(1+\alpha\mu_q(\bm{v}))^{q-2}\mu_q(\bm{v})^2}{p\bm{v}^{\top}\EE[\XX_{k,p}]\bm{v}} < 1,
\end{align}
then, for a dropout matrix $D,$
$$\sup_{\bm{v}\in\RR^d,\|\bm{v}\|_2=1}\EE\big\|(I_d - \alpha D\XX_kD)\bm{v}\big\|_2^q<1.$$
\end{lemma}

This provides a sufficient condition for the learning rate $\alpha$ which ensures contraction of $I_d - \alpha D\XX_kD$ for moments $q\geq 2.$ This will lead to $L^q$-convergence of the SGD dropout iterates and determines the convergence rate in the Gaussian approximation in Theorem~\ref{thm_fclt_asgd}. 

For the special case $q=2,$ the identities $\mu_2(\bm{v})^2=\EE\|D_k\XX_kD_k\bm{v}\|_2^2$ and $\EE(D_k\XX_kD_k)=p\EE[\XX_{k,p}]$ imply that condition (\ref{eq_sgd_condition}) can be rewritten into
\begin{align}
    \frac{\alpha}{2}\sup_{\bm{v}\in\RR^d,\|\bm{v}\|_2=1} \ \frac{\mu_2(\bm{v})^2}{p\bm{v}^{\top}\EE[\XX_{k,p}]\bm{v}} < 1,
\end{align}
and
\begin{align}
    \label{eq_sgd_condtion_q2}
    0 < \alpha < \inf_{\bm{v}\in\RR^d,\|\bm{v}\|_2=1}\frac{2\bm{v}^{\top}\EE(D_k\XX_kD_k)\bm{v}}{\EE\|D_k\XX_kD_k\bm{v}\|_2^2}.
\end{align}
For $q=2,$ Lemma~\ref{lemma_L2_convergence} in the Appendix states that the conclusion of the previous lemma is also implied by the condition $\EE[2\XX_k-\alpha\XX_k^2]>0$.

\begin{remark}[Interpretation for the range of learning rate]\label{rmk_interpretation_learning_rate}
Condition (\ref{eq_sgd_condition}) can be viewed as an \textit{``$L^2$-$L^q$ equivalence''}, where the left hand side can be interpreted as a measure of the convexity and smoothness of the loss functions. We show this 
for the case $q=2$, working with the equivalent condition \eqref{eq_sgd_condtion_q2}.

The loss function $g(\breve\bbeta_{k-1},(D_k, y_k,\bm{x}_k)):=\|y_k - \bm{x}_k^{\top}D_k\breve\bbeta_{k-1}\|_2^2/2$ in (\ref{eq_dropout_sgd_origin}) is strongly convex and the gradient is stochastic Lipschitz continuous. To see this, recall $\XX_k=\bm{x}_k\bm{x}_k^{\top}$. By taking the gradient with respect to the first argument, we obtain $\nabla g(\breve\bbeta_{k-1},(D_k, y_k,\bm{x}_k)) = D_k\bm{x}_k(y_k-\bm{x}_k^{\top}D_k\breve\bbeta_{k-1}) = -(D_k\XX_kD_k)\breve\bbeta_{k-1} + D_k\bm{x}_ky_k.$ For any two vectors $\breve\bbeta_{k-1},\breve\bbeta_{k-1}'$, we have the strong convexity
\begin{align*}
    \Big\langle \EE \nabla g\big(\breve\bbeta_{k-1},(D_k, y_k,\bm{x}_k)\big) - \EE\nabla g\big(\breve\bbeta_{k-1}',(D_k, y_k,\bm{x}_k)\big), \,\breve\bbeta_{k-1} - \breve\bbeta_{k-1}' \Big\rangle \ge J\|\breve\bbeta_{k-1} - \breve\bbeta_{k-1}'\|_2^2,
\end{align*}
with the constant $J:= \inf_{\bm{v}\in\RR^d,\|\bm{v}\|_2=1}\EE\|D_k\XX_kD_k\bm{v}\|_2 = \lambda_{\min}\{\EE(D_k\XX_kD_k)\}.$ Furthermore, given any two vectors $\breve\bbeta_{k-1},\breve\bbeta_{k-1}'\in\RR^d$, we also have $\EE\|\nabla g\big(\breve\bbeta_{k-1},(D_k, y_k,\bm{x}_k)\big) - \nabla g\big(\breve\bbeta_{k-1}',(D_k, y_k,\bm{x}_k)\big)\|_2^2 \le K\|\breve\bbeta_{k-1} - \breve\bbeta_{k-1}'\|_2^2,$ with $K:=\sup_{\bm{v}\in\RR^d,\|\bm{v}\|_2=1}\EE\|D_k\XX_kD_k\bm{v}\|_2^2 \le \lambda_{\max}^2\{\EE(D_k\XX_kD_k)\}.$ This implies the stochastic Lipschitz continuity of the gradient $\nabla g.$ Therefore, learning rate $\alpha$ satisfying 
\begin{align}
    0 < \alpha < 2J/K,
\end{align}
ensures by Lemma~\ref{lemma_gmc_sgd_cond} contraction of the second moment of $I_d - \alpha D\XX_kD.$ The constant $K/J$ is also related to the \textit{condition number} of the matrix $\EE[D_k\XX_kD_k]$. When the dimension $d$ grows, the condition number can be larger and thus the learning rate $\alpha$ needs to be small. 
\end{remark}

For the geometric-moment contraction for the SGD dropout sequence, we impose the following moment conditions. 

\begin{assumption}[Finite moment]
    \label{asm_fclt_asgd}
    Assume that for some $q\ge2$, the random noises $\epsilon$ and the random sample $\bm{x}$ in model (\ref{eq_model_sgd}) have finite $2q$-th moment  $\EE[|\epsilon|^{2q}]+\|\bm{x}\|_2^{2q}]<\infty.$
\end{assumption}
Lemma~\ref{lemma_true_l2} in the Appendix shows that this assumption ensures the finite $q$-th moment of the stochastic gradient in (\ref{eq_dropout_sgd_origin}) evaluated at the true parameter $\bbeta^*$ and the $\ell^2$-minimizer $\breve\bbeta$ in model (\ref{eq_model_sgd}), that is, 
$$\Big(\EE\Big\|\nabla_{\bbeta^*}\frac{1}{2}\big(y-\bm{x}^{\top} D\bbeta^*\big)^2\Big\|_2^q\Big)^{1/q} = \Big(\EE\big\|D\bm{x}\big(y-\bm{x}^{\top}D\bbeta^*\big)\big\|_2^q\Big)^{1/q}<\infty,$$
and $(\EE\|\nabla_{\breve\bbeta}\frac{1}{2}\big(y-\bm{x}^{\top} D\breve\bbeta\big)^2\|_2^q)^{1/q} <\infty.$ Now, we are ready to show the GMC property of the SGD dropout sequence.

\begin{theorem}[Geometric-moment contraction of SGD dropout]\label{thm_gmc_sgd}
Let $q>1$. Suppose that Assumption~\ref{asm_fclt_asgd} holds and the learning rate $\alpha$ satisfies \eqref{eq_sgd_condition}. For two dropout sequences $\breve\bbeta_k(\alpha)$ and $\breve\bbeta_k'(\alpha),$ $k=0,1,\ldots$, that are generated by the recursion (\ref{eq_dropout_sgd_origin}) with the same sequence of dropout matrices $\{D_k\}_{k\in\NN}$ but possibly different initializations $\breve\bbeta_0$, $\breve\bbeta_0'$, we have
\begin{equation}
    \label{eq_sgd_gmc}
    \Big(\EE\big\|\breve\bbeta_k(\alpha) - \breve\bbeta_k'(\alpha)\big\|_2^q \Big)^{1/q} \le \breve r_{\alpha,q}^k\|\breve\bbeta_0 - \breve\bbeta_0'\|_2, \quad \text{for all} \ k=1,2,\ldots,
\end{equation}
with
\begin{equation}
    \label{eq_sgd_r}
    \breve r_{\alpha,q} = \Big(\sup_{\bm{v}\in\RR^d:\,\|\bm{v}\|_2=1}\EE\big\|\breve A_1(\alpha)\bm{v}\big\|_2^q\Big)^{1/q} <1.
\end{equation}
Moreover, for any initial vector $\breve\bbeta_0\in\RR^d$, there exists a unique stationary distribution $\breve\pi_{\alpha}$ which does not depend on $\breve\bbeta_0$, such that $\breve\bbeta_k(\alpha)\Rightarrow\breve\pi_{\alpha}$ as $k\rightarrow\infty$.
\end{theorem}

By Theorem~\ref{thm_gmc_sgd}, initializing $\breve\bbeta_0^{\circ}\sim\breve\pi_{\alpha}$ leads to the stationary SGD dropout sequence $\{\breve\bbeta_k^{\circ}(\alpha)\}_{k\in\NN}$ by following the recursion \begin{align}
    \label{eq_dropout_sgd_stationary}
    \breve\bbeta_k^{\circ}(\alpha) - \breve\bbeta = \breve A_k(\alpha)(\breve\bbeta_{k-1}^{\circ}(\alpha) - \breve\bbeta) + \breve{\bm{b}}_k(\alpha), \quad k=1,2,\ldots,
\end{align}
where the $\ell^2$-regularized minimizer $\breve\bbeta$ is defined in (\ref{eq_goal_sgd}), and the random coefficients $\breve A_k(\alpha) = I_d - \alpha D_k\XX_kD_k$ and $ \breve{\bm{b}}_k(\alpha) = \alpha D_k\bm{x}_k(y_k - \bm{x}_k^{\top}D_k\breve\bbeta)$ are defined in (\ref{eq_dropout_sgd}). Furthermore, recall the iterated random function $\breve f_{D,(y,\bm{x})}(\bbeta)$ defined in (\ref{eq_dropout_function_sgd}). As a direct consequence of Theorem~\ref{thm_gmc_sgd}, we have 
\begin{equation}
    \label{eq_SGD_stationary_f}
    \breve\bbeta_k^{\circ}(\alpha)=\breve f_{D_k,(y_k,\bm{x}_k)}(\breve\bbeta_{k-1}^{\circ}(\alpha)),
\end{equation}
which holds for all $k\in\ZZ$. To see the case with $k\le0$, we only need to notice that, for any $\bbeta\in\RR^d$, we have the limit
\begin{equation}
    \label{eq_stationary_sgd}
    \breve\bbeta_k^{\circ}:=\lim_{m\rightarrow\infty}\breve f_{\bm{\xi}_k}\circ \cdots \breve f_{\bm{\xi}_{k-m}}(\bbeta) =: \breve h_{\alpha}(\bm{\xi}_k,\bm{\xi}_{k-1},\ldots),
\end{equation}
where $\breve h_{\alpha}$ is a measurable function that depends on $\alpha$, and we use $\bm{\xi}_k$ to denote all the new-coming random parts in the $k$-th iteration, that is,
\begin{equation}
    \label{eq_sgd_random_pair}
    \bm{\xi}_k=(D_k,(y_k,\bm{x}_k)), \quad k\in\ZZ.
\end{equation}
For $k\le0$, $\bm{\xi}_k$ can be viewed as an i.i.d.\ copy of $\bm{\xi}_j$ for some $j\ge1$. The limit $\breve h_{\alpha}(\bm{\xi}_k,\bm{\xi}_{k-1},\ldots)$ exists almost surely and does not depend on $\bbeta$. Therefore, the iteration $\breve\bbeta_k^{\circ}(\alpha)=\breve f_{D_k,(y_k,\bm{x}_k)}(\breve\bbeta_{k-1}^{\circ}(\alpha))$ in (\ref{eq_SGD_stationary_f}) holds for all $k\in\ZZ$.

\subsection{Asymptotics of Dropout in SGD}\label{subsec_asymp_sgd}

In this section, we provide the asymptotics for the $k$-th iterate of SGD dropout and the Ruppert-Polyak averaged version.

\begin{lemma}[Moment convergence of iterative SGD dropout]\label{lemma_q_moment_sgd}
    Let $q\ge2$ and suppose that Assumption~\ref{asm_fclt_asgd} holds. For the stationary SGD dropout sequence $\{\breve\bbeta_k^{\circ}(\alpha)\}_{k\in\NN}$ defined in (\ref{eq_dropout_sgd_stationary}) with learning rate $\alpha$ satisfying \eqref{eq_sgd_condition}, we have
    \begin{equation}
        \max_k \big(\EE\|\breve\bbeta_k^{\circ}(\alpha)-\breve\bbeta\|_2^q\big)^{1/q} = O(\sqrt{\alpha}).
    \end{equation}
\end{lemma}

Besides the stochastic order of the last iterate of SGD dropout $\breve\bbeta_k(\alpha)$, we are also interested in the limiting distribution of the Ruppert-Polyak averaged SGD dropout, which can effectively reduce the variance and keep the online computing scheme. In particular, we define
\begin{equation}
    \label{eq_ave_dropout_sgd}
    \bar\bbeta_n^{\sgd}(\alpha) = \frac{1}{n}\sum_{k=1}^n\breve\bbeta_k(\alpha).
\end{equation}

\begin{theorem}[Quenched CLT of averaged SGD dropout]\label{thm_clt_ave_sgd}
If the learning rate $\alpha$ satisfies \eqref{eq_sgd_condition}, then, 
\begin{equation}
    {\sqrt{n}(\bar\bbeta_n^{\sgd}(\alpha) - \breve\bbeta)}
      \Rightarrow \N(0,\breve\Sigma(\alpha)), \quad \text{as} \ n\to \infty,
\end{equation}
with $\breve\Sigma(\alpha):= \sum_{k=-\infty}^{\infty}\EE[(\breve\bbeta_0^{\circ}(\alpha) - \breve\bbeta)(\breve\bbeta_k^{\circ}(\alpha)-\breve\bbeta)^{\top}]$ the long-run covariance matrix of the stationary process $\breve\bbeta_k^{\circ}(\alpha)\sim\breve\pi_{\alpha}.$
\end{theorem}
As discussed above Corollary~\ref{cor_clt_agd_parallel}, one can also choose different learning rates $\alpha_1,\ldots,\alpha_s$ and then run the SGD dropout sequences $\breve\bbeta_k(\alpha_1),\ldots,\breve\bbeta_k(\alpha_s)$ in parallel. For $d$-dimensional vectors $\bm{u}_1,\ldots,\bm{u}_s,$ recall that $\mathrm{vec}(\bm{u}_1,\ldots,\bm{u}_s):=(\bm{u}_1^{\top},\ldots,\bm{u}_s^{\top})^{\top}$ is the $ds$-dimensional concatenation. 

\begin{corollary}[Quenched CLT of parallel averaged SGD dropout]\label{cor_clt_asgd_parallel}
    Let $s\ge1$. Consider constant learning rates $\alpha_1,\ldots,\alpha_s$ satisfying the condition in (\ref{eq_sgd_condition}). Then, for any initial vectors $\breve\bbeta_0(\alpha_1),\ldots,\breve\bbeta_0(\alpha_s)$,
    \begin{align}
        \sqrt{n}\cdot\mathrm{vec}\Big(\bar\bbeta_n^{\sgd}(\alpha_1)-\breve\bbeta,\ldots,\bar\bbeta_n^{\sgd}(\alpha_s)-\breve\bbeta\Big) \Rightarrow \N\big(0,\breve\Sigma^{\mathrm{vec}}\big),
    \end{align}
    with the long-run covariance matrix
    $$\breve\Sigma^{\mathrm{vec}}=\sum_{k=-\infty}^{\infty}\mathrm{Cov}\big(\mathrm{vec}\big(\breve\bbeta_0^{\circ}(\alpha_1),\ldots,\breve\bbeta_0^{\circ}(\alpha_s)\big),\mathrm{vec}\big(\breve\bbeta_k^{\circ}(\alpha_1),\ldots,\breve\bbeta_k^{\circ}(\alpha_s)\big)\big).$$
\end{corollary}

\begin{theorem}[Quenched invariance principle of averaged SGD dropout]\label{thm_fclt_asgd}
    Suppose that Assumption~\ref{asm_fclt_asgd} holds for some $q>2$ and the learning rate $\alpha$ satisfies \eqref{eq_sgd_condition}. Define a partial sum process $(\breve S_i^{\circ}(\alpha))_{1\le i\le n}$ with
    \begin{equation}
        \label{eq_asgd_partial_sum_stationary}
        \breve S_i^{\circ}(\alpha) = \sum_{k=1}^i(\breve\bbeta_k^{\circ}(\alpha) - \breve\bbeta).
    \end{equation}
    Then, there exists a (richer) probability space $(\breve\Omega^{\star},\breve\A^{\star},\breve\PP^{\star})$ on which one can define $d$-dimensional random vectors $\breve\bbeta_k^{\star}$, the associated partial sum process $\breve S_i^{\star}(\alpha)=\sum_{k=1}^i(\breve\bbeta_k^{\star}(\alpha) - \breve\bbeta)$, and a Gaussian process $\breve G_i^{\star}=\sum_{k=1}^i\breve{\bm{z}}_k^{\star}$, with independent Gaussian random vectors $\breve{\bm{z}}_k\sim \N(0,I_d)$, such that $(\breve S_i^{\circ})_{1\le i\le n}\overset{\D}{=}(\breve S_i^{\star})_{1\le i\le n}$ and
    \begin{equation}
        \max_{1\le i\le n}\big\|\breve S_i^{\star} - \breve\Sigma^{1/2}(\alpha)\breve G_i^{\star}\big\|_2 = o_{\PP}(n^{1/q}), \quad \text{in }(\breve\Omega^{\star},\breve\A^{\star},\breve\PP^{\star}),
    \end{equation}
    where $\breve\Sigma(\alpha)$ is the long-run covariance matrix defined in Theorem \ref{thm_clt_ave_sgd}. In addition, this approximation holds for all $(\breve S_i^{\breve\bbeta_0}(\alpha))_{1\le i\le n}$ given any arbitrary initialization $\breve\bbeta_0\in\RR^d$, where
    \begin{equation}
        \breve S_i^{\breve\bbeta_0}(\alpha) = \sum_{k=1}^i(\breve\bbeta_k(\alpha) - \breve\bbeta).
    \end{equation}
\end{theorem}

\section{Online Inference for SGD with Dropout}\label{sec_online_inference}

The long-run covariance matrix $\breve\Sigma(\alpha)$ of the averaged SGD dropouts is usually unknown and needs to be estimated. We now propose an online estimation method for $\breve\Sigma(\alpha)$, and establish theoretical guarantees. 

The key idea is to adopt the non-overlapping batched means (NBM) method (\cite{lahiri_theoretical_1999,lahiri_resampling_2003,xiao_single-pass_2011}), which resamples blocks of observations to estimate the long-run covariance of dependent data. Essentially, a sequences of non-overlapping blocks are pre-specified. When the block sizes are large enough, usually increasing as the the sample size grows, the block sums shall behave similar to independent observations and therefore can be used to estimate the long-run covariance. In this paper, to facilitate the online inference of the dependent SGD dropout iterates $\{\breve\bbeta_k(\alpha)\}_{k\in\NN}$, we shall extend the offline NBM estimators to online versions by only including the past SGD dropout iterates in each batch. The overlapped batch-means (OBM) methods are also investigated in literature; see for example \textcite{xiao_single-pass_2011,zhu_online_2023}. We shall only focus on the NBM estimates in this study given its simpler structure.

Let $\eta_1, \eta_2, \ldots$ be a strictly increasing integer-valued sequence satisfying $\eta_1=1$ and $\eta_{m+1}-\eta_m\rightarrow\infty$ as $m\rightarrow\infty$. For each $m$, we let $B_m$ denote the block
\begin{equation}
    \label{eq_block}
    B_m = \{\eta_m,\eta_m+1,\ldots,\eta_{m+1}-1\}.
\end{equation}
For the $n$-th SGD dropout iteration, denote by $\psi(n)$ the largest index $m$ such that $\eta_m\le n$. For any $d$-dimensional vector $\bm{v}=(v_1,\ldots,v_d)^{\top},$ the Kronecker product is the $d\times d$ matrix $\bm{v}^{\otimes 2}=(v_iv_j)_{i,j=1}^d.$ Based on the non-overlapping blocks $\{B_m\}_{m\in\NN}$, for the $n$-th iteration, we can estimate the long-run covariance matrix $\breve\Sigma(\alpha)$ in Theorem \ref{thm_clt_ave_sgd} by
\begin{equation}
    \label{eq_longrun_est}
    \hat\Sigma_n(\alpha) = \frac{1}{n} \sum_{m = 1}^{\psi(n) - 1} \Big(\sum_{k \in B_{m}} \big[\breve\bbeta_k(\alpha) - \bar\bbeta_{n}^{\sgd}(\alpha)\big]\Big)^{\otimes 2} + \frac 1n \Big(\sum_{k = \eta_{\psi(n)}}^{n} \big[\breve\bbeta_k(\alpha) - \bar\bbeta_{n}^{\sgd}(\alpha)\big]\Big)^{\otimes 2}.
\end{equation}
The estimator $\hat\Sigma_n(\alpha)$ is composed of two parts. The first part takes the sum within each block and then estimates the sample covariances of these centered block sums. The second part accounts for the remaining observations, which can be viewed as the estimated covariance of the tail block.

For the recursive computation of $\hat\Sigma_n(\alpha)$, we need to rewrite (\ref{eq_longrun_est}) such that, in the $n$-th iteration, we can update $\hat\Sigma_n(\alpha)$ based on the information from the $(n-1)$-st step and the latest iterate $\breve\bbeta_n(\alpha)$. To this end, we denote the number of iterates included in the tail part (i.e., the second part) in (\ref{eq_longrun_est}) by 
\begin{equation}
    \delta_{\eta}(n) = n - \eta_{\psi(n)} + 1,
\end{equation}
and define two partial sums
\begin{align}
    \label{eq_longrun_twoparts}
    \mathcal{S}_{m}(\alpha) = \sum_{k \in B_{m}} \breve\bbeta_k(\alpha) \enspace \quad\mathrm{and}\quad \enspace \mathcal{R}_{n}(\alpha) = \sum_{k = \eta_{\psi(n)}}^{n} \breve\bbeta_k(\alpha). 
\end{align}
Then, we notice that the estimator $\hat\Sigma_n(\alpha)$ in (\ref{eq_longrun_est}) can be rewritten as follows, 
\begin{align}
    \label{eq_longrun_online}
    \hat\Sigma_n(\alpha) & = \frac{1}{n}\Bigg[\Big(\sum_{m = 1}^{\psi(n) - 1} \mathcal{S}_{m}(\alpha)^{\otimes 2} + \mathcal{R}_{n}(\alpha)^{\otimes 2}\Big) + \Big(\sum_{m = 1}^{\psi(n) - 1} |B_{m}|^{2} + |\delta_{\eta}(n)|^{2}\Big) \bar\bbeta_{n}^{\sgd}(\alpha)^{\otimes 2} \nonumber \\
    & \quad - \Big(\sum_{m = 1}^{\psi(n) - 1} |B_{m}| \mathcal{S}_{m}(\alpha) + \delta_{\eta}(n) \mathcal{R}_{n}(\alpha)\Big) \bar\bbeta_{n}^{\sgd}(\alpha)^{\top} \nonumber \\
    & \quad - \bar\bbeta_{n}^{\sgd}(\alpha) \Big(\sum_{m = 1}^{\psi(n) - 1} |B_{m}| \mathcal{S}_{m}(\alpha) + \delta_{\eta}(n) \mathcal{R}_{n}(\alpha)\Big)^{\top}\Bigg] \nonumber \\
    & =: \frac{1}{n}\Big[\mathcal{V}_{n}(\alpha) + K_{n} \bar\bbeta_{n}^{\sgd}(\alpha)^{\otimes 2} - H_{n}(\alpha) \bar\bbeta_{n}^{\sgd}(\alpha)^{\top} - \bar\bbeta_{n}^{\sgd}(\alpha) H_{n}(\alpha)^{\top}\Big]. 
\end{align}
As such, the estimation of $\breve\Sigma(\alpha)$ reduces to recursively computing $\{\mathcal{V}_{n}(\alpha), K_{n}, H_{n}(\alpha), \bar\bbeta_{n}^{\sgd}(\alpha)\}$ with respect to $n.$ We provided the pseudo codes of the recursion in Algorithm~\ref{algo_longrun}. We shall further establish the convergence rate of the proposed online estimator $\hat\Sigma_n(\alpha)$ in Theorem~\ref{thm_longrun_precision_old}. 

\begin{algorithm}[hbt!]
   \setstretch{1.25}
   \caption{Online estimation of the long-run covariance matrix of ASGD dropout}
   \label{algo_longrun}
   \KwData{Sequential random samples $(y_1,\bm{x}_1),\ldots,(y_n,\bm{x}_n)$; sequential dropout matrices $D_1,\ldots,D_n$; constant learning rate $\alpha$; predefined sequences $\{\eta_m\}_{m\in\NN}$}
   \KwResult{ASGD dropout $\bar\bbeta_{n+1}^{\sgd}(\alpha)$; estimated long-run covariance matrix $\tilde{\Sigma}_{n+1}(\alpha)$} 
   \Init{$\breve\bbeta_0(\alpha)=\bar\bbeta_0^{\sgd}(\alpha)=\mathcal{R}_0(\alpha)\gets0$,}{$\psi(0)\gets1$, $\delta_{\eta}(0)\gets1$, $K_0=H_0(\alpha)\gets1$, $\mathcal{V}_0(\alpha)\gets0$
   }
   \For{$n=0,1,2,3,\ldots$}{
   $\breve\bbeta_{n+1}(\alpha) \gets \breve\bbeta_n(\alpha) + \alpha D_n\bm{x}_n\big(y_n - \bm{x}_n^{\top}D_n\breve\bbeta_{n-1}(\alpha)\big)$ \Comment*[r]{SGD dropout}
   $\bar\bbeta_{n + 1}^{\sgd}(\alpha) \gets \{n\bar\bbeta_{n}^{\sgd}(\alpha) + \breve\bbeta_{n+1}(\alpha)\}/(n + 1)$ \Comment*[r]{ASGD dropout}
   \eIf{$n + 1 < \eta_{\psi(n) + 1}$}{
   $\mathcal{R}_{n + 1}(\alpha) \gets \mathcal{R}_{n}(\alpha) + \breve\bbeta_{n + 1}(\alpha)$, $\delta_{\eta}(n + 1) \gets \delta_{\eta}(n) + 1$\;
   $K_{n + 1} \gets K_{n} - \delta_{\eta}^{2}(n) + \delta_{\eta}^{2}(n + 1)$, $\psi(n + 1) \gets \psi(n)$\;
   $H_{n + 1}(\alpha) \gets H_{n}(\alpha) - \delta_{\eta}(n) \mathcal{R}_{n}(\alpha) + \delta_{\eta}(n + 1) \mathcal{R}_{n + 1}(\alpha)$\;
   $\mathcal{V}_{n + 1}(\alpha) \gets \mathcal{V}_{n}(\alpha) - \mathcal{R}_{n}(\alpha)^{\otimes 2} + \mathcal{R}_{n + 1}(\alpha)^{\otimes 2}$\;
   }{
   $\mathcal{R}_{n + 1}(\alpha) \gets \breve\bbeta_{n + 1}(\alpha)$, $\delta_{\eta}(n + 1) \gets 1$\;
   $\psi(n + 1) \gets \psi(n) + 1$\;
   $K_{n + 1} \gets K_{n} + 1$, $H_{n + 1}(\alpha) \gets H_{n}(\alpha) + \mathcal{R}_{n + 1}(\alpha)$\;
   $\mathcal{V}_{n + 1}(\alpha) \gets \mathcal{V}_{n}(\alpha) + \mathcal{R}_{n + 1}(\alpha)^{\otimes 2}$\;
   }
   $\hat{\Sigma}_{n+1}(\alpha) \gets \big[\mathcal{V}_{n + 1}(\alpha) + K_{n + 1} \bar\bbeta^{\sgd}_{n + 1}(\alpha)^{\otimes 2} - H_{n + 1}(\alpha) \bar\bbeta^{\sgd}_{n + 1}(\alpha)^{\top} - \bar\bbeta^{\sgd}_{n + 1}(\alpha) H_{n + 1}(\alpha)^{\top}\big]/(n+1)$\Comment*[r]{Estimated long-run covariance matrix}
   }
\end{algorithm}
The rational behind Algorithm~\ref{algo_longrun} is as follows: if $n + 1 < \eta_{\psi(n) + 1}$, then the index $n + 1$ still belongs to the block $B_{\psi(n)}$ and $\psi(n + 1) = \psi(n)$. Also we have $\mathcal{R}_{n + 1}(\alpha) = \mathcal{R}_{n}(\alpha) + \breve\bbeta_{n+1}(\alpha)$ and $\delta_{\eta}(n + 1) = \delta_{\eta}(n) + 1$. Consequently, $\{K_{n + 1}, \mathcal{V}_{n + 1}(\alpha), H_{n + 1}(\alpha)\}$ can be recursively updated via
\begin{align*}
    & K_{n + 1} = K_{n} - |\delta_{\eta}(n)|^{2} + |\delta_{\eta}(n + 1)|^{2}, \\
    & \mathcal{V}_{n + 1}(\alpha) = \mathcal{V}_{n}(\alpha) - \mathcal{R}_{n}(\alpha)^{\otimes 2} + \mathcal{R}_{n + 1}(\alpha)^{\otimes 2}, \cr
    & H_{n + 1}(\alpha) = H_{n}(\alpha) - \delta_{\eta}(n) \mathcal{R}_{n}(\alpha) + \delta_{\eta}(n + 1) \mathcal{R}_{n + 1}(\alpha).
\end{align*}
Otherwise, if $n + 1 = \eta_{\psi(n)}$, we have $\psi(n + 1) = \psi(n) + 1$. Hence $\mathcal{R}_{n + 1}(\alpha) = \breve\bbeta_{n+1}(\alpha)$ and $\delta_{\eta}(n + 1) = 1$. In this case, $\{K_{n + 1}, \mathcal{V}_{n + 1}(\alpha), H_{n + 1}(\alpha)\}$ can be recursively updated as follows,
\begin{align*}
    & K_{n + 1} = K_{n} + 1,\\
    & \mathcal{V}_{n + 1}(\alpha) = \mathcal{V}_{n}(\alpha) + \mathcal{R}_{n + 1}(\alpha)^{\otimes 2},\\
    & H_{n + 1}(\alpha) = H_{n}(\alpha) + \mathcal{R}_{n + 1}(\alpha).
\end{align*}
As such, given $\breve\bbeta_1(\alpha), \ldots, \breve\bbeta_n(\alpha)$, the estimator $\hat\Sigma_n(\alpha)$ for the long-run covariance matrix $\breve\Sigma(\alpha)$ can be updated in an online manner, requiring only $O(1)$ memory storage.

\begin{theorem}[Precision of $\hat\Sigma_n(\alpha)$]
\label{thm_longrun_precision_old}
Let $\eta_{m} = \lfloor c m^{\zeta} \rfloor$ for some $c > 0$ and $\zeta > 1$. Let conditions in Theorem~\ref{thm_gmc_sgd} hold with some $q\ge4$. Then, we have
\begin{align*}
    \mathbb{E} \|\hat{\Sigma}_{n}(\alpha) - \breve\Sigma(\alpha)\| \lesssim n^{(1/\zeta - 1) \vee (-1/(2\zeta))},
\end{align*}
where $\|\cdot\|$ denotes the operator norm, and the constant in $\lesssim$ depends on $c,q$ and $d$.
\end{theorem}
In particular, for $\zeta=3/2,$
\begin{align}
    \mathbb{E} \|\hat{\Sigma}_{n}(\alpha) - \breve\Sigma(\alpha)\| \lesssim n^{-1/3},
\end{align}
and this rate is optimal among long-run covariance estimators, even when comparing to offline estimation. See \textcite{xiao_single-pass_2011} for details.

By the estimation procedure summarized in Algorithm~\ref{algo_longrun}, we can asymptotically estimate the long-run covariance matrix of the SGD dropout iterates $\bar\bbeta_n^{\sgd}(\alpha)$ for any arbitrarily fixed initial vector. For some given confidence level $\omega\in(0,1),$ in the $n$-th iteration the online confidence interval for each coordinate $\bbeta_j^*$, $j=1,\ldots,d$, of the unknown parameter $\bbeta^*$ in model (\ref{eq_model_sgd}) is
\begin{equation}
    \label{eq_CI_1d}
    \mathrm{CI}_{\omega,n,j} := \Big[\bar\bbeta_{n,j}^{\sgd}(\alpha) - z_{1-\omega/2}\sqrt{\hat\sigma_{n,jj}(\alpha)/n}, \,\, \bar\bbeta_{n,j}^{\sgd}(\alpha) + z_{1-\omega/2}\sqrt{\hat\sigma_{n,jj}(\alpha)/n}\Big],
\end{equation}
with $z_{1-\omega/2}$ denoting the $(1-\omega/2)$-percentile of the standard normal distribution. Here, $\hat\sigma_{n,jj}(\alpha)$ is the $j$-th diagonal of the proposed online long-run covariance estimator $\hat\Sigma_n(\alpha)$ in (\ref{eq_longrun_est}), and $\bar\bbeta_{n,j}^{\sgd}(\alpha)$ is the $j$-th coordinate of the averaged SGD dropout estimate $\bar\bbeta_n^{\sgd}(\alpha)$. Furthermore, the online joint confidence regions for the vector $\bbeta^*$ is
\begin{equation}
    \label{eq_CI}
    \mathrm{CI}_{\omega,n} := \Big\{\bbeta\in\RR^d: \, n\big(\bar\bbeta_n^{\sgd}(\alpha) - \bbeta\big)^{\top}\hat\Sigma_n^{-1}(\alpha)\big(\bar\bbeta_n^{\sgd}(\alpha) - \bbeta\big)\le \chi^2_{d,1-\omega/2}\Big\},
\end{equation}
where $\chi^2_{d,1-\omega/2}$ is the $(1-\omega/2)$-percentile of the $\chi_d^2$ distribution with $d$ degrees of freedom.
\begin{corollary}[Asymptotic coverage probability]\label{cor_asymp_CI}
Suppose that Assumption~\ref{asm_fclt_asgd} holds and the learning rate $\alpha$ satisfies \eqref{eq_sgd_condition}. Given $\omega\in(0,1)$ and $\eta_{m} = \lfloor c m^{\zeta} \rfloor$ for some $c > 0$ and $\zeta > 1$, $\mathrm{CI}_{\omega,n,j}$ defined in (\ref{eq_CI_1d}), and $\mathrm{CI}_{\omega,n}$ defined in (\ref{eq_CI}) are asymptotic $100(1-\omega)\%$ confidence intervals, that is, $\PP\big(\bbeta_j^* \in \mathrm{CI}_{\omega,n,j}\big) \rightarrow 1-\omega$ for all $j=1,\ldots,d$, and $\PP\big(\bbeta^* \in \mathrm{CI}_{\omega,n}\big) \rightarrow 1-\omega,$ as $n\rightarrow\infty$. More generally, for any $d$-dimensional unit-length vector $\bm{v}$ with $\|\bm{v}\|_2=1$, and $z_{1-\omega/2}$ the $(1-\omega/2)$-quantile of the standard normal distribution,
\begin{equation}
    \label{eq_CI_projection}
    \mathrm{CI}_{\omega,n}^{\mathrm{Proj}} := \Big[\bm{v}^{\top}\bar\bbeta_n^{\sgd}(\alpha) - z_{1-\omega/2}\sqrt{n^{-1}(\bm{v}^{\top}\hat\Sigma_n(\alpha)\bm{v})}, \,\, \bm{v}^{\top}\bar\bbeta_n^{\sgd}(\alpha) + z_{1-\omega/2}\sqrt{n^{-1}(\bm{v}^{\top}\hat\Sigma_n(\alpha)\bm{v})}\Big]
\end{equation}
is an asymptotic $100(1-\omega)\%$ confidence interval for the one-dimensional projection $\bm{v}^{\top}\bbeta^*,$ that is, $\PP\big(\bm{v}^{\top}\bbeta^* \in \mathrm{CI}_{\omega,n}^{\mathrm{Proj}}\big) \rightarrow 1-\omega$, as $n\rightarrow\infty$.
\end{corollary}
By the quenched CLT of the averaged SGD dropout sequence $\{\bar\bbeta_n^{\sgd}(\alpha)\}_{n\in\NN}$ in Theorem~\ref{thm_fclt_asgd} and the consistency of $\hat\Sigma_n(\alpha)$ in Theorem~\ref{thm_longrun_precision_old}, we can apply Slutsky's theorem and obtain the results in Corollary~\ref{cor_asymp_CI}. In Section~\ref{sec_simulation}, we shall validate the proposed online inference method by examining the estimation accuracy of the proposed online estimator $\hat\Sigma_n(\alpha)$ and the coverage probability of $\mathrm{CI}_{\omega,n}^{\mathrm{Proj}}$ under different settings.

\section{Simulation Studies}\label{sec_simulation}

In this section, we present the results of the numerical experiments to demonstrate the validity of the proposed online inference methodology. The codes for reproducing all results and figures can be found online\footnote{\url{https://github.com/jiaqili97/Dropout_SGD}}.

\subsection{Sharp Range of the Learning Rate}
The GD dropout iterates can be defined via the recursion (\ref{eq_dropout_gd}), $\tilde\bbeta_k(\alpha)-\tilde\bbeta = A_k(\alpha)(\tilde\bbeta_{k-1}(\alpha) - \tilde\bbeta) + \bm{b}_k(\alpha)$, and the derived theory requires the learning rate $\alpha$ to satisfy $\alpha\|\XX\|<2$. Via a simulation study we show that this range is close to sharp to guarantee that the contraction constant
\begin{equation}
    r_{\alpha,2}^2=\sup_{\bm{v}\in\RR^d:\|\bm{v}\|_2=1}\EE\big\|A_1(\alpha)\bm{v}\big\|_2^2 = \lambda_{\max}\big(\EE[A_1^{\top}(\alpha)A_1(\alpha)]\big) <1.
\end{equation}
This then indicates that the condition $\alpha\|\XX\|<2$ in Lemma \ref{lemma_gmc_cond} can likely not be improved. 

For the $n\times d$ full design matrix $X$, we independently generate each entry of $X$ from the standard normal distribution. Since $\XX=X^{\top}X$, the upper bound $2/\lambda_{\max}(\XX)$ of the learning rate $\alpha$ can be computed. Then, we independently generate $N=500$ dropout matrices $D_i$, $i=1,\ldots,N$ with retaining probability $p.$ The simulation study evaluates the empirical contraction constant
\begin{equation}
    \hat r_{\alpha,2}^2 := \lambda_{\max}\Big(N^{-1}\sum_{i=1}^NA_i^{\top}(\alpha)A_i(\alpha)\Big),
\end{equation}
for different sample size $n$, dimension $d$, retaining probability $p,$ and learning rate $\alpha$.
\begin{table}[htbp!]
    \centering
    \begin{tabular}{c c c c c}
        \hline\hline
        $n,d$ & $2/\lambda_{\max}(\XX)$ & $p$ & $\alpha$ & $\hat r_{\alpha,2}^2$ \\
        \hline
        100, 5 & 0.0151 & 0.9 & 0.0150 & 0.97 \\
         &  &  & 0.0154 & 1.02 \\
        100, 50 & 0.0068 & 0.9 & 0.0067 & 0.93 \\
         &  &  & 0.0072 & 1.01 \\
         &  & 0.8 & 0.0068 & 0.90 \\
         &  &  & 0.0075 & 1.02 \\   
        100, 100 & 0.0052 & 0.9 & 0.0050 & 0.93 \\
         &  &  & 0.0057 & 1.15 \\
         &  & 0.5 & 0.0050 & 0.94 \\
         &  &  & 0.0075 & 1.06 \\
        \hline\hline
    \end{tabular}
    \caption{Effects of the learning rate $\alpha$ on the geometric moment contraction of the GD iterates with dropout.}
    \label{table_gmc}
\end{table}
Table~\ref{table_gmc} shows that even if the learning rate $\alpha$ exceeds the upper bound $2/\lambda_{\max}(\XX)$ by a small margin, the contraction will not hold any more since $\hat r_{\alpha,2}^2>1$. This indicates that the condition $\alpha\|\XX\|<2$ is close to sharp.

\subsection{Estimation of Long-Run Covariance Matrix}
In this section, we provide the simulation results of the proposed long-run covariance estimator $\hat\Sigma_n(\alpha)$ defined in (\ref{eq_longrun_est}), and its online version (\ref{eq_longrun_online}).

Figure~\ref{fig:gmc_example} shows the convergence of the GD and SGD iterates with dropout. The coordinates of the true regression vector $\bbeta^*$ are equidistantly spaced between 0 and 1. One can see that the initialization is quickly forgotten in both GD and SGD algorithms.
\begin{figure}[!hbt]
    \centering
    \begin{subfigure}[b]{0.49\textwidth}
        \centering
        \includegraphics[width=1\textwidth]{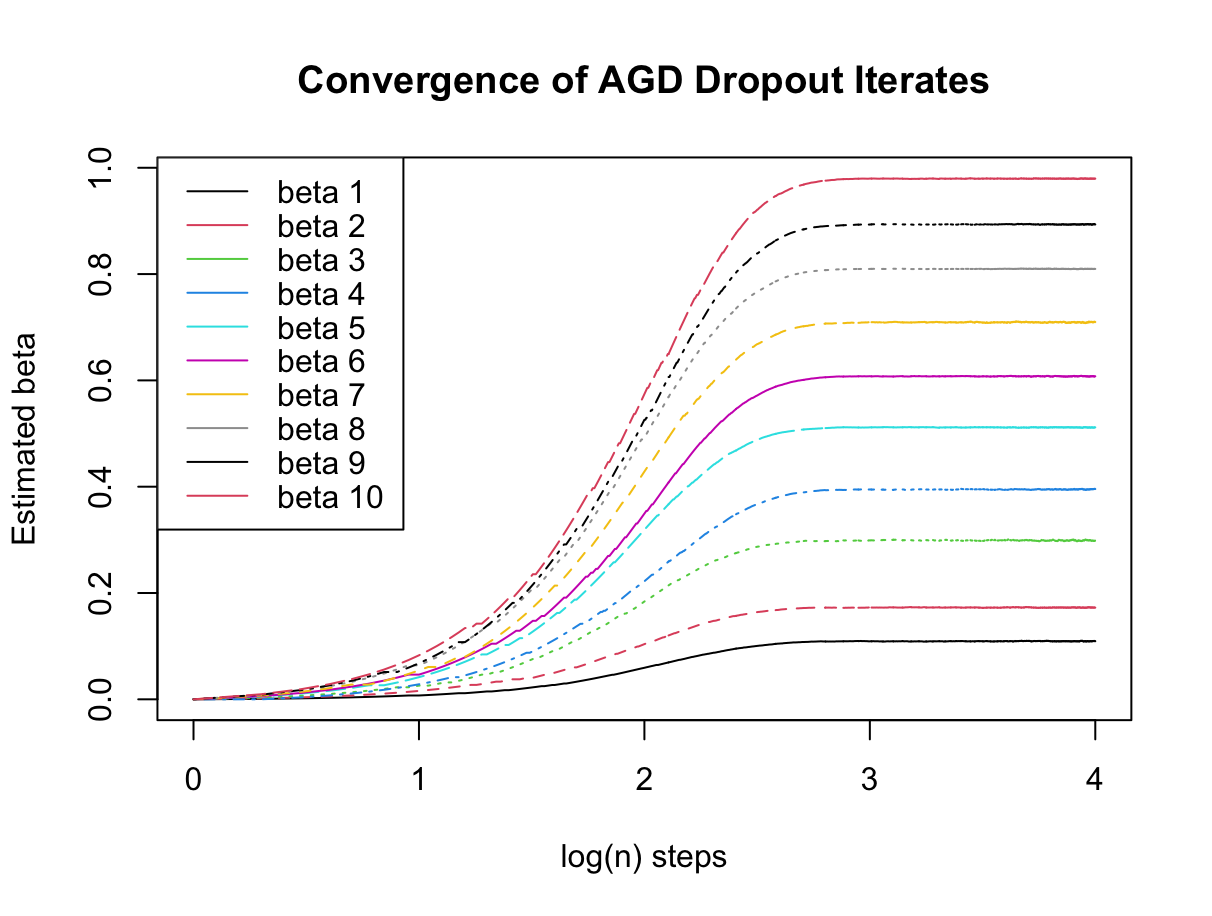}
        \caption{AGD.}
    \end{subfigure}
    \begin{subfigure}[b]{0.49\textwidth}
        \centering
        \includegraphics[width=1\textwidth]{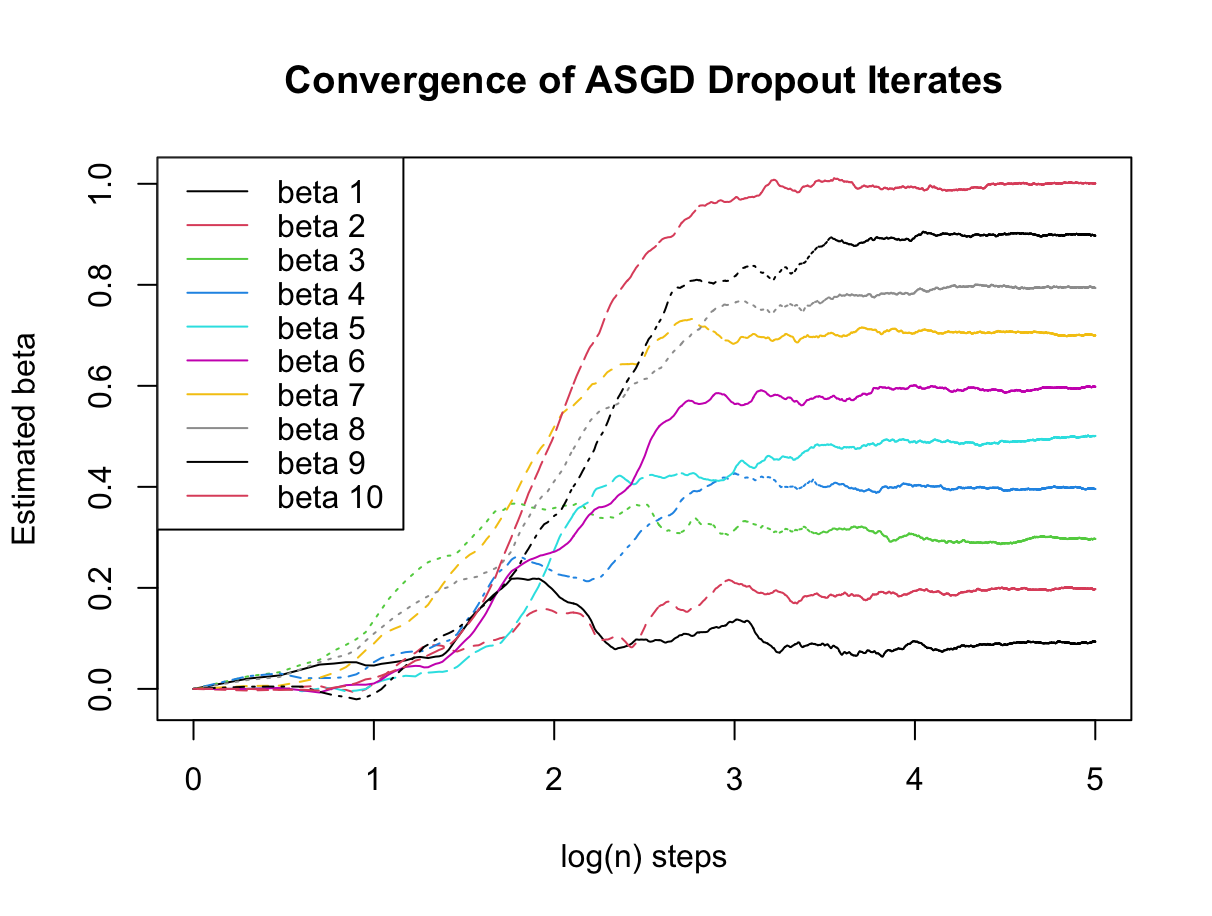}
        \caption{ASGD.}
    \end{subfigure}
    \caption{Convergence traces of AGD and ASGD iterates with dropout regularization based on a single run, with dimension $d=10$ and initialization at zero. The coordinates of the true parameter $\bbeta^*$ are equidistantly spaced between 0 and 1, the learning rate $\alpha=0.01$, and the retaining probability $p=0.9$. Each curve represents the convergence trace of one coordinate.}
    \label{fig:gmc_example}
\end{figure}

Figure~\ref{fig:var_asgd_convergence} evaluates the performance of the online long-run covariance estimator $\hat\Sigma_n(\alpha)$ in the same setting as considered before. As there is no closed-form expression for the true long-run covariance matrices $\breve\Sigma(\alpha)$ defined in Theorem \ref{thm_clt_ave_sgd}, we shall only report the convergence trace of the estimated long-run covariance matrix. For the non-overlapping blocks $B_m=\{\eta_m,\eta_{m+1},\ldots,\eta_{m+1}-1\}$, $m=1,\ldots,M$, defined in (\ref{eq_block}), the number of blocks is $M=\lfloor\sqrt{n}\rfloor$ and $\eta_m=m^2$. In Figure~\ref{fig:var_asgd_convergence}, we can see that the long-run variances of each coordinate of the ASGD dropout iterates $\bar\bbeta_n^{\sgd}(\alpha)$ converges as the number of iterations $n$ grows. The length of the joint confidence interval for the one-dimensional projection $\bm{v}^{\top}\bbeta^*$ is also shown in Figure~\ref{fig:asgd_CI_length}, where we set each coordinate of the unit-length vector $\bm{v}$ to be $1/\sqrt{d}$. In the next section, we shall show that by using these estimated long-run variances, the online confidence intervals achieve asymptotically the nominal coverage probability.
\begin{figure}[!hbt]
    \centering
    \includegraphics[width=0.95\textwidth]{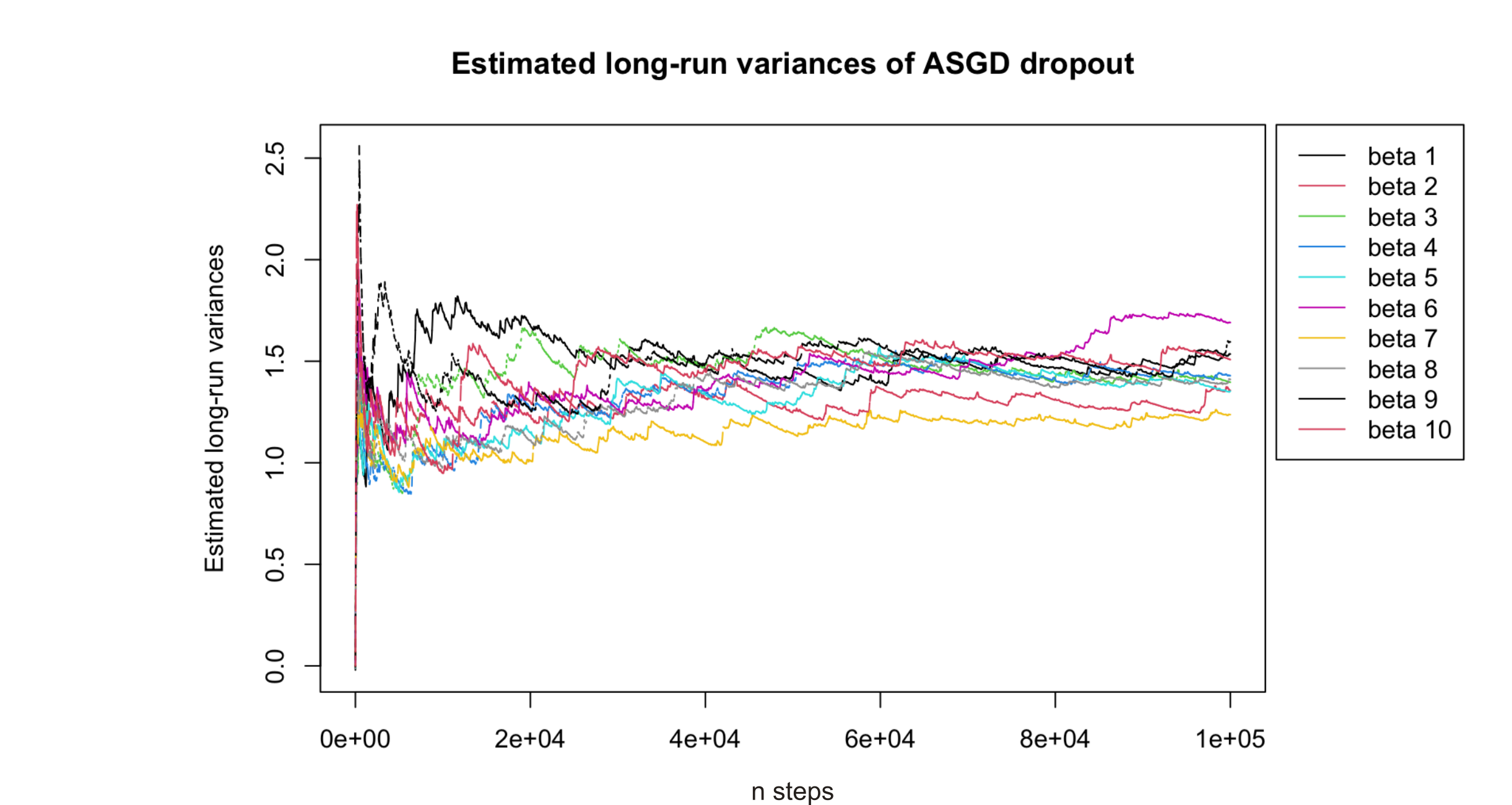}
    \caption{Estimated long-run variances of ASGD dropout iterates, i.e., diagonals of the estimated long-run covariance matrix $\hat\Sigma_n(\alpha)$ for the same setting as in Figure \ref{fig:gmc_example}.}
    \label{fig:var_asgd_convergence}
\end{figure}

\begin{figure}[!hbt]
    \centering
    \includegraphics[width=0.71\textwidth]{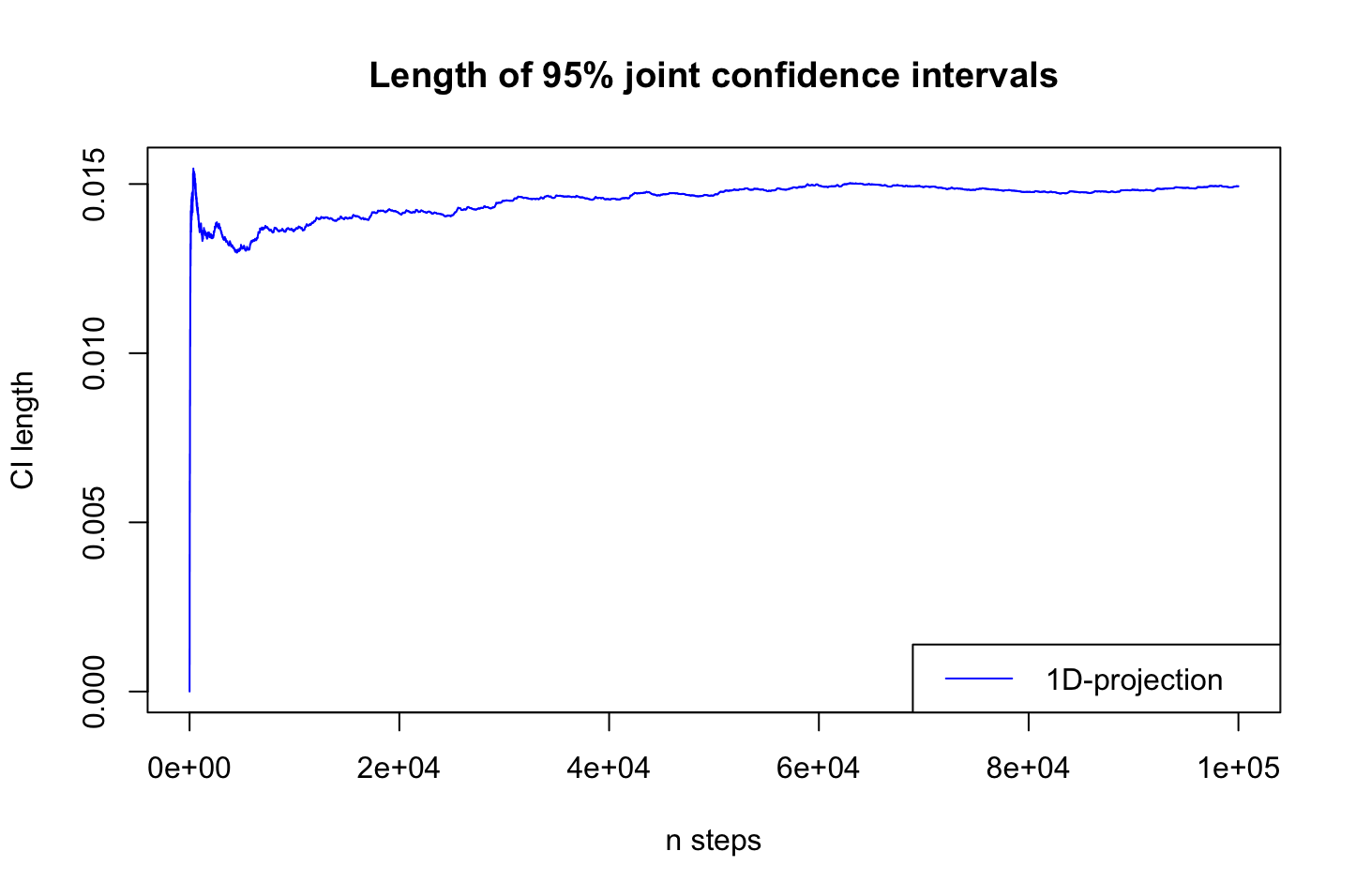}
    \caption{Length of the joint CI for the one-dimensional projection $\bm{v}^{\top}\bbeta^*$ of the ASGD dropout iterates for the same setting as in Figure \ref{fig:gmc_example}.}
    \label{fig:asgd_CI_length}
\end{figure}

\subsection{Online Confidence Intervals of ASGD Dropout Iterates}

Recall the $100(1-\omega)\%$ online confidence interval $\mathrm{CI}_{\omega,n,j}$ in (\ref{eq_CI_1d}) for each coordinate $\bbeta_j^*$ of the true parameter $\bbeta^*$, for $j=1,\ldots,d$. Let dimension $d=10$. We constructed the confidence interval $\mathrm{CI}_{\omega,n,j}$ for each $j=1,\ldots,d$, and averaged the coverage probabilities of $\mathrm{CI}_{\omega,n,j}$ over $j$. As shown in Figure~\ref{fig:asgd_CI_prob_d10}, the averaged coverage probabilities converge to the nominal coverage rate 0.95 as the number of steps $n$ increases. 

Furthermore, recall the $100(1-\omega)\%$ joint online confidence interval $\mathrm{CI}_{\omega,n}^{\mathrm{Proj}}$ in (\ref{eq_CI_projection}) for the one-dimensional projection of the true parameter $\bbeta^*$, i.e., $\bm{v}^{\top}\bbeta^*$. Let dimension $d=50$. A similar performance in convergence of the coverage probability is observed in Figure~\ref{fig:asgd_CI_prob_d50}. In Tables~\ref{table_prob_d3}--\ref{table_prob_d50}, we report the coverage probabilities of the joint confidence intervals $\mathrm{CI}_{\omega,n}^{\mathrm{Proj}}$ under different settings. In particular, we consider the dimensions $d=3,20$ and $50$, the retaining probabilities of the dropout regularization $p=0.9$ and $0.5$, and the constant learning rates $\alpha$ ranging from $0.01$ to $0.1$. All the results demonstrate the effectiveness of our proposed online inference method.

\clearpage

\begin{figure}[!hbt]
    \centering
    \includegraphics[width=0.8\textwidth]{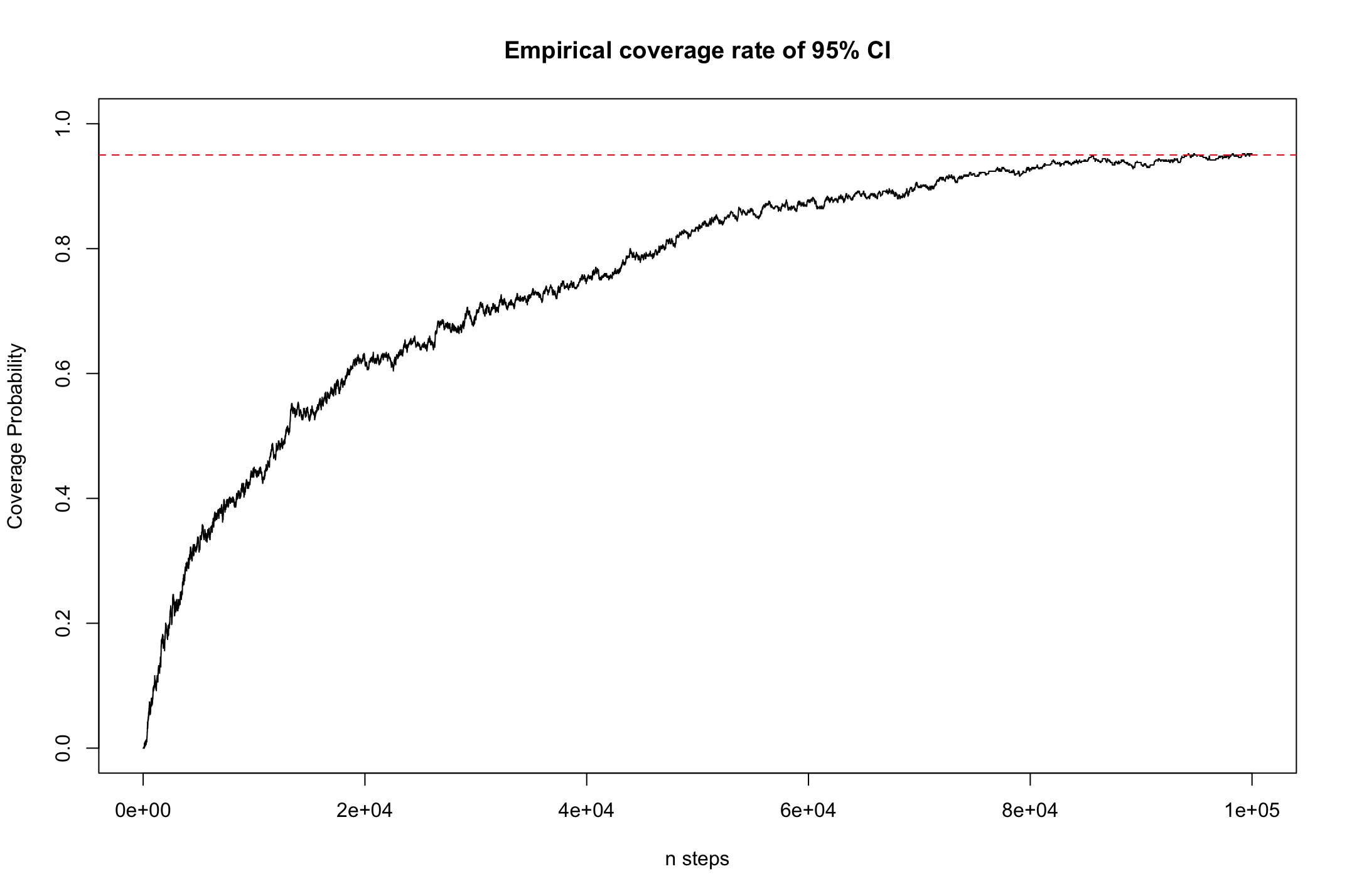}
    \caption{Coverage probabilities of 95\% CI for ASGD dropout iterates \textit{averaged over $d$ coordinates} from 200 independent runs. Red dashed line denotes the nominal coverage rate of 0.95. Dimension $d=10$, $p=0.9$, $\alpha=0.01$, and coordinates of $\bbeta^*$ are equidistantly spaced between 0 and 1 with initializations at zero.}
    \label{fig:asgd_CI_prob_d10}
\end{figure}

\begin{figure}[!hbt]
    \centering
    \includegraphics[width=0.8\textwidth]{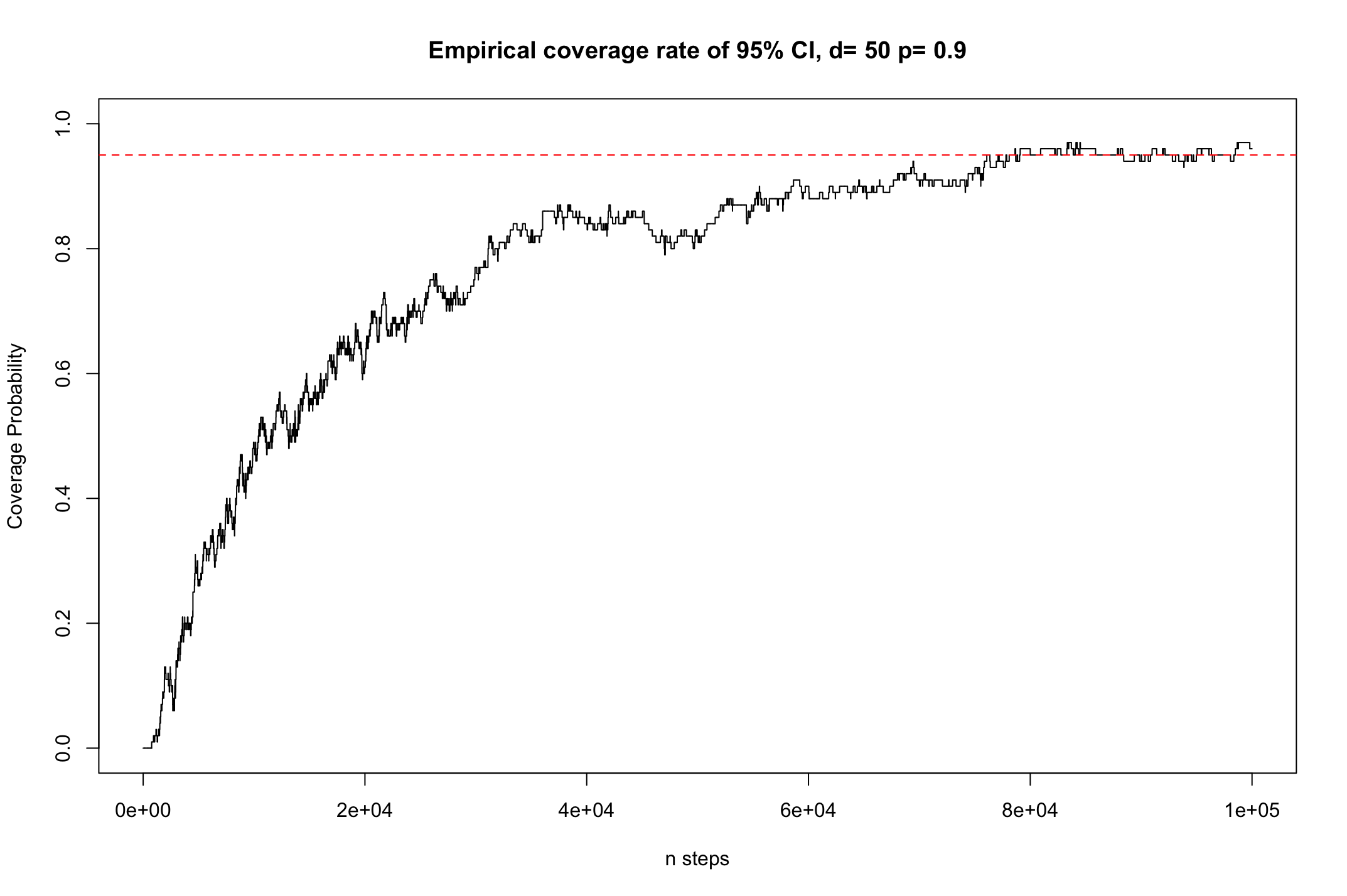}
    \caption{Coverage probabilities of 95\% \textit{joint confidence intervals} for one-dimensional projection of ASGD dropout iterates from 200 independent runs. Red dashed line denotes the nominal coverage rate of 0.95. Dimension $d=50$, $p=0.9$, $\alpha=0.01$, and coordinates of $\bbeta^*$ are equidistantly spaced between 0 and 1 with initializations at zero.}
    \label{fig:asgd_CI_prob_d50}
\end{figure}

\clearpage

\begin{table}[htb!]
    \centering
    \begin{tabular}{ccccccc}
    \hline\hline
         & $\alpha$ & $n=100,000$ & $n=150,000$ & $n=180,000$ & $n=190,000$ & $n=200,000$ \\
         \hline
         & $0.01$ & 0.815 (0.0275) & 0.875 (0.0234) & 0.920 (0.0192) & 0.940 (0.0168) & 0.945 (0.0161) \\
    $p=0.9$ & $0.05$ & 0.800 (0.0283) & 0.895 (0.0217) & 0.940 (0.0168) & 0.925 (0.0186) & 0.930 (0.0180) \\
         & $0.1$ & 0.830 (0.0266) & 0.915 (0.0197) & 0.930 (0.0180) & 0.955 (0.0146) & 0.960 (0.0138) \\
    \hline\hline
         & $\alpha$ & $n=200,000$ & $n=250,000$ & $n=280,000$ & $n=290,000$ & $n=300,000$ \\
         \hline
         & $0.01$  & 0.850 (0.0253) & 0.900 (0.0212) & 0.910 (0.0202) & 0.915 (0.0197) & 0.925 (0.0186) \\
    $p=0.5$ & $0.05$  & 0.860 (0.0245) & 0.890 (0.0221) & 0.920 (0.0192) & 0.925 (0.0186) & 0.945 (0.0161) \\
         & $0.1$  & 0.910 (0.0202) & 0.900 (0.0212) & 0.940 (0.0168) & 0.950 (0.0154) & 0.950 (0.0154) \\
        \hline\hline
    \end{tabular}
    \caption{Empirical coverage probability of $95\%$ confidence intervals from 200 independent runs (with standard errors in the brackets). Dimension $d=3$.}
    \label{table_prob_d3}
\end{table}

\begin{table}[htb!]
    \centering
    \begin{tabular}{ccccccc}
    \hline\hline
         & $\alpha$ & $n=100,000$ & $n=150,000$ & $n=180,000$ & $n=190,000$ & $n=200,000$ \\
         \hline
         & $0.01$ & 0.855 (0.0249) & 0.885 (0.0226) & 0.935 (0.0174) & 0.935 (0.0174) & 0.945 (0.0161) \\
    $p=0.9$ & $\alpha=0.025$ & 0.810 (0.0278) & 0.895 (0.0217) & 0.935 (0.0174) & 0.950 (0.0154) & 0.960 (0.0138) \\
         & $0.05$ & 0.875 (0.0234) & 0.895 (0.0217) & 0.905 (0.0207) & 0.900 (0.0212) & 0.930 (0.0180) \\
    \hline\hline
         & $\alpha$ & $n=200,000$ & $n=250,000$ & $n=280,000$ & $n=290,000$ & $n=300,000$ \\
         \hline
         & $0.01$  & 0.845 (0.0256) & 0.890 (0.0221) & 0.925 (0.0186) & 0.930 (0.0180) & 0.935 (0.0174) \\
    $p=0.5$ & $0.025$  & 0.900 (0.0212) & 0.950 (0.0154) & 0.945 (0.0161) & 0.955 (0.0146) & 0.955 (0.0146) \\
         & $0.05$  & 0.880 (0.0230) & 0.950 (0.0154) & 0.955 (0.0146) & 0.960 (0.0138) & 0.960 (0.0138) \\
        \hline\hline
    \end{tabular}
    \caption{Empirical coverage probability of $95\%$ confidence intervals from 200 independent runs (with standard errors in the brackets). Dimension $d=20$.}
    \label{table_prob_d20}
\end{table}

\begin{table}[htb!]
    \centering
    \begin{tabular}{ccccccc}
    \hline\hline
         & $\alpha$ & $n=100,000$ & $n=150,000$ & $n=180,000$ & $n=190,000$ & $n=200,000$ \\
         \hline
         & $0.01$ & 0.850 (0.0253) & 0.935 (0.0174) & 0.945 (0.0161) & 0.925 (0.0186) & 0.950 (0.0154) \\
    $p=0.9$ & $0.0125$ & 0.840 (0.0259) & 0.945 (0.0161) & 0.940 (0.0168) & 0.940 (0.0168) & 0.945 (0.0161) \\
         & $0.02$ & 0.750 (0.0306) & 0.910 (0.0202) & 0.925 (0.0186) & 0.925 (0.0186) & 0.945 (0.0161) \\
    \hline\hline
         & $\alpha$ & $n=200,000$ & $n=250,000$ & $n=280,000$ & $n=290,000$ & $n=300,000$ \\
         \hline
         & $0.01$  & 0.805 (0.0280) & 0.875 (0.0234) & 0.895 (0.0217) & 0.920 (0.0192) & 0.910 (0.0202) \\
    $p=0.5$ & $0.0125$  & 0.735 (0.0230) & 0.930 (0.0180) & 0.920 (0.0192) & 0.930 (0.0180) & 0.925 (0.0186) \\
         & $0.02$  & 0.870 (0.0238) & 0.910 (0.0202) & 0.920 (0.0192) & 0.915 (0.0197) & 0.930 (0.0180) \\
        \hline\hline
    \end{tabular}
    \caption{Empirical coverage probability of $95\%$ confidence intervals from 200 independent runs (with standard errors in the brackets). Dimension $d=50$.}
    \label{table_prob_d50}
\end{table}

\section*{Acknowledgements}
Wei Biao Wu’s research is partially supported by the NSF (Grants NSF/DMS-2311249, NSF/DMS-2027723). Johannes Schmidt-Hieber has received funding from the Dutch Research Council (NWO) via the Vidi
grant VI.Vidi.192.021.

\printbibliography

\newpage

\begin{appendix}

\section{Some Useful Lemmas}

\begin{lemma}[\cite{burkholder1988,rio_moment_2009}]
    \label{lemma_burkholder}
    Let $q>1,q'=\min\{q,2\}$, and $M_T=\sum_{t=1}^{T}\xi_t$, where $\xi_t$ are martingale differences with a finite $q$-th moment. Then 
    $$\big(\EE\|M_T\|_2^q\big)^{q'/q}\le K_{q}^{q'}\sum_{t=1}^{T}\big(\EE\| \xi_T\|_2^q\big)^{q'/q},\quad\text{where } K_q=\max\{(q-1)^{-1},\sqrt{q-1}\}.$$
\end{lemma}

\begin{lemma}[\cite{clara_dropout_2023}]\label{lemma_clara}
    For any matrices $A$ and $B$ in $\RR^{d\times d}$, $p\in(0,1)$, and a diagonal matrix $D\in\RR^{d\times d}$, the following results hold:\\ 
    \indent (i) $\overline{AD} = \overline{A}D$, $\overline{DA}=D\overline{A}$, and $\overline{A}_p = p\overline{A}=\overline{A_p}$; \\ 
    If in addition, the diagonal matrix $D$ is random and independent of $A$ and $B$, with the diagonal entries satisfying $D_{ii}\overset{\mathrm{i.i.d.}}{\sim}\mathrm{Bernoulli}(p)$, $1\le i\le d$, then, \\
    \indent (ii) $\EE[DAD]=pA_p$, where $A_p=pA+(1-p)\mathrm{Diag}(A)$;\\
    \indent (iii) $\EE[DADBD]=pA_pB_p + p^2(1-p)\mathrm{Diag}(\overline{A}B)$, where $\overline{A}=A-\mathrm{Diag}(A)$; \\
    \indent (iv) $\EE[DADBDCD]=pA_pB_pC_p + p^2(1-p)\big[\mathrm{Diag}(\overline{A}B_p\overline{C}) + A_p\mathrm{\Diag}(\overline{B}C) + \mathrm{\Diag}(A\overline{B})C_p + (1-p)A\odot\overline{B}^{\top}\odot C\big]$, where $\odot$ denotes the Hadamard product.
\end{lemma}

\begin{lemma}[Properties of operator norm]\label{lemma_operator_norm}
    Let $A=(a_{ij})_{1\le i,j\le d}$ be a real $d\times d$ matrix. View $A$ as a linear map $\RR^d\mapsto\RR^d$ and denote its operator norm by $\|A\|$.\\
    \indent (i) (Inequalities for variants of $A$). $\|\mathrm{Diag}(A)\|\le\|A\|$, $\|A_p\|\le\|A\|$, and if in addition, $A$ is positive semi-definite, then also $\|\overline{A}\|\le \|A\|$; \\
    \indent (ii) (Frobenius norm). $\|A\|=\sup_{\bm{v}\in\RR^d,\|\bm{v}\|_2=1}\|A \bm{v}\|_2 \le \|A\|_F$, where $\|A\|_F$ denotes the Frobenius norm, i.e., $\|A\|_F=\big(\sum_{i,j=1}^d|a_{ij}|^2\big)^{1/2}$; \\
    \indent (iii) (Largest magnitude of eigenvalues). For a symmetric $d\times d$ matrix $A,$ $\max_{1\le i\le d}|\lambda_i(A)| = \|A\|$, where $\lambda_i(A)$ denotes the $i$-th largest eigenvalue of $A$. If in addition, $A$ is positive semi-definite, then also $\lambda_{\max}(A)=\|A\|$.
\end{lemma}

\begin{proof}[Proof of Lemma~\ref{lemma_operator_norm}]
    The inequalities in (i) follow directly from Lemma 19 in \textcite{clara_dropout_2023}. For (ii), we notice that for any unit vector $\bm{v}\in\RR^d$, one can find a basis $\{\bm{e}_1,\ldots,\bm{e}_d\}$ and write $\bm{v}$ into $\bm{v}=\sum_{j=1}^dc_j\bm{e}_j$, with $\bm{e}_j\in\RR^d$, and the real coefficients $c_j$ satisfying $\sum_{j=1}^dc_j^2=1$. Then, it follows from the orthogonality of $\bm{e}_j$ and the Cauchy-Schwarz inequality that
    \begin{align}
        \|A\bm{v}\|_2^2 = \Big\|\sum_{j=1}^dc_jA\bm{e}_j\Big\|_2^2 \le \Big(\sum_{j=1}^d|c_j|\|A\bm{e}_j\|_2\Big)^2 \le \Big(\sum_{j=1}^d|c_j|^2\Big)\Big(\sum_{j=1}^d\|A\bm{e}_j\|_2^2\Big)= \sum_{j=1}^d\|A\bm{e}_j\|_2^2 = \|A\|_F^2.
    \end{align}
    Since this result holds for any unit vector $\bm{v}\in\RR^d$, the desired result in (ii) is achieved.

    To see (iii), for any eigenvalue $\lambda_i(A)$, we denote its associated unit eigenvector by $\bm{v}$. Then, $\|A\bm{v}\|_2=|\lambda_i(A)|\|\bm{v}\|_2=|\lambda_i(A)|$, which further yields $|\lambda_i(A)| \le \sup_{\bm{v}\in\RR^d,\|\bm{v}\|_2=1}\|A\bm{v}\|_2 = \|A\|$, uniformly over $i$. Hence, the inequality can be obtained. If in addition, $A$ is symmetric, then $A$ can be diagonalized by an orthogonal matrix $Q$ and a diagonal matrix $\Lambda$ such that $A=Q^{\top}\Lambda Q$. Therefore, $\sup_{\bm{v}\in\RR^d,\|\bm{v}\|_2=1}\|A\bm{v}\|_2 = \max_{1\le i\le d}|\lambda_i(A)|$, which completes the proof.
\end{proof}

\section{Proofs in Section~\ref{subsec_gmc}}

This section is devoted to the proofs of the geometric-moment contraction (GMC) for the dropout iterates with gradient descent (GD), i.e., $\{\tilde\bbeta_k(\alpha)-\bbeta\}_{k\in\NN}$. We first extend the results in \textcite{wu_limit_2004} to the cases where the inputs of
iterated random functions are i.i.d.\ random matrices. Then, we present the proof for the sufficient condition of the GMC in terms of the constant learning rate $\alpha$ in Lemma~\ref{lemma_gmc_cond}, and showcase the GMC of $\{\tilde\bbeta_k(\alpha)-\bbeta\}_{k\in\NN}$ in Theorem~\ref{thm_gmc_gd}.

\subsection{GMC -- Random Matrix Version}\label{subsec_gmc_mat}

Let $(\mathcal{Y},\rho)$ be a complete and separable metric space, endowed with its Borel sets $\mathbb{Y}$. Consider an iterated random function on the state space $\mathcal{Y}\subset\RR^d$, for some fixed $d\ge1$, with the form
\begin{equation}
    \label{eq_iterated_function}
    \bm{y}_k = f(\bm{y}_{k-1},X_k) = f_{X_k}(\bm{y}_{k-1}), \quad k\in\NN,
\end{equation}
where $f(\bm{y},\cdot)$ is the $\bm{y}$-section of a jointly measurable function $f:\,\mathcal{Y}\times \mathcal{X}\mapsto\mathcal{Y}$; the random matrices $X_k$, $k\in\NN$, take values in a second measurable space $\mathcal{X}\subset\RR^{d\times d}$, and are independently distributed with identical marginal distribution $H$. The initial point $\bm{y}_0\in\mathcal{Y}$ is independent of all $X_k$. 

We are interested in the sufficient conditions on $f_X(\bm{y})$ such that there is a unique stationary probability $\pi$ on $\mathcal{Y}$ with $\bm{y}_k\Rightarrow\pi$ as $k\rightarrow\infty$. To this end, define a composite function
\begin{equation}
    \bm{y}_k(\bm{y}) = f_{X_k}\circ f_{X_{k-1}}\circ \cdots \circ f_{X_1}(\bm{y}), \quad \text{for }\bm{y}\in\mathcal{Y}.
\end{equation}
We say that $\bm{y}_k$ is \textit{geometric-moment contracting} if for any two independent random vectors $\bm{y}\sim\pi$ and $\bm{y}'\sim\pi$ in $\mathcal{Y}$, there exist some $q>0$, $C_q>0$ and $r_q\in(0,1)$, such that for all $k\in\NN$,
\begin{equation}
    \label{eq_gmc_mat}
    \EE\big[\rho^q\big(\bm{y}_k(\bm{y}), \bm{y}_k(\bm{y}')\big)\big] \le C_q r_q^k.
\end{equation}
\textcite{wu_limit_2004} provided the sufficient conditions for~(\ref{eq_gmc_mat}) when $X_k$ are random variables and $\bm{y}_k$ and $\bm{y}$ are one-dimensional. Their results can be directly extended to the random matrix version and we state them here for the completeness of this paper.

\begin{assumption}[Finite moment]\label{asm_finite_moment_mat}
    Assume that there exists a fixed vector $\bm{y}^*\in\mathcal{Y}$ and some $q>0$ such that
    $$I(q,\bm{y}^*):= \EE_{X\sim H}\big[\rho^q\big(\bm{y}^*,f_X(\bm{y}^*)\big)\big] = \int_{\mathcal{X}}\rho^q\big(\bm{y}^*,f_X(\bm{y}^*)\big)H(dX) < \infty.$$
\end{assumption}

\begin{assumption}[Stochastic Lipschitz continuity]\label{asm_Lip_mat}
    Assume that there exists some $q>0$ and some $\bm{y}_0\in\mathcal{Y}$ such that
    $$L_q:=\sup_{\bm{y}_0'\in\mathcal{Y},\,\bm{y}_0\neq\bm{y}_0'}\frac{\EE_{X\sim H}\big[\rho^q\big(f_X(\bm{y}_0),f_X(\bm{y}_0')\big)\big]}{\rho^q(\bm{y}_0,\bm{y}_0')} <1,$$
    where $L_q=L_q(\bm{y}_0)$ is a local Lipschitz constant.
\end{assumption}

\begin{corollary}[GMC -- random matrix version]\label{cor_gmc_mat}
    Suppose that Assumptions~\ref{asm_finite_moment_mat}~and~\ref{asm_Lip_mat} hold. Define a backward iteration process 
    \begin{equation}
        \label{eq_backward}
        \bm{z}_k(\bm{y}) = \bm{z}_{k-1}(f_{X_k}(\bm{y}))=f_{X_1}\circ f_{X_2}\circ \cdots \circ f_{X_k}(\bm{y}), \quad \text{for }\bm{y}\in\mathcal{Y}.
    \end{equation}
    Then, $\bm{z}_k(\bm{y})\overset{\mathcal{D}}{=}\bm{y}_k(\bm{y})$, and there exists a random vector $\bm{z}_{\infty}\in\mathcal{Y}$ such that for any $\bm{y}\in\mathcal{Y}$, 
    $$\bm{z}_k(\bm{y}) \overset{a.s.}{\rightarrow} \bm{z}_{\infty}.$$
    The limit $\bm{z}_{\infty}$ is measurable with respect to the $\sigma$-algebra $\sigma(X_1,X_2,\ldots)$ and does not depend on $\bm{y}$. In addition,
    \begin{equation}
        \label{eq_gd_result_z}
        \EE\big[\rho^q\big(\bm{z}_k(\bm{y}), \bm{z}_{\infty}\big)\big] \le Cr_q^k, \quad k\in\NN,
    \end{equation}
    where the constant $C>0$ only depends on $q$ and $\bm{y}^*$ in Assumption~\ref{asm_finite_moment_mat}, $L_q$ in Assumption~\ref{asm_Lip_mat} and $\bm{y}_0$. Consequently, (\ref{eq_gmc_mat}) holds.
\end{corollary}

\begin{remark}[Backward iteration]
We shall comment on the intuition for defining the backward iteration $\bm{z}_k$ in~(\ref{eq_backward}). Recall the i.i.d.\ random samples $X_1,\ldots,X_n$. Clearly, for any fixed initial point $\bm{y}_0\in\mathcal{Y}$, for all $k\in\NN$, we have the relations
\begin{align*}
    \bm{y}_{k+1}(\bm{y}_0) & = f_{X_{k+1}} \big( \bm{y}_k(\bm{y}_0)\big), \nonumber \\
    \bm{z}_{k+1}(\bm{y}_0) & = \bm{z}_k \big(f_{X_{k+1}}(\bm{y}_0)\big).
\end{align*}
To prove the existence of the limit for $\bm{y}_k = f_{X_k}\circ f_{X_{k-1}}\circ \cdots \circ f_{X_1}(\bm{y}_0)$, we need to make use of the contracting property of the function $f_X(\cdot)$ stated in Assumption~\ref{asm_Lip_mat}. However, we cannot directly apply it to the forward iteration, because by the Markov property, given the present position of the chain, the conditional distribution of the future does not depend on the past. This indicates
\begin{align}
    \EE\big[\rho^q\big(\bm{y}_{k+1}(\bm{y}^*), \bm{y}_k(\bm{y}^*)\big)\big] & = \EE\big[\EE\big[\rho^q\big(f_{X_{k+1}} \big(\bm{y}_k(\bm{y}^*)\big), \bm{y}_k(\bm{y}^*)\big) \mid X_{k+1}\big]\big],
\end{align}
where the two parts inside of $\rho(\cdot,\cdot)$ are operated by two different functions, which are $f_{X_{k+1}}(\cdot)$ and $f_{X_k}\circ f_{X_{k-1}}\circ \cdots \circ f_{X_1}(\cdot)$ respectively. In fact, as pointed out by \textcite{diaconis_iterated_1999}, the forward iteration $\bm{y}_k$ moves ergodically through $\mathcal{Y}$, which behaves quite differently from the backward iteration $\bm{z}_k(\cdot)=f_{X_1}\circ f_{X_2}\circ\cdots\circ f_{X_k}(\cdot)$ in~(\ref{eq_backward}), which does converge to a limit. To see this, we note that by Assumptions~\ref{asm_finite_moment_mat}~and~\ref{asm_Lip_mat}, there exists some $\bm{y}^*\in\mathcal{Y}$ such that
\begin{align}
    \label{eq_backward_onestep}
    \EE\big[\rho^q\big(\bm{z}_{k+1}(\bm{y}^*), \bm{z}_k(\bm{y}^*)\big)\big] & = \EE\big[\EE\big[\rho^q\big(\bm{z}_k \big(f_{X_{k+1}}(\bm{y}^*)\big), \bm{z}_k(\bm{y}^*)\big) \mid X_{k+1}\big]\big] \nonumber \\
    & \le L_q^k\EE\big[\rho^q\big(f_{X_{k+1}}(\bm{y}^*),\bm{y}^*\big)\big] \nonumber \\
    & = L_q^k I(q,\bm{y}^*),
\end{align}
which is summable over $k$ by Assumption~\ref{asm_Lip_mat}. Since $\mathcal{Y}$ is a complete space, we can mimic the idea of a Cauchy sequence to prove the existence of the limit $\bm{z}_{\infty}$ and further show $\bm{z}_k\overset{a.s.}{\rightarrow}\bm{z}_{\infty}$, by applying the Borel-Cantelli lemma. Since $X_1,\ldots,X_n$ are i.i.d.\ and thus exchangeable, we have $\bm{z}_k(\bm{y}_0)\overset{\mathcal{D}}{=}\bm{y}_k(\bm{y}_0)$. Hence, we can show that $\bm{y}_k$ also converges to $\bm{z}_{\infty}$ in distribution. 
\end{remark}

\begin{proof}[Proof of Corollary \ref{cor_gmc_mat}]
Let $q\in(0,1]$ such that both Assumptions~\ref{asm_finite_moment_mat}~and~\ref{asm_Lip_mat} hold. We will only show the desired results for this choice of $q$, since if Assumptions~\ref{asm_finite_moment_mat}~and~\ref{asm_Lip_mat} are satisfied for some $q>1$, then they are also valid for all $q\le1$ by Hölder's inequality (\cite{wu_limit_2004}). Recall the definition of integral $I(q,\bm{y}^*)$ in Assumption~\ref{asm_finite_moment_mat}. Let $\bm{y}_0\in\mathcal{Y}$ satisfy Assumption~\ref{asm_Lip_mat}. Then,
\begin{align*}
    I(q,\bm{y}_0) & = \EE\big[\rho^q\big(\bm{y}_0,f_X(\bm{y}_0)\big)\big] \nonumber \\
    & \le \EE\big[\rho(\bm{y}_0,\bm{y}^*) + \rho\big(\bm{y}^*,f_X(\bm{y}^*)\big) + \rho\big(f_X(\bm{y}^*),f_X(\bm{y}_0)\big)\big]^q \nonumber \\
    & \le \rho^q(\bm{y}_0,\bm{y}^*) + I(q,\bm{y}^*) + \EE\big[\rho^q\big(f_X(\bm{y}^*),f_X(\bm{y}_0)\big)\big] \nonumber \\
    & \le \rho^q(\bm{y}_0,\bm{y}^*) + I(q,\bm{y}^*) + L_q\rho^q(\bm{y}^*,\bm{y}_0) <\infty,
\end{align*}
where the first inequality follows from the triangle inequality, the second one is by Assumption~\ref{asm_Lip_mat} and Jensen's inequality, and the last one is due to Assumption~\ref{asm_finite_moment_mat}. A similar argument as in~(\ref{eq_backward_onestep}) yields
\begin{align}
    \EE\big[\rho^q\big(\bm{z}_{k+1}(\bm{y}_0), \bm{z}_k(\bm{y}_0)\big)\big] \le L_q^k I(q,\bm{y}_0) =:\delta_k,
\end{align}
where $\delta_k=\delta_k(q,\bm{y}_0)$ solely depends on $k$, $q$, $L_q$ and $\bm{y}_0$. By Markov's inequality, we have
\begin{align}
    \PP\Big(\rho\big(\bm{z}_{k+1}(\bm{y}_0),\bm{z}_k(\bm{y}_0)\big)\ge \delta_k^{1/(2q)}\Big) \le \delta_k^{1/2}.
\end{align}
Since $\sum_{k=1}^{\infty}\delta_k^{1/2}<\infty$, it follows from the first Borel-Cantelli lemma that
\begin{align}
    \PP\Big(\rho\big(\bm{z}_{k+1}(\bm{y}_0),\bm{z}_k(\bm{y}_0)\big)\ge \delta_k^{1/(2q)} \,\text{for infinitely many }k\Big) = 0.
\end{align}
Again, since $\delta_k^{1/2}$ is summable, $\bm{z}_k$ is a Cauchy sequence in space $\mathcal{Y}$, which together with the completeness of $\mathcal{Y}$ gives that almost surely, there exists a random vector $\bm{z}_{\infty}\in\mathcal{Y}$ such that
\begin{align*}
    \bm{z}_k(\bm{y}_0) \overset{a.s.}{\rightarrow} \bm{z}_{\infty}, \quad \text{as }k\rightarrow\infty,
\end{align*}
where $\bm{z}_{\infty}$ is $\sigma(X_1,X_2,\ldots)$-measurable. Let $\pi$ be the probability distribution of $\bm{z}_{\infty}$.

Furthermore, it follows from the triangle inequality and Jensen's inequality that for any fixed $\bm{y}_0\in\mathcal{Y}$, 
\begin{align}
    \EE\big[\rho^q\big(\bm{z}_k(\bm{y}_0),\bm{z}_{\infty}\big)\big] & \le \EE\Big[\sum_{l=0}^{\infty}\rho\big(\bm{z}_{k+1+l}(\bm{y}_0),\bm{z}_{k+l}(\bm{y}_0)\big)\Big]^q \nonumber \\
    & \le \sum_{l=0}^{\infty}\EE\big[\rho^q\big(\bm{z}_{k+1+l}(\bm{y}_0),\bm{z}_{k+l}(\bm{y}_0)\big)\big] \nonumber \\
    & \le \delta_k/(1-L_q),
\end{align}
For any $\bm{y}\in\mathcal{Y}$, by Assumption~\ref{asm_Lip_mat} and triangle inequality,
\begin{align}
    \EE\big[\rho^q\big(\bm{z}_k(\bm{y}),\bm{z}_{\infty}\big)\big] & \le \EE\big[\rho^q\big(\bm{z}_k(\bm{y}),\bm{z}_k(\bm{y}_0)\big)\big]  + \EE\big[\rho^q\big(\bm{z}_k(\bm{y}_0),\bm{z}_{\infty}\big)\big] \nonumber \\
    & \le L_q^k\rho^q(\bm{y}_0,\bm{y}) + \delta_k/(1-L_q).
\end{align}
Recall that $\delta_k=L_q^kI(q,\bm{y}_0)$. Let $C= I(q,\bm{y}_0)/(1-L_q) + \rho^q(\bm{y}_0,\bm{y})$ and we have shown result~(\ref{eq_gd_result_z}) with $r_q=L_q$. Since $Cr_q^k$ in~(\ref{eq_gd_result_z}) is summable over $k$, it again follows from Borel-Cantelli lemma that for any $\bm{y}\in\mathcal{Y}$,
\begin{align*}
    \bm{z}_k(\bm{y}) \overset{a.s.}{\rightarrow} \bm{z}_{\infty}, \quad \text{as }k\rightarrow\infty,
\end{align*}
and therefore, the limit
\begin{equation}
    \label{eq_operator_v}
    \bm{v}_k(\bm{y}) = \lim_{m\rightarrow\infty}f_{X_{k+1}}\circ f_{X_{k+2}}\circ \cdots \circ f_{X_{k+m}}(\bm{y})
\end{equation}
exists almost surely. 

Finally, we notice that for any two independent random vectors $\bm{y}\sim\pi$ and $\bm{y}'\sim\pi$,
\begin{align}
    \EE\big[\rho^q\big(\bm{y}_k(\bm{y}),\bm{y}_k(\bm{y}')\big)\big] & \le \EE\big[\rho^q\big(\bm{y}_k(\bm{y}),\bm{y}_k(\bm{y}_0)\big)\big] + \EE\big[\rho^q\big(\bm{y}_k(\bm{y}_0),\bm{y}_k(\bm{y}')\big)\big] \nonumber \\
    & = 2\EE\big[\rho^q\big(\bm{z}_k(\bm{v}_k),\bm{z}_k(\bm{y}_0)\big)\big] \nonumber \\
    & = 2\EE\big[\rho^q\big(\bm{z}_{\infty},\bm{z}_k(\bm{y}_0)\big)\big] \le 2\delta_k/(1-L_q),
\end{align}
where the first equation follows from the observation that $\bm{v}_k$ has the identical distribution as $\bm{z}_{\infty}=\bm{z}_k(\bm{v}_k(\bm{y}))\sim\pi$ and is independent of i.i.d.\ random matrices $X_1,\ldots,X_k$ because $\bm{v}_k$ as defined in~(\ref{eq_operator_v}) only depends on $X_i$ for large $i\ge k+1$. The desired result in~(\ref{eq_gmc_mat}) has been achieved.
\end{proof}

The recursion $\bm{y}_k=f(\bm{y}_{k-1},X_k)$ is only defined for positive integers $k$. Nevertheless, Corollary~\ref{cor_gmc_mat} guarantees that for $k=0,-1,\ldots$, the relation $\bm{y}_k=f(\bm{y}_{k-1},X_k)$ also holds. See Remark 2 in \textcite{wu_limit_2004} for a simple way to define $\bm{y}_k$ when $k=0,-1,\ldots$ in the one-dimensional case. The vector versions can be similarly constructed.

\subsection{Proof of Lemma~\ref{lemma_gmc_cond}}

\begin{proof}[Proof of Lemma~\ref{lemma_gmc_cond}]
Let $D$ be a dropout matrix with the same distribution as $D_1.$ Since $\XX=X^\top X$ is positive semi-definite and by assumption $\alpha \|\XX\|<2$, we have $-I_d< I_d-\alpha D\XX D\leq I_d$ and consequently $\|I_d-\alpha D\XX D\|\leq 1.$ Thus for a unit vector $\bm{v},$ $\|(I_d-\alpha D_k\XX D_k)\bm{v}\|_2\leq 1.$ This means that for $q\geq 2,$ we can use $\|\cdot\|_2^q=\|\cdot\|_2^2 \|\cdot\|_2^{q-2}$ to bound 
\begin{align}
    r_{\alpha,q}^q\leq \sup_{\bm{v}\in\RR^d:\|\bm{v}\|_2=1}\EE\Big\|\big(I_d-\alpha D\XX D\big)\bm{v}\Big\|_2^2
    = 
    \sup_{\bm{v}\in\RR^d:\|\bm{v}\|_2=1} \bm{v}^\top 
  \EE\Big[\big(I_d-\alpha D\XX D\big)^2\Big]\bm{v} 
  = \Big\|\EE\Big[\big(I_d-\alpha D\XX D\big)^2\Big]\Big\|.
  \label{eq.iduncs}
\end{align}
For a $d\times d$ and positive semi-definite matrix $A$, we have $A^2\leq \|A\| A.$ To see this, let $\bm{v}_j$ be the eigenvectors of $A$ with corresponding eigenvalues $\lambda_j.$ Any vector $\bm{w}$ can be written as $\bm{w}=\gamma_1 \bm{v}_1+\ldots+\gamma_d \bm{v}_d$ with coefficients $\gamma_1,\ldots,\gamma_d.$ Now $\bm{w}^\top A^2\bm{w} = \gamma_1^2 \lambda_1^2+\ldots + \gamma_d^2 \lambda_d^2\leq (\max_j\lambda_j) (\gamma_1^2 \lambda_1+\ldots + \gamma_d^2 \lambda_d)= \bm{w}^\top \|A\| A\bm{w}.$ Since $\bm{w}$ was arbitrary, this proves $A^2\leq \|A\| A.$ Moreover, recall that $D_k$ is a diagonal matrix with diagonal entries $0$ and $1$. Thus $D_k^2=D_k\leq I_d.$ Because $\XX$ is positive semi-definite and by assumption $\Delta:=2-\alpha \|X\|>0,$ we have
$\alpha^2 D_1\XX D_1^2\XX D_1\leq \alpha^2 D_1\XX^2 D_1\leq \alpha^2 D_1\|\XX\|\XX D_1 \leq (2-\Delta)\alpha D_1\XX D_1.$ Thus,
\begin{align*}
    (I_d-\alpha D_1\XX D_1)^2
    &=I_d-2\alpha D_1\XX D_1+\alpha^2 D_1\XX D_1^2\XX D_1 \leq I_d-\Delta \alpha D_1\XX D_1.
\end{align*}
Taking expectation and using Lemma \ref{lemma_clara} (ii) yields $\EE\big[(I_d-\alpha D_1\XX D_1)^2\big]\leq I_d-\Delta\alpha p \XX_p.$ The fact that $\EE\big[(I_d-\alpha D_1\XX D_1)^2\big]$ is positive semi-definite implies that $\|\EE\big[(I_d-\alpha D_1\XX D_1)^2\big]\|$ is bounded by the largest eigenvalue of $I_d-\Delta\alpha p \XX_p.$ By definition, $\XX_p=p\XX+(1-p)\Diag(\XX)\geq (1-p)\min_j \XX_{jj} I_d.$ By assumption the design is in reduced form which implies that $\min_j \XX_{jj}>0.$ This shows that $\XX_p$ is positive definite and the largest eigenvalue of $I_d-\Delta\alpha p \XX_p$ must be strictly smaller than $1.$ This implies $\|\EE[(I_d-\alpha D_1\XX D_1)^2]\|<1.$ Combined with \eqref{eq.iduncs} this proves $r_{\alpha,q}<1.$

If Lemma~\ref{lemma_gmc_cond} holds for some $q\ge2$, then by Hölder's inequality, it also holds for all $1<q<2$. To see this, consider a unit vector $\bm{v}\in\RR^d$ and set $r(q) := \big(\EE\|(I_d-\alpha D_1\XX D_1)\bm{v}\|_2^q\big)^{1/q}.$ Then for any $1<q'<q$, it follows from Hölder's inequality that
\begin{align}
    r(q')^{q'} = \EE\Big\|\big(I_d-\alpha D_1\XX D_1\big)\bm{v}\Big\|_2^{q'} \le \Big(\EE\Big\|\big(I_d-\alpha D_1\XX D_1\big)\bm{v}\Big\|_2^q\Big)^{q'/q}= r(q)^{q'} <1.
\end{align}
\end{proof}

\subsection{Proof of Theorem~\ref{thm_gmc_gd}}

\begin{proof}[Proof of Theorem~\ref{thm_gmc_gd}]
Recall the recursive estimator $\tilde\bbeta_k$ defined in (\ref{eq_dropout_gd}). Write $A_k=A_k(\alpha)=I_d-\alpha D_k\XX D_k.$  We consider arbitrary $d$-dimensional initialization vectors $\tilde\bbeta_0,\,\tilde\bbeta_0'\in\RR^d$ and write $\tilde\bbeta_k$ and $\tilde\bbeta_k'$ for the respective iterates (sharing the same dropout matrices). Now
$\tilde\bbeta_k-\tilde\bbeta_k' = A_k(\tilde\bbeta_{k-1}-\tilde\bbeta_{k-1}') =: A_k\Delta_{k-1},$ with independent $A_k$ and $\Delta_{k-1}$. By Lemma \ref{lemma_gmc_cond}, $r :=  \sup_{\bm{v}\in\RR^d,\|\bm{v}\|_2=1}\big(\EE\|A_k\bm{v}\|_2^q\big)^{1/q} <1,$ and thus, for any fixed vector $\bm{v}$, $\EE\|A_k\bm{v}\|_2^q \le r^q\|\bm{v}\|_2^q.$
Due to the independence between $A_k$ and $\Delta_{k-1}$, it follows from the tower rule and the condition above that
\begin{align*}
    \EE\|\tilde\bbeta_k-\tilde\bbeta_k'\|_2^q &= \EE\|A_k\Delta_{k-1}\|_2^q 
    = \EE\big[\EE[\|A_k\Delta_{k-1}\|_2^q \mid \Delta_{k-1}]\big]
    \le \EE\big[r^q\|\Delta_{k-1}\|_2^q\big]
    = r^q\EE\|\Delta_{k-1}\|_2^q.
\end{align*}
Since $A_k(\alpha)$ are i.i.d.\ random matrices induction on $k$ yields the claimed geometric-moment contraction $\big(\EE\|\tilde\bbeta_k(\alpha) - \tilde\bbeta_k'(\alpha)\|_2^q \big)^{1/q} \le r_{\alpha,q}^k\|\tilde\bbeta_0 - \tilde\bbeta_0'\|_2.$

Finally, by Corollary~\ref{cor_gmc_mat},  this geometric-moment contraction implies the existence of a unique stationary distribution $\tilde\pi_{\alpha}$ of the GD dropout sequence $\tilde\bbeta_k(\alpha)$. This completes the proof.
\end{proof}

\section{Proofs in Section~\ref{subsec_iter_dropout}}

\subsection{Proof of Lemma~\ref{lemma_affine_approx}}

\begin{proof}[Proof of Lemma~\ref{lemma_affine_approx}]
Since $\tilde\bbeta_k^{\circ}(\alpha)$ is stationary and $\EE[\bm{b}_k]=0$ by (\ref{eq_bk_mean0_gd}), it follows that $\tilde\bbeta_k^{\circ}(\alpha)-\tilde\bbeta$ and $\bbeta_k^{\dagger}(\alpha)-\tilde\bbeta$ both have zero mean, and thus $\EE[\bm{\delta}_k(\alpha)]=0.$

To prove the second claim, we first note that
\begin{align}
    \bm{\delta}_k(\alpha) & = (I_d-\alpha D_k\XX D_k)(\tilde\bbeta_{k-1}^{\circ}(\alpha)-\tilde\bbeta) + \bm{b}_k(\alpha) - \big[(I_d-\alpha p\XX_p)(\bbeta_{k-1}^{\dagger}(\alpha)-\tilde\bbeta) + \bm{b}_k(\alpha)\big]\nonumber \\
    & = (I_d-\alpha p\XX_p)\bm{\delta}_{k-1} + \alpha(p\XX_p - D_k\XX D_k)(\tilde\bbeta_{k-1}^{\circ} - \tilde\bbeta)
\end{align}
is a stationary sequence. By induction on $k$, we can write $\bm{\delta}_k(\alpha)$ into
\begin{align}
    \bm{\delta}_k(\alpha) & = \alpha\Big[(p\XX_p - D_k\XX D_k)(\tilde\bbeta_{k-1}^{\circ}-\tilde\bbeta) + \cdots + (p\XX_p - D_1\XX D_1)(I_d-\alpha p\XX_p)^{k-1}(\tilde\bbeta_0^{\circ}-\tilde\bbeta) + \cdots\Big] \nonumber \\
    & = \alpha\sum_{i=1}^{\infty}(p\XX_p - D_{k-i+1}\XX D_{k-i+1})(I_d-\alpha p\XX_p)^{i-1}(\tilde\bbeta_{k-i}^{\circ}-\tilde\bbeta) \nonumber \\
    & =: \alpha\sum_{i=1}^{\infty}\M_{k-i}(\alpha).
\end{align}
For any $k\in\NN$, $\{\M_{k-i}(\alpha)\}_{i\ge1}$ is a sequence of martingale differences with respect to the filtration $\F_{k-i}=\sigma(\ldots,D_{k-i-1},D_{k-i})$, since the dropout matrix $D_k$ is independent of $\tilde\bbeta_{k-1}^{\circ}$ and $\tilde\bbeta$. Therefore, we can apply Burkholder's inequality in Lemma~\ref{lemma_burkholder} to $\sum_{i=1}^{\infty}\M_{k-i}(\alpha)$, and obtain, for $q\ge2$,
\begin{align*}
    \bigg(\EE\Big\|\sum_{i=1}^{\infty}\mathcal{M}_{k-i}(\alpha)\Big\|_2^q\bigg)^{1/q} & = \Big(\EE\Big\|\sum_{i=1}^{\infty}(I_d-\alpha p\XX_p)^{i-1}(p\XX_p - D_{k-i+1}\XX D_{k-i+1})(\tilde\bbeta_{k-i}^{\circ}-\tilde\bbeta)\Big\|_2^q\Big)^{1/q} \nonumber \\
    & \lesssim \Big[\sum_{i=1}^{\infty}\big(\EE\big\|(I_d-\alpha p\XX_p)^{i-1}(p\XX_p - D_{k-i+1}\XX D_{k-i+1})(\tilde\bbeta_{k-i}^{\circ}-\tilde\bbeta)\big\|_2^q\big)^{2/q}\Big]^{1/2} \nonumber \\
    & \le \Big[\sum_{i=1}^{\infty}\|I_d-\alpha p \XX_p\|^{2(i-1)}\big(\EE\big\|(p\XX_p - D_{k-i+1}\XX D_{k-i+1})(\tilde\bbeta_{k-i}^{\circ}-\tilde\bbeta)\big\|_2^q\big)^{2/q}\Big]^{1/2},
\end{align*}
where the constant in $\lesssim$ here and the rest of the proof only depends on $q$ unless it is additionally specified. 

We shall proceed the proof with two main steps. First, we show the bound $\|I_d-\alpha p\XX_p\|<1$ for the operator norm  and thus $\sum_{i=1}^{\infty}\|I_d-\alpha p \XX_p\|^{2(i-1)}<\infty$. Second, we provide a bound for $\EE\big\|(p\XX_p - D_k\XX D_k)(\tilde\bbeta_{k-1}^{\circ}-\tilde\bbeta)\big\|_2^q$ uniformly over $k$. 

\textbf{Step 1.} 
Since $\alpha\|\XX\|<2$, it follows from Lemma~\ref{lemma_operator_norm}~(i) that $\alpha\|\XX_p\|\le \alpha\|\XX\|<2$. Moreover, the assumption that the design matrix $X$ has no zero columns guarantees that all diagonal entries of $\XX$ are positive and thus $\Diag(\XX)>0.$ Together with $p<1$, this lead to $\XX_p=p\XX+(1-p)\Diag(\XX)\geq (1-p) \Diag(\XX) > 0.$ We thus have $-I_d< I_d-\alpha p\XX_p < I_d.$ Consequently, $\|I_d-\alpha p\XX_p\|<1$ and 
\begin{align}
    \label{eq_thm_bili_geo_series}
    \sum_{i=1}^{\infty}\|I_d-\alpha p\XX_p\|^{2(i-1)} = \frac{1}{1-\|I_d-\alpha p\XX_p\|^2}=O\big(\alpha^{-1}\big).
\end{align}

\textbf{Step 2.} Next, we shall bound the term $\EE\big\|(p\XX_p - D_i\XX D_i)(\tilde\bbeta_{i-1}^{\circ}- \tilde\bbeta)\big\|_2^q$. We first consider the case $q=2$. Denote $\mathbb{M}_i=p\XX_p - D_i\XX D_i$. Using that $\EE[D_i\XX D_i]=p\XX_p$, we find $\EE[\mathbb{M}_i]=0$ and by the tower rule,
\begin{align}
    \EE\big\|(p\XX_p - D_i\XX D_i)(\tilde\bbeta_{i-1}^{\circ}- \tilde\bbeta)\big\|_2^2 
    & = \EE[(\tilde\bbeta_{i-1}^{\circ}- \tilde\bbeta)^{\top}\mathbb{M}_i^{\top}\mathbb{M}_i(\tilde\bbeta_{i-1}^{\circ}- \tilde\bbeta)] \nonumber \\
    & = \EE\big[\EE\big[(\tilde\bbeta_{i-1}^{\circ}- \tilde\bbeta)^{\top}\mathbb{M}_i^{\top}\mathbb{M}_i(\tilde\bbeta_{i-1}^{\circ}- \tilde\bbeta) \mid \F_{i-1}\big] \big] \nonumber \\
    & \le \|\EE[\mathbb{M}_i^{\top}\mathbb{M}_i]\| \cdot \EE\|\tilde\bbeta_{i-1}^{\circ}- \tilde\bbeta\|_2^2.
\end{align}
By Lemma~\ref{lemma_q_moment_gd}, we have $\EE\|\tilde\bbeta_{i-1}^{\circ}- \tilde\bbeta\|_2^2 = O(\alpha)$. We only need to bound the operator norm $\|\EE[\mathbb{M}_i^{\top}\mathbb{M}_i]\|$. To this end, we use again $\EE[D_i\XX D_i]=p\XX_p$ and moreover $\EE[D_i\XX D_i\XX D_i] = p\XX_p^2+p^2(1-p)\mathrm{Diag}(\overline{\XX}\XX)$, which yields,
\begin{align}
    \EE[\mathbb{M}_i^{\top}\mathbb{M}_i] & = \EE[(p\XX_p - D_i\XX D_i)^{\top}(p\XX_p - D_i\XX D_i)] \nonumber \\
    & = p^2\XX_p^2 - 2p^2\XX_p^2 + p\XX_p^2+p^2(1-p)\mathrm{Diag}(\overline{\XX}\XX) \nonumber \\
    & = p^2(1-p)\mathrm{Diag}(\overline{\XX}\XX).
\end{align}
Recall that $\XX=X^{\top}X$, where $X$ is the fixed design matrix. Then, by Lemma~\ref{lemma_operator_norm}~(i) and the sub-multiplicativity of the operator norm, we have $\|\mathrm{Diag}(\overline{\XX}\XX)\| \le \|\overline{\XX}\XX\| \le \|\overline{\XX}\|\|\XX\|\le \|\XX\|^2.$ As a direct consequence, $\|\EE[\mathbb{M}_i^{\top}\mathbb{M}_i]\|\le p^2(1-p)\|\XX\|^2<\infty$, which together with Lemma~\ref{lemma_q_moment_gd} and (\ref{eq_thm_bili_geo_series}) gives
\begin{align*}
    \EE\|\bm{\delta}_k(\alpha)\|_2 = \Big(\EE\Big\|\alpha\sum_{i=1}^{\infty}\mathcal{M}_{k-i}(\alpha)\Big\|_2^2\Big)^{1/2} \lesssim \alpha\Big(\sum_{i=1}^{\infty}\|I_d-\alpha p \XX_p\|^{2(i-1)}\alpha\Big)^{1/2} = O(\alpha),
\end{align*}
uniformly over $k$. For the case with $q>2$, we can similarly apply the tower rule and obtain
\begin{align}
    & \quad \EE\big\|(p\XX_p - D_i\XX D_i)(\tilde\bbeta_{i-1}^{\circ}- \tilde\bbeta)\big\|_2^q \nonumber \\
    & = \EE\Big[\EE\big[\|(p\XX_p - D_i\XX D_i)(\tilde\bbeta_{i-1}^{\circ}- \tilde\bbeta)\|_2^q \mid \F_{i-1}\big] \Big] \nonumber \\
    & \le \sup_{\bm{v}\in\RR^d,\|\bm{v}\|_2=1}\EE\|(p\XX_p - D_i\XX D_i)\bm{v}\|_2^q\cdot \EE\|\tilde\bbeta_{i-1}^{\circ}- \tilde\bbeta\|_2^q,
\end{align}
where the last inequality can be achieved by writing $\tilde\bbeta_{i-1}^{\circ} - \tilde\bbeta = \|\tilde\bbeta_{i-1}^{\circ} - \tilde\bbeta \|_2\bm{v}$. Here $\bm{v}$ is the unit vector $(\tilde\bbeta_{i-1}^{\circ} - \tilde\bbeta)/\|\tilde\bbeta_{i-1}^{\circ} - \tilde\bbeta\|_2$ with $\|\bm{v}\|_2=1$. In addition, recall the Frobenius norm denoted by $\|\cdot\|_F$. It follows from Lemma~\ref{lemma_operator_norm}~(i) and (ii) that
\begin{align*}
    \sup_{\bm{v}\in\RR^d,\|\bm{v}\|_2=1}\EE\|(p\XX_p-D_i\XX D_i)\bm{v}\|_2^q 
    & \le \EE\|p\XX_p-D_i\XX D_i\|_F^q  \nonumber \\
    & \lesssim \EE\big(\|p\XX_p\|_F^q + \|D_i\XX D_i\|_F^q\big) \nonumber \\
    & \le (p^q+1)\|\XX\|^q <\infty,
\end{align*}
where the constant in $\lesssim$ only depends on $q$. Combining this with the inequality (\ref{eq_thm_bili_geo_series}), we obtain $(\EE\|\bm{\delta}_k(\alpha)\|_2^q)^{1/q}=O(\alpha)$, completing the proof.
\end{proof}

\subsection{Proof of Lemma~\ref{lemma_q_moment_gd}}

\begin{proof}[Proof of Lemma~\ref{lemma_q_moment_gd}]
Recall that by applying induction on $k$ to Equation (\ref{eq_dropout_gd}), we can rewrite the GD dropout iterates $\tilde\bbeta_k(\alpha)$ into
\begin{align*}
    \tilde\bbeta_k(\alpha) - \tilde\bbeta & = A_k(\alpha)(\tilde\bbeta_{k-1}(\alpha) - \tilde\bbeta) + \bm{b}_k(\alpha) \nonumber \\
    & = \sum_{i=0}^{k-1}\Big(\prod_{j=k-i+1}^kA_j(\alpha)\Big)\bm{b}_{k-i}(\alpha) + \Big(\prod_{j=1}^kA_j(\alpha)\Big)(\tilde\bbeta_0(\alpha)-\tilde\bbeta),
\end{align*}
where we set $\prod_{j=k+1}^kA_j(\alpha)=I_d$. Following \textcite{brandt_stochastic_1986}, since both $A_k$ and $\bm{b}_k$ are i.i.d.\ random coefficients, the stationary solution $\{\tilde\bbeta_k^{\circ}(\alpha)-\tilde\bbeta\}_{k\in\NN}$ of this recursion can be written into
\begin{align}
    \label{eq_iter_dropout_stationary}
    \tilde\bbeta_k^{\circ}(\alpha) - \tilde\bbeta & = A_k(\alpha)(\tilde\bbeta_{k-1}^{\circ}(\alpha) - \tilde\bbeta) + \bm{b}_k(\alpha) \nonumber \\
    & = \sum_{i=0}^{\infty}\Big(\prod_{j=k-i+1}^kA_j(\alpha)\Big)\bm{b}_{k-i}(\alpha) \nonumber \\
    & = \alpha\sum_{i=0}^{\infty}\Big[\prod_{j=k-i+1}^k(I_d - \alpha D_j\XX D_j)\Big]D_{k-i}\overline{\XX}(pI_d-D_{k-i})\tilde\bbeta \nonumber \\
    & =: \alpha\sum_{i=0}^{\infty}\tilde\M_{i,k}(\alpha).
\end{align}
We observe that, for any $k\in\NN$, $\{\tilde\M_{i,k}(\alpha)\}_{i\in\NN}$ is a sequence of martingale differences with respect to the filtration $\F_{k-i}=\sigma(D_{k-i},D_{k-i-1},\ldots)$. Hence, it follows from Burkholder's inequality in Lemma~\ref{lemma_burkholder} that, for $q\ge2$,
\begin{align}
    \big(\EE\|\tilde\bbeta_k^{\circ}(\alpha)-\tilde\bbeta\|_2^q\big)^{1/q} & = \alpha\Big(\EE\Big\|\sum_{i=0}^{\infty}\tilde\M_{i,k}(\alpha)\Big\|_2^q\Big)^{1/q} \nonumber \\
    & \lesssim \alpha\Big(\sum_{i=0}^{\infty}\big(\EE\|\tilde\M_{i,k}(\alpha)\|_2^q\big)^{2/q}\Big)^{1/2},
\end{align}
where the constant in $\lesssim$ only depends on $q$. Recall $H_k$ defined in (\ref{eq_H_k}), and we define a $d\times d$ matrix $B_{i,k}$ by
\begin{align}
    B_{i,k} = \Big[\prod_{j=k-i+1}^k(I_d - \alpha D_j\XX D_j)\Big]D_{k-i}\overline{\XX}(pI_d-D_{k-i})= \Big(\prod_{j=k-i+1}^kA_j\Big)H_{k-i}.
\end{align}
This random matrix is independent of $\tilde\bbeta$. For $q=2$, by the tower rule, we have
\begin{align}
    \EE\|\tilde\M_{i,k}(\alpha)\|_2^2 
    & = \EE\big[\EE\big[\tilde\bbeta^{\top}B_{i,k}^{\top}B_{i,k}\tilde\bbeta\mid \F_k\big]\big] \nonumber \\
    & = \EE\big[\EE\big[\mathrm{tr}(\tilde\bbeta\tilde\bbeta^{\top}B_{i,k}^{\top}B_{i,k})\mid\F_k\big]\big] \nonumber \\
    & = \EE\big[\mathrm{tr}\big(\EE\big[\tilde\bbeta\tilde\bbeta^{\top}B_{i,k}^{\top}B_{i,k}\mid\F_k\big]\big)\big] \nonumber \\
    & = \EE\big[\mathrm{tr}\big(\EE[\tilde\bbeta\tilde\bbeta^{\top}]B_{i,k}^{\top}B_{i,k}\big)\big] \nonumber \\
    & = \mathrm{tr}\big(\EE[\tilde\bbeta\tilde\bbeta^{\top}]\EE[B_{i,k}^{\top}B_{i,k}]\big) \nonumber \\
    & \le \|\EE[B_{i,k}^{\top}B_{i,k}]\|\cdot\EE\|\tilde\bbeta\|_2^2.
\end{align}
Following the similar arguments, we obtain for $q\ge2$,
\begin{align}
    \label{eq_MD_q_gtr_2}
     \EE\|\tilde\M_{i,k}(\alpha)\|_2^q \le \sup_{\bm{v}\in\RR^d,\|\bm{v}\|_2=1}\EE\|B_{i,k}\bm{v}\|_2^q\cdot\EE\|\tilde\bbeta\|_2^q.
\end{align}
Moreover, we notice that by the tower rule
\begin{align}
    \|\EE[B_{i,k}^{\top}B_{i,k}]\| & = \Big\|\EE\Big[H_{k-i}^{\top}\Big(\prod_{j=k-i+1}^kA_j\Big)^{\top}\Big(\prod_{j=k-i+1}^kA_jH_{k-i}\Big)\Big]\Big\| \nonumber \\
    & \le \|\EE[H_{k-i}^{\top}A_k^{\top}A_kH_{k-i}]\|\cdot \Big\|\EE\Big[\Big(\prod_{j=k-i+1}^{k-1}A_j\Big)^{\top}\Big(\prod_{j=k-i+1}^{k-1}A_j\Big)\Big]\Big\|.
\end{align}
By a similar argument as Step 2 in the proof of Lemma~\ref{lemma_affine_approx}, we obtain
\begin{align}
    \|\EE[H_{k-i}^{\top}A_k^{\top}A_kH_{k-i}]\| \lesssim p^2\|\XX\|^2 <\infty,
\end{align}
where the constant in $\lesssim$ is independent of $\alpha$. Further, recall that $A_j$ are i.i.d.\ random matrices and $\|\EE[A_1^{\top}A_1]\|\le 1-\alpha p\lambda_{\min}[X^{\top}(2I_d-\alpha\XX)X]$ by the proof of Lemma~\ref{lemma_gmc_cond}. When $\alpha\|\XX\|<2$, it follows from the sub-multiplicativity of operator norm and the similar lines as the Step 1 in the proof of Lemma~\ref{lemma_affine_approx} that 
\begin{align}
    \sum_{i=0}^{\infty}\Big\|\EE\Big[\Big(\prod_{j=k-i+1}^{k-1}A_j\Big)^{\top}\Big(\prod_{j=k-i+1}^{k-1}A_j\Big)\Big]\Big\| = \sum_{i=0}^{\infty}\Big\|\prod_{j=k-i+1}^{k-1}\EE[A_j^{\top}A_j]\Big\| \le \sum_{i=2}^{\infty}\big\|\EE[A_1^{\top}A_1]\big\|^{i-2} = O(1/\alpha).
\end{align}
Therefore, $\sum_{i=0}^{\infty}\EE\|\tilde\M_{i,k}(\alpha)\|_2^2 =O(1/\alpha)$, which yields $(\EE\|\tilde\bbeta_k^{\circ}(\alpha) - \tilde\bbeta\|_2^2)^{1/2}=O(\sqrt{\alpha})$. By leveraging the inequality in (\ref{eq_MD_q_gtr_2}) and the similar techniques adopted in the proof of Lemma~\ref{lemma_gmc_cond} for the case with $q>2$, we obtain that for any $q\ge2$, $\sum_{i=0}^{\infty}(\EE\|\tilde\M_{i,k}(\alpha)\|_2^q)^{2/q} =O(1/\alpha)$. As a direct consequence, we obtain $(\EE\|\tilde\bbeta_k^{\circ}(\alpha) - \tilde\bbeta\|_2^q)^{1/q}=O(\sqrt{\alpha})$, which completes the proof.
\end{proof}

\subsection{Proof of Theorem~\ref{thm_clt_iter_gd}}

\begin{proof}[Proof of Theorem~\ref{thm_clt_iter_gd}]
If we can establish the asymptotic normality for the affine sequence $\{\bbeta_k^{\dagger}(\alpha) - \tilde\bbeta\}_{k\in\NN}$, then by applying Lemma~\ref{lemma_affine_approx} and Markov's inequality, we can prove the CLT for the stationary sequence $\{\tilde\bbeta_k^{\circ}(\alpha) -\tilde\bbeta\}_{k\in\NN}$, which together with the geometric-moment contraction of the dropout iterates $\{\tilde\bbeta_k(\alpha) -\tilde\bbeta\}_{k\in\NN}$ in Theorem~\ref{thm_gmc_gd} can yield the desired result. Therefore, in this proof, we shall show the CLT for $\{\bbeta_k^{\dagger}(\alpha) - \tilde\bbeta\}_{k\in\NN}$, that is,
$$\frac{\bbeta_k^{\dagger}(\alpha) - \tilde\bbeta}{\sqrt{\alpha}} \Rightarrow \N(0,\Xi(\alpha)), \quad \text{as }\alpha\rightarrow0.$$

First, we recall the random vectors $\bm{b}_k(\alpha)$ in (\ref{eq_affine_dropout}) and let
\begin{equation}
    \label{eq_H_k}
    \bm{b}_k(\alpha) =: \alpha H_k\tilde\bbeta, \quad \text{with } \  H_k := D_k \overline{\XX} (p I_d-D_k).
\end{equation}
Then, since $\{\bbeta_k^{\dagger}(\alpha) -\tilde\bbeta\}_{k\in\NN}$ is a stationary sequence and using induction on $k$, we can rewrite $\bbeta_k^{\dagger}(\alpha) -\tilde\bbeta$ into 
\begin{align}
    \label{eq_linear_sum}
    \bbeta_k^{\dagger}(\alpha) -\tilde\bbeta & = \alpha\Big(I_dH_k + (I_d-\alpha p\XX_p)H_{k-1} + \cdots + (I_d-\alpha p\XX_p)^{k-1}H_1 + \cdots\Big)\tilde\bbeta \nonumber \\
    & = \alpha\sum_{i=0}^{\infty}(I_d-\alpha p\XX_p)^iH_{k-i}\tilde\bbeta.
\end{align}
The $H_k$ are i.i.d.\ random matrices, and independent of $\tilde\bbeta$ and $\bbeta^{\dagger}_{k-1}$. Therefore, we shall apply the Lindeberg-Feller central limit theorem to the partial sum in (\ref{eq_linear_sum}). To this end, we first take the expectation on both sides of (\ref{eq_linear_sum}). Since the random matrices $H_i$ are independent of $\tilde\bbeta$ for all $i\in\NN$, we obtain
\begin{align}
    \EE[\bbeta_k^{\dagger}(\alpha) -\tilde\bbeta ] & = \alpha\sum_{i=0}^{\infty} (I_d-\alpha p\XX_p)^i\EE[H_{k-i}\tilde\bbeta] \nonumber \\
    & = \alpha\sum_{i=0}^{\infty}(I_d-\alpha p\XX_p)^i\EE[H_{k-i}]\EE[\tilde\bbeta] = 0.
\end{align}
To see the last equality, we apply Lemma~\ref{lemma_clara}~(i) and (ii) and obtain $\EE[D_k\overline{\XX}D_k]=p\overline{\XX}_p=p^2\overline{\XX}$, which gives
\begin{align}
    \EE[H_k] = \EE[D_k \overline{\XX} (p I_d-D_k)] = p^2\overline{\XX} - p^2\overline{\XX} = 0,
\end{align}
As a direct consequence, by (\ref{eq_H_k}) and the independence of $D_k$ and $\tilde\bbeta$, we have
\begin{equation}
    \label{eq_bk_mean0_gd}
    \EE[\bm{b}_k(\alpha)] = \alpha\EE[H_k]\EE[\tilde\bbeta] = 0.
\end{equation}
Next, we shall provide a closed form of the covariance matrix $\mathrm{Cov}(\bbeta_k^{\dagger}(\alpha)-\tilde\bbeta)$. Notice that the random vectors $H_i\tilde\bbeta$ are uncorrelated over different $i$, and $\EE[H_i\tilde\bbeta\tilde\bbeta^{\top}H_i] = \EE[H_1\tilde\bbeta\tilde\bbeta^{\top}H_1]$ due to the stationarity of the sequence $\{H_i\tilde\bbeta\}_{i\in\NN}$. Hence, by (\ref{eq_linear_sum}), we have
\begin{align}
    \label{eq_cov_fixed_dropout}
    V_{\alpha} & := \mathrm{Cov}\Big(\alpha^{-1/2}(\bbeta_k^{\dagger}(\alpha)-\tilde\bbeta)\Big) \nonumber \\
    & = \alpha^{-1} \EE\big[(\bbeta_k^{\dagger}(\alpha)-\tilde\bbeta)(\bbeta_k^{\dagger}(\alpha)-\tilde\bbeta)^{\top}\big] \nonumber \\
    & = \alpha\sum_{i=0}^{\infty}(I_d-\alpha p\XX_p)^i\EE[H_{k-i}\tilde\bbeta\tilde\bbeta^{\top}H_{k-i}](I_d-\alpha p\XX_p)^i \nonumber \\
    & =: \alpha\sum_{i=0}^{\infty}(I_d-\alpha p\XX_p)^iS(I_d-\alpha p\XX_p)^i,
\end{align}
with $d\times d$ matrix
\begin{align}
    S & := \EE[H_1\tilde\bbeta\tilde\bbeta^{\top}H_1].
\end{align}
Since $D_k$ is independent of $\tilde\bbeta$, and $\tilde\bbeta\tilde\bbeta^{\top}$ is a symmetric matrix, it follows from the tower rule that
\begin{align}
    S = \EE[\EE(H_1\tilde\bbeta\tilde\bbeta^{\top}H_1\mid \bm{y},X)] = \EE[H_1\EE(\tilde\bbeta\tilde\bbeta^{\top})H_1] =:\EE[H_1S_0H_1].
\end{align}
By the closed form solution of $\tilde\bbeta$ in (\ref{eq_l2_closed_form_gd}) and $\EE[\bm{\epsilon}]=0$, $\mathrm{Cov}(\bm{\epsilon})=I_n$, we obtain
\begin{align}
    \label{eq_S_part1}
    S_0=\EE(\tilde\bbeta\tilde\bbeta^{\top}) & = \XX_p^{-1}X^{\top}\EE(\bm{y}\bm{y}^{\top})X\XX_p^{-1} \nonumber \\
    & = \XX_p^{-1}X^{\top}\EE[(X\bbeta^* + \bm{\epsilon})(X\bbeta^* + \bm{\epsilon})^{\top}]X\XX_p^{-1} \nonumber \\
    & = \XX_p^{-1}X^{\top}\big(X\bbeta^*\bbeta^{*\top}X^{\top} + I_n\big)X\XX_p^{-1}.
\end{align}
Furthermore, by Lemma~\ref{lemma_clara}~(i), one can show that $\overline{(\overline{\XX})} = \overline{\XX}$ and $\mathrm{Diag}(A\overline{\XX}) = \mathrm{Diag}(\overline{A}\overline{\XX})$ for any matrix $A.$ Then, by the definition of $H_k$ in (\ref{eq_H_k}) and Lemma~\ref{lemma_clara}~(ii)--(iv), we can simplify $\EE[H_1S_0H_1]$ as follows:
\begin{align}
    \label{eq_S_part2}
    \EE[H_1S_0H_1] & = \EE[D_k\overline{\XX} (pI_d-D_k)S_0(pI_d-D_k)\overline{\XX}D_k] \nonumber \\
    & = p^2\EE[D_1\overline{\XX}S_0\overline{\XX}D_1] -p\EE[D_1\overline{\XX}D_1S_0\overline{\XX}D_1] - p\EE[D_1\overline{\XX}S_0D_1\overline{\XX}D_1] + \EE[D_1\overline{\XX}D_kS_0D_k\overline{\XX}D_1] \nonumber \\
    & = p^3(\overline{\XX}S_0\overline{\XX})_p -2p\Big(p\overline{\XX}_p(S_0\overline{\XX})_p + p^2(1-p)\mathrm{Diag}(\overline{\XX}S_0\overline{\XX})\Big) \nonumber \\
    & \quad +p\overline{\XX}_p(S_0)_p\overline{\XX}_p + p^2(1-p)\Big(\mathrm{Diag}(\overline{\XX}(S_0)_p\overline{\XX}) + 2\overline{\XX}_p\mathrm{Diag}(\overline{S_0}\overline{\XX}) + (1-p)\overline{\XX}\odot\overline{S_0}^{\top}\odot\overline{\XX}\Big).
\end{align}
Combining (\ref{eq_S_part1}) and (\ref{eq_S_part2}), we obtain a closed form solution of $S$ which is independent of $\alpha$.

Now we are ready to solve the covariance matrix $V_{\alpha}$ in (\ref{eq_cov_fixed_dropout}). We multiply the matrix $I_d - \alpha p\XX_p$ to the left and right sides of (\ref{eq_cov_fixed_dropout}) and obtain
\begin{align}
    (I_d - \alpha p\XX_p)V_{\alpha}(I_d - \alpha p\XX_p) = \alpha\sum_{i=0}^{\infty}(I_d-\alpha p\XX_p)^iS(I_d-\alpha p\XX_p)^i.
\end{align}
Taking the difference between $V_{\alpha}$ and $(I_d - \alpha p\XX_p)V_{\alpha}(I_d - \alpha p\XX_p)$ yields
\begin{align}
    V_{\alpha} - (I_d - \alpha p\XX_p)V_{\alpha}(I_d - \alpha p\XX_p) = \alpha S.
\end{align}
Denote the symmetric matrix $A_p=p\XX_p$. By simplifying the equation above, for $\alpha>0$, we have
\begin{align}
    V_{\alpha}A_p - A_pV_{\alpha} + \alpha A_pV_{\alpha}A_p = S.
\end{align}
Let $V_0=\lim_{\alpha\rightarrow0}V_{\alpha}$. As $\alpha\rightarrow0$, the quadratic term $\alpha A_pV_{\alpha}A_p$ vanishes. Thus, we only need to solve the equation
\begin{align}
    S-V_0A_p - A_pV_0 = 0,
\end{align}
to get the solution for
$$V_0=\lim_{\alpha\rightarrow0}V_{\alpha}.$$ 
Following Theorem 1 in \textcite{pflug_stochastic_1986} and the subsequent Remark therein, we can get the closed form solution of $V_0$, that is,
\begin{align}
    \mathrm{vec}(V_0) = (I_d\otimes A_p + A_p\otimes I_d)^{-1}\cdot \mathrm{vec}(S),
\end{align}
where the $d^2\times d^2$ matrix $I_d\otimes A_p + A_p\otimes I_d$ is invertible since the fixed design matrix $X$ is assumed to be in a reduced form with no zero columns. For a small $\alpha>0$, we shall provide a similar closed form solution for $V_{\alpha} = V_0 + \alpha B_p$. Specifically, we need to get the closed form of the matrix $B_p$ by solving a similar equation:
\begin{equation}
    A_pV_0A_p - B_pA_p - A_pB_p=0,
\end{equation}
which gives
\begin{equation}
    \mathrm{vec}(B_p) = (I_d\otimes A_p + A_p\otimes I_d)^{-1}\times \mathrm{vec}(A_pV_0A_p).
\end{equation}
The deterministic matrices $V_0$, $A_p$ and $B_p$ are all independent of $\alpha$. By inserting the results of $V_0$ and $B_p$ into $V_{\alpha}=V_0+\alpha B_p$, we obtain 
$$\Xi(\alpha) = V_{\alpha} = V_0 + \alpha B_p,$$
which holds uniformly over $k$ due to the stationarity of $\{\bbeta_k^{\dagger}(\alpha)-\tilde\bbeta\}_{k\in\NN}$.

Finally, by applying the Lindeberg-Feller central limit theorem to the partial sum in (\ref{eq_linear_sum}), we establish the asymptotic normality of $\{\bbeta_k^{\dagger}(\alpha)-\tilde\bbeta\}_{k\in\NN}$ and complete the proof.
\end{proof}

\section{Proofs in Section~\ref{subsec_ave_dropout}}

We first outline the main techniques for establishing the asymptotic normality of the averaged GD dropout sequence $\{\bar\bbeta_n^{\gd}(\alpha)\}_{n\in\NN}$ defined in (\ref{eq_ave_dropout_gd}).

Recall the observation $\bm{y}$ in model (\ref{eq_model_gd}) and the dropout matrix $D$. For the GD dropout $\{\tilde\bbeta_k(\alpha)\}_{k\in\NN}$ in (\ref{eq_dropout_gd_origin}), by Theorem~\ref{thm_gmc_gd}, we can define a centering term as follows,
\begin{equation}
    \label{eq_center_gd}
    \tilde\bbeta_{\infty}(\alpha) = \lim_{k\rightarrow\infty}\EE_{D}[\tilde\bbeta_k(\alpha)] = \EE_{D}[\tilde\bbeta_1^{\circ}(\alpha)],
\end{equation}
where $\tilde\bbeta_1^{\circ}(\alpha)$ follows the stationary distribution $\pi_{\alpha}$ as stated in (\ref{eq_iter_dropout_stationary0}). According to Lemma 1 in \textcite{clara_dropout_2023}, we note that $\EE_{D}[\tilde\bbeta_k(\alpha) - \tilde\bbeta]\neq0$ but $\|\EE_{D}[\tilde\bbeta_k(\alpha) - \tilde\bbeta]\|_2\rightarrow0$ if $\alpha p\|\XX\|<1$ with a geometric rate as $k\rightarrow\infty$. Therefore, we shall first show the central limit theorems for the partial sum of $\{\tilde\bbeta_k(\alpha) - \tilde\bbeta_{\infty}(\alpha)\}_{k\in\NN}$ and then for the one of $\{\tilde\bbeta_k(\alpha) - \tilde\bbeta\}_{k\in\NN}$.

Next, we take a closer look at the partial sum of $\{\tilde\bbeta_k(\alpha) - \tilde\bbeta_{\infty}(\alpha)\}_{k\in\NN}$. The iterative function $f$ defined in (\ref{eq_dropout_function_gd}) allows us to write $\tilde\bbeta_k(\alpha) = f_{D_k}(\tilde\bbeta_{k-1}(\alpha))$ for all $k\in\NN$. Similarly, for the initialization $\tilde\bbeta_0^{\circ}(\alpha)$ that follows the unique stationary distribution $\pi_{\alpha}$ in Theorem~\ref{thm_gmc_gd}, we can write the stationary GD dropout sequence $\{\tilde\bbeta_k^{\circ}(\alpha)\}_{k\in\NN}$ into
\begin{equation}
    \tilde\bbeta_k^{\circ}(\alpha) = f_{D_k}(\tilde\bbeta_{k-1}^{\circ}(\alpha)), \quad k\in\NN.
\end{equation}
Recall $\tilde\bbeta_{\infty}(\alpha)$ defined in (\ref{eq_center_gd}). Then, we can recursively rewrite $\tilde\bbeta_k^{\circ}(\alpha)$ using the iterative function $f$ and obtain the partial sum
\begin{align}
    \label{eq_agd_partial_sum_stationary_mean0}
    \tilde S_n^{\circ}(\alpha) & := \sum_{k=1}^n[\tilde\bbeta_k^{\circ}(\alpha)-\tilde\bbeta_{\infty}(\alpha)] \nonumber \\
    & = \big\{f_{D_1}(\tilde\bbeta_0^{\circ}(\alpha)) - \EE\big[f_{D_1}(\tilde\bbeta_0^{\circ}(\alpha))\big]\big\} + \big\{f_{D_2}\circ f_{D_1}(\tilde\bbeta_0^{\circ}(\alpha)) - \EE\big[f_{D_2}\circ f_{D_1}(\tilde\bbeta_0^{\circ}(\alpha))\big]\big\} + \cdots \nonumber \\
    & \quad + \big\{f_{D_n}\circ\cdots\circ f_{D_1}(\tilde\bbeta_0^{\circ}(\alpha)) - \EE\big[f_{D_n}\circ\cdots\circ f_{D_1}(\tilde\bbeta_0^{\circ}(\alpha))\big]\big\}.
\end{align}

Primarily, we aim to (i) prove the central limit theorem for the partial sum $n^{-1/2}\tilde S_n^{\circ}(\alpha)$, and (ii) prove the invariance principle for the partial sum process $(\tilde S_i^{\circ}(\alpha))_{1\le i\le n}$. To this end, we borrow the idea of \textit{functional dependence measure} in \textcite{wu_nonlinear_2005}, which was further investigated in \textcite{shao_asymptotic_2007} to establish the asymptotic normality for sequences with \textit{short-range dependence} (see (\ref{eq_short_range_dep_gd}) for the definition). We shall show that the GD dropout sequence $\{\tilde\bbeta_k^{\circ}(\alpha)\}_{k\in\NN}$ that satisfies the geometric-moment contraction (as proved in Theorem~\ref{thm_gmc_gd}) satisfies such short-range dependence condition.

Finally, we shall complete the proofs of the quenched central limit theorems by showing that, for any given constant learning rate $\alpha>0$ satisfying the conditions in Theorem~\ref{thm_clt_ave_gd}, and any initialization $\tilde\bbeta_0\in\RR^d$, the partial sum 
\begin{equation}
    \label{eq_agd_partial_sum_mean0}
    \tilde S_n^{\tilde\bbeta_0}(\alpha):=\sum_{k=1}^n[\tilde\bbeta_k(\alpha)-\tilde\bbeta_{\infty}(\alpha)] 
\end{equation}
converges to the stationary partial sum process $\tilde S_n^{\circ}(\alpha)$, in the sense that $n^{-1/2}\big(\EE\|\tilde S_n^{\tilde\bbeta_0}(\alpha) - \tilde S_n^{\circ}(\alpha)\|_2^q\big)^{1/q}=o(1)$ as $n\rightarrow\infty$.

\subsection{Functional Dependence Measure}\label{subsec_functional_dependence}

Before proceeding to the proofs of Theorems~\ref{thm_clt_ave_gd}~and~\ref{thm_fclt_agd}, we first provide the detailed form of the functional dependence measure in \textcite{wu_nonlinear_2005} for the iterated random functions with i.i.d.~random matrices as inputs. This will serve as the foundational pillar to build the asymptotic normality of averaged GD dropout iterates.

First, for any random vector $\boldsymbol{\zeta}\in\RR^d$ satisfying $\EE\|\boldsymbol{\zeta}\|_2<\infty$, define projection operators
\begin{equation}
    \label{eq_projection_operator_gd}
    \P_k[\boldsymbol{\zeta}] = \EE[\boldsymbol{\zeta}\mid\F_k] - \EE[\boldsymbol{\zeta}\mid\F_{k-1}], \quad k\in\ZZ,
\end{equation}
where we recall the filtration $\F_i=\sigma(D_i,D_{i-1},\ldots)$ with i.i.d.\ dropout matrices $D_i$, $i\in\ZZ.$ By Theorem~\ref{thm_gmc_gd} and (\ref{eq_function_h_gd}), there exists a measurable function $h_{\alpha}(\cdot)$ such that the stationary GD dropout sequence $\{\tilde\bbeta_k^{\circ}(\alpha)\}_{k\in\NN}$ can be written as the following causal process
\begin{equation}
    \tilde\bbeta_k^{\circ}(\alpha) = h_{\alpha}(D_k,D_{k-1},\ldots) = h_{\alpha}(\F_k).
\end{equation}
Define a coupled version of filtration $\F_i$ as $\F_{i,\{j\}}=\sigma(D_i,\ldots,D_{j+1},D_j',D_{j-1},\ldots)$. In addition, $\F_{i,\{j\}}= \F_i$ if $j>i$.  For $q>1$, define the \textit{functional dependence measure} of $\tilde\bbeta_k^{\circ}(\alpha)$ as
\begin{equation}
    \label{eq_functional_dep_measure}
    \theta_{k,q}(\alpha) = \big(\EE\|\tilde\bbeta_k^{\circ}(\alpha) - \tilde\bbeta_{k,\{0\}}^{\circ}(\alpha)\|_2^q\big)^{1/q}, \quad \text{where }\tilde\bbeta_{k,\{0\}}^{\circ}(\alpha) = h_{\alpha}(\F_{k,\{0\}}).
\end{equation}
The above quantity can be interpreted as the dependence of $\tilde\bbeta_k^{\circ}(\alpha)$ on $D_0$ (see the discussion below Theorem~\ref{thm_gmc_gd} for the meaning of $\tilde\bbeta_k^{\circ}$ with $k\le0$), and $\tilde\bbeta_{k,\{0\}}^{\circ}(\alpha)$ is a coupled version of $\tilde\bbeta_k^{\circ}(\alpha)$ with $D_0$ in the latter replaced by its i.i.d.\ copy $D_0'$. If $\tilde\bbeta_k^{\circ}(\alpha)$ does not functionally depend on $D_0$, then $\theta_{k,q}(\alpha)=0$.

Furthermore, if $\sum_{k=0}^{\infty}\theta_{k,q}(\alpha)<\infty$, we define the tail of the \textit{cumulative dependence measure} as
\begin{align}
    \label{eq_cumulative_dep_measure}
    \Theta_{m,q}(\alpha) = \sum_{k=m}^{\infty}\theta_{k,q}(\alpha), \quad m\in\NN.
\end{align}
This can be interpreted as the cumulative dependence of $\{\tilde\bbeta_k^{\circ}(\alpha)\}_{k\ge m}$ on $D_0$, or equivalently, the cumulative dependence of $\tilde\bbeta_0^{\circ}(\alpha)$ on $D_j$, $j\ge m$. The functional dependence measure in (\ref{eq_functional_dep_measure}) and its cumulative variant in (\ref{eq_cumulative_dep_measure}) are easy to work with and they can directly reflect the underlying data-generating mechanism of the iterative function $\tilde\bbeta_k^{\circ}(\alpha)=f_{D_k}(\tilde\bbeta_{k-1}^{\circ}(\alpha))$.

Specifically, for all $q\ge2$, Theorem 1 in \textcite{wu_nonlinear_2005} pointed out a useful inequality for the functional dependence measure as follows,
\begin{align}
    \label{eq_functional_dep_measure_ineq}
    \sum_{k=0}^{\infty}\big(\EE\|\P_0[\tilde\bbeta_k^{\circ}(\alpha)]\|_2^q\big)^{1/q} \le \sum_{k=0}^{\infty}\theta_{k,q}(\alpha)
     = \Theta_{0,q}(\alpha).
\end{align}
In particular, for some given learning rate $\alpha>0$, we say the sequence $\{\tilde\bbeta_k^{\circ}(\alpha)\}_{k\in\NN}$ satisfies the short-range dependence condition if
\begin{equation}
    \label{eq_short_range_dep_gd}
    \Theta_{0,q}(\alpha) < \infty, \quad \text{for some }q\ge2.
\end{equation}
This dependence assumption has been widely adopted in the literature; see for example the central limit theorems in \textcite{shao_asymptotic_2007} and the invariance principle in \textcite{berkes_kmt_2014,karmakar_optimal_2020}. If condition (\ref{eq_short_range_dep_gd}) fails, then $\tilde\bbeta_k^{\circ}(\alpha)$ can be long-range dependent, and the partial sum (resp. partial sum processes) behave no longer like Gaussian random vectors (resp. Brownian motions).

Here, we introduce Theorem 2.1 in \textcite{shao_asymptotic_2007} and Theorem 2 in \textcite{karmakar_optimal_2020}, which are the fundamental tools for the proofs of Theorems~\ref{thm_clt_ave_gd}~and~\ref{thm_fclt_agd}, respectively. 

\begin{lemma}[Asymptotic normality (\cite{shao_asymptotic_2007})]
    \label{lemma_thm21_shao_wu_2007}
    Consider a sequence of stationary mean-zero random variables $x_k=g(\epsilon_k,\epsilon_{k-1},\ldots)\in\RR$, for $k=1,\ldots,n$, where $\epsilon_k$'s are i.i.d.\ random variables, and $g(\cdot)$ is a measurable function such that each $x_k$ is a proper random variable. Assume that $(\EE|x_k|^2)^{1/2}<\infty$. 
    Define the Fourier transform of $x_k$ by
    $$S_n(\omega)=\sum_{k=1}^nx_ke^{ik\omega},$$
    and let $f(\omega)=(2\pi)^{-1}\sum_{k\in\ZZ}\EE[x_0x_k]e^{ik\omega}$, $\omega\in\RR$, be the spectral density of $x_k$. Denote the real and imaginary parts of $S_n(\omega_j)/\sqrt{\pi nf(\omega_j)}$ by
    $$z_j= \frac{\sum_{k=1}^nx_k\cos{(k\omega_j)}}{\sqrt{\pi nf(\omega_j)}},\quad z_{j+m}= \frac{\sum_{k=1}^nx_k\sin{(k\omega_j)}}{\sqrt{\pi nf(\omega_j)}}, \quad j=1,\ldots,m,$$
    where $m=\lfloor(n-1)/2\rfloor$ with $\lfloor a\rfloor$ denoting the integer part of $a$. Let $\Omega_d=\{\bm{c} \in \RR^d:|\bm{c}|=1\}$ be the unit sphere. For the set $\bm{j}=\{j_1, \ldots, j_d\}$ with $1 \le j_1< \cdots <j_d \le 2m$, write the vector $\bm{z}_{\bm{j}}=(z_{j_1}, \ldots, z_{j_d})^{\top}$. Let the class $\Xi_{m,d}=\{\bm{j} \subset\{1, \ldots, 2m\}: \#\bm{j}=d\}$, where $\#\bm{j}$ is the cardinality of $\bm{j}$. If $\min_{\omega\in\RR}f(\omega)>0$ and
    \begin{equation}
        \label{eq_shao_wu_short_range}\sum_{k=0}^{\infty}\sup_{\bm{c}\in\Omega_d}\big(\EE|\P_0[x_k]|_2^2\big)^{1/2} < \infty,
    \end{equation}
    where the projection operator $\P_k[\cdot]$ is in (\ref{eq_projection_operator_gd}), then
    \begin{equation}
        \sup_{\bm{j}\in\Xi_{m,d}}\sup_{\bm{c}\in\Omega_d}\sup_{u\in\RR}\Big|\PP\big( \bm{z}_{\bm{j}}^{\top}\bm{c}\le u\big) - \Phi(u)\Big| =o(1),\quad \text{as }n\rightarrow\infty.
    \end{equation}
\end{lemma}

\begin{lemma}[Gaussian approximation (\cite{karmakar_optimal_2020})]\label{lemma_thm2_karmakar_wu_2020}
    Consider a sequence of nonstationary mean-zero random vectors $\bm{x}_k=g_k(\epsilon_k,\epsilon_{k-1},\ldots)\in\RR^d$, for $k=1,\ldots,n$, where the $\epsilon_k$'s are i.i.d.\ random variables, and $g_k(\cdot)$ is a measurable function such that each $\bm{x}_k$ is a proper random vector. Let $S_j=\sum_{k=1}^j\bm{x}_k$. Assume the following conditions hold for some $q>2$:
    \begin{itemize}
        \item[(i)] The series $(\|\bm{x}_k\|_2^q)_{k\ge1}$ is uniformly integrable: $\sup_{k\ge1}\EE\big[\|\bm{x}_k\|_2^q\One_{\|\bm{x}_k\|_2\ge u}\big]\rightarrow0$ as $u\rightarrow\infty$, 
        \item[(ii)] The eigenvalues of covariance matrices of increment processes are lower-bounded, that is, there exists $\lambda_*>0$ and $l_*\in\NN$, such that for all $t\ge1$, $l\ge l_*$,
        $$\lambda_{\min}\big(\mathrm{Cov}(S_{t+l}-S_t)\big)\ge \lambda_*l;$$
        \item[(iii)] There exist constants $\chi>\chi_0$ and $\kappa>0$, where
        $$\chi_0=\frac{q^2-4+(q-2)\sqrt{q^2+20q+4}}{8q},$$
        such that the tail cumulative dependence measure 
        \begin{align}
            \Theta_{m,q}(\alpha) = \sum_{k=m}^{\infty}\theta_{k,q}(\alpha) = O\big\{m^{-\chi}\big(\log(m)\big)^{-\kappa}\big\}.
        \end{align}
    \end{itemize}
    Then, for all $q>2$, there exists a probability space $(\Omega^{\star}, \A^{\star}, \PP^{\star})$ on which we can define random vectors $\bm{x}_k^{\star}$, with the partial sum process $S_i^{\star}=\sum_{k=1}^i \bm{x}_k^{\star}$ and a Gaussian process $G_i^{\star}=\sum_{k=1}^i \bm{z}_k^{\star}$. Here $\bm{z}_k^{\star}$ is a mean-zero independent Gaussian vector, such that $(S_i^{\star})_{1\le i\le n}\overset{\D}{=}(S_i)_{1\le i\le n}$ and
    $$\max_{i\le n}|S_i^{\star}-G_i^{\star}|=o_{\PP}(n^{1/q}) \quad \text {in }(\Omega^{\star}, \A^{\star}, \PP^{\star}).$$
\end{lemma}

As a special case of Lemma~\ref{lemma_thm21_shao_wu_2007}, by taking $\omega_j=0$, one can establish the asymptotic normality of $\sum_{k=1}^nx_k$. We shall leverage this result in the proof of Theorem~\ref{thm_clt_ave_gd}. Moreover, we notice that condition (ii) in Lemma~\ref{lemma_thm2_karmakar_wu_2020} on the non-singularity is required when the sequence $\{\bm{x}_k\}_{k\in\NN}$ is non-stationary. However, if the function $g_k(\cdot)\equiv g(\cdot)$, that is, the sequence $\{\bm{x}_k\}_{k\in\NN}$ is stationary, then the covariance matrix of the increments is allowed to be singular. To see this, consider a stationary partial sum $S_l=(S_{l,1},\ldots,S_{l,d})^{\top}$ with a singular covariance matrix $\Sigma\in\RR^{d\times d}$ and assume $\mathrm{rank}(\Sigma)=d-1$. Then, there exists a unit vector $\bm{v}\in\RR^d$ such that $\Sigma\bm{v}=0$, which indicates that $S_{l,1}$ can be written into a linear combination of $S_{l,2},\ldots,S_{l,d}$, and the covariance matrix of this linear combination is non-singular. Hence, condition (ii) in Lemma~\ref{lemma_thm2_karmakar_wu_2020} is not required for stationary processes.

In addition, the original Theorem 2.1 in \textcite{shao_asymptotic_2007} and Theorem 2 in \textcite{karmakar_optimal_2020} considered a simple case where the i.i.d.\ inputs $\epsilon_k$ are one-dimensional. These two theorems still hold even if the inputs are i.i.d.\ random matrices such as the dropout matrices $D_k$ in our case. In fact, as long as the inputs are i.i.d.\ elements, the functional dependence measure can be similarly computed as the one in one-dimensional case. The essence is that the short-range dependence condition (\ref{eq_short_range_dep_gd}) is satisfied using an appropriate norm (e.g., $L^2$-norm for vectors, operator norm for matrices) by the output $x_k$. For example, \textcite{wu_limit_2004} considered iterated random functions on a general metric space, and \textcite{chen_stability_2016} assumed the $\epsilon_i$'s to be i.i.d.\ random elements to derive asymptotics for $x_k$. We will verify this short-range dependence condition on the GD dropout vector estimates $\{\tilde\bbeta_k^{\circ}(\alpha)\}_{k\in\NN}$ in the proof of Theorem~\ref{thm_clt_ave_gd}.

\subsection{Proof of Theorem~\ref{thm_clt_ave_gd}}

\begin{proof}[Proof of Theorem~\ref{thm_clt_ave_gd}]
We verify the short-range dependence condition for the stationary GD dropout sequence $\{\tilde\bbeta_k^{\circ}(\alpha)\}_{k\in\NN}$. 

First, consider two different initial vectors $\tilde\bbeta_0^{\circ},\,\tilde\bbeta_0^{\circ'}\in\RR^d$ following the unique stationary distribution $\pi_{\alpha}$ in Theorem~\ref{thm_gmc_gd}. Denote the two GD dropout sequences by $\{\tilde\bbeta_k^{\circ}(\alpha)\}_{k\in\NN}$ and $\{\tilde\bbeta_k^{\circ'}(\alpha)\}_{k\in\NN}$ accordingly. By the geometric-moment contraction in Theorem~\ref{thm_gmc_gd}, for all $q\ge2$, we have
\begin{align}
    \sup_{\tilde\bbeta_0^{\circ},\,\tilde\bbeta_0^{\circ'}\in\RR^d,\,\tilde\bbeta_0^{\circ}\neq\tilde\bbeta_0^{\circ'}}\frac{\big(\EE\|\tilde\bbeta_k^{\circ}(\alpha) - \tilde\bbeta_k^{\circ'}(\alpha)\|_2^q\big)^{1/q}}{\|\tilde\bbeta_0^{\circ} - \tilde\bbeta_0^{\circ'}\|_2} \le r_{\alpha,q}^k, \quad k\in\NN,
\end{align}
for some constant $r_{\alpha,q}\in(0,1)$. Equivalently, it can be rewritten in terms of the iterative function $f$ defined in (\ref{eq_dropout_function_gd}) and $h_{\alpha}(\cdot)$ defined in (\ref{eq_function_h_gd}). That is, for all $\tilde\bbeta_0^{\circ},\,\tilde\bbeta_0^{\circ'}\in\RR^d$, such that $\tilde\bbeta_0^{\circ}\neq\tilde\bbeta_0^{\circ'}$, we have
\begin{align}
    \label{eq_thm_agd_clt_function}
    & \quad \big(\EE\|f_{D_k}\circ\cdots\circ f_{D_1}(\tilde\bbeta_0^{\circ}) - f_{D_k}\circ\cdots\circ f_{D_1}(\tilde\bbeta_0^{\circ'})\|_2^q\big)^{1/q}\nonumber \\
    & = \big(\EE\|h_{\alpha}(D_k,\ldots,D_1,D_0,D_{-1},\ldots) - h_{\alpha}(D_k,\ldots,D_1,D_0',D_{-1}',\ldots)\|_2^q\big)^{1/q} \nonumber \\
    & = \big(\EE\|h_{\alpha}(\F_k) - h_{\alpha}(\F_{k,\{0,-1,\ldots\}})\|_2^q\big)^{1/q} \nonumber \\
    & \le c_qr_{\alpha,q}^k,
\end{align}
where we recall the filtration $\F_{i,\{j\}}=\sigma(D_i,\ldots,D_{j+1},D_j',D_{j-1},\ldots)$, and $c_q>0$ is some constant independent of $k$. Moreover, since $h_{\alpha}(\F_k)$ is stationary over $k$ and $D_i$ and $D_j'$ are i.i.d.\ random matrices, for all $i,j\in\ZZ$, it follows that
\begin{align}
    \label{eq_thm_agd_clt_function2}
    & \quad \EE\|h_{\alpha}(\F_{k,\{0\}}) - h_{\alpha}(\F_{k,\{0,-1,\ldots\}})\|_2^q  \nonumber \\
    & = \EE\|h_{\alpha}(\F_k) - h_{\alpha}(\F_{k,\{-1,-2\ldots\}})\|_2^q \nonumber \\
    & = \EE\|h_{\alpha}(\F_{k+1}) - h_{\alpha}(\F_{k+1,\{0,-1,\ldots\}})\|_2^q \le c_q'r_{\alpha,q}^k,
\end{align}
where the constant $c_q'>0$ is also independent of $k$. Hence, by (\ref{eq_thm_agd_clt_function})~and~(\ref{eq_thm_agd_clt_function2}), we can bound the functional dependence measure defined in (\ref{eq_functional_dep_measure}) as follows
\begin{align}
    \label{eq_functional_dep_measure_gd}
    \theta_{k,q}(\alpha) & = \big(\EE\|h_{\alpha}(\F_k) - h_{\alpha}(\F_{k,\{0\}})\|_2^q\big)^{1/q} \nonumber \\
    & \le \big(\EE\|h_{\alpha}(\F_k) - h_{\alpha}(\F_{k,\{0,-1,\ldots\}})\|_2^q\big)^{1/q} + \big(\EE\|h_{\alpha}(\F_{k,\{0,-1,\ldots\}}) - h_{\alpha}(\F_{k,\{0\}})\|_2^q\big)^{1/q}\nonumber \\
    & \le (c_q+c_q')r_{\alpha,q}^k.
\end{align}
As a direct result, we have finite cumulative dependence measure defined in (\ref{eq_cumulative_dep_measure}), i.e., 
\begin{equation}
    \label{eq_cumulative_dep_measure_gd}
    \Theta_{m,q}(\alpha)=\sum_{k=m}^{\infty}\theta_{k,q}(\alpha)=O(r_{\alpha,q}^m)<\infty.
\end{equation}
Therefore, for the constant learning rate $\alpha>0$ satisfying the assumptions in Theorem~\ref{thm_gmc_gd}, the stationary GD dropout sequence $\{\tilde\bbeta_k^{\circ}(\alpha)\}_{k\in\NN}$ meets the short-range dependence requirement in (\ref{eq_short_range_dep_gd}). Consequently, the condition (\ref{eq_shao_wu_short_range}) in Lemma~\ref{lemma_thm21_shao_wu_2007} is satisfied, which along with the Cramér-Wold device yields the central limit theorem for $\tilde S_n^{\circ}(\alpha)$ defined in (\ref{eq_agd_partial_sum_stationary_mean0}), that is,
\begin{equation}
    n^{-1/2}\tilde S_n^{\circ}(\alpha) \Rightarrow \N(0,\Sigma(\alpha)),
\end{equation}
where the long-run covariance matrix $\Sigma(\alpha)$ is defined in Theorem \ref{thm_clt_ave_gd}.

Next, we bound the difference between $\tilde S_n^{\circ}(\alpha)$ and $\tilde S_n^{\tilde\bbeta_0}(\alpha)$ for any arbitrarily fixed $\tilde\bbeta_0\in\RR^d$ in the $q$-th moment, for all $q\ge2$. For the constant learning rate $\alpha>0$ satisfying $\alpha\|\XX\|<2$, applying Theorem~\ref{thm_gmc_gd} yields
\begin{align}
    \label{eq_thm_agd_clt_diff}
    & \quad \Big(\EE\big\|\tilde S_n^{\circ}(\alpha) - \tilde S_n^{\tilde\bbeta_0}(\alpha)\big\|_2^q\Big)^{1/q} \nonumber \\
    & = \Big(\EE\big\|\big[f_{D_1}(\tilde\bbeta_0^{\circ}(\alpha)) + f_{D_2}\circ f_{D_1}(\tilde\bbeta_0^{\circ}(\alpha)) + \cdots + f_{D_n}\circ\cdots\circ f_{D_1}(\tilde\bbeta_0^{\circ}(\alpha))\big] \nonumber \\
    & \quad - \big[f_{D_1}(\tilde\bbeta_0(\alpha)) + f_{D_2}\circ f_{D_1}(\tilde\bbeta_0(\alpha)) + \cdots + f_{D_n}\circ\cdots\circ f_{D_1}(\tilde\bbeta_0(\alpha))\big]\big\|_2^q\Big)^{1/q} \nonumber \\
    & \le \Big(\sum_{k=1}^nr_{\alpha,q}^k\Big)\|\tilde\bbeta_0^{\circ} - \tilde\bbeta_0\|_2.
\end{align}
Since the contraction constant $r_{\alpha,q}\in(0,1)$, we can derive the limit for the sum of the geometric series $\{r_{\alpha,q}^k\}_{k=1}^n$ as follows
\begin{align}
    \lim_{n\rightarrow\infty}\sum_{k=1}^nr_{\alpha,q}^k = \lim_{n\rightarrow\infty}\frac{r_{\alpha,q}\big(1-r_{\alpha,q}^n\big)}{1-r_{\alpha,q}} = \frac{r_{\alpha,q}}{1-r_{\alpha,q}}.
\end{align}
This, together with (\ref{eq_thm_agd_clt_diff}) gives
\begin{align}
    \label{eq_thm_agd_clt_diff_rate}
    \big(\EE\big\|\tilde S_n^{\circ}(\alpha) - \tilde S_n^{\tilde\bbeta_0}(\alpha)\big\|_2^q\big)^{1/q} = O(1) = o(\sqrt{n}),
\end{align}
which yields the quenched central limit theorem for the partial sum $\tilde S_n^{\tilde\bbeta_0}(\alpha)$ defined in (\ref{eq_agd_partial_sum_mean0}), that is, for any fixed initial point $\tilde\bbeta_0\in\RR^d$,
\begin{equation}
    n^{-1/2}\tilde S_n^{\tilde\bbeta_0}(\alpha) \Rightarrow \N(0,\Sigma(\alpha)).
\end{equation}

Finally, we shall show that $\|\sum_{k=1}^n\EE[\tilde\bbeta_k(\alpha) - \tilde\bbeta]\|_2=o(\sqrt{n})$. 
To see this, we note that given two independently chosen initial vectors $\tilde\bbeta_0$ and $\tilde\bbeta_0^{\circ}$, where $\tilde\bbeta_0^{\circ}$ follows the stationary distribution $\pi_{\alpha}$ while $\tilde\bbeta_0$ is an arbitrary initial point in $\RR^d$, it follows from the triangle inequality that
\begin{align}
    \label{eq_expectation_two_parts_gd}
    \Big\|\sum_{k=1}^n\EE[\tilde\bbeta_k(\alpha) - \tilde\bbeta]\Big\|_2 & = \Big\|\EE\Big[\sum_{k=1}^n\big(\tilde\bbeta_k(\alpha) - \tilde\bbeta_k^{\circ}(\alpha) + \tilde\bbeta_k^{\circ}(\alpha)- \tilde\bbeta\big) \Big]\Big\|_2 \nonumber \\
    & \le \Big\|\EE\Big[\sum_{k=1}^n\big(\tilde\bbeta_k(\alpha) - \tilde\bbeta_k^{\circ}(\alpha)\big) \Big]\Big\|_2 + \Big\|\EE\Big[\sum_{k=1}^n\big(\tilde\bbeta_k^{\circ}(\alpha) - \tilde\bbeta\big)\Big]\Big\|_2 \nonumber \\
    & =: \III_1 + \III_2.
\end{align}
We first show $\III_2=0$. Recall the representation of $\{\tilde\bbeta_k^{\circ}(\alpha)-\tilde\bbeta\}_{k\in\NN}$ in (\ref{eq_iter_dropout_stationary0}). Since $\EE[A_k(\alpha)]=\EE[I_d-\alpha D_k\XX D_k]=I_d-\alpha p\XX_p$ and $\EE[\bm{b}_k(\alpha)]=0$ by (\ref{eq_bk_mean0_gd}), it follows that
\begin{equation}
    \EE[\tilde\bbeta_k^{\circ}(\alpha) - \tilde\bbeta]=(I_d-\alpha p\XX_p)\EE[\tilde\bbeta_{k-1}^{\circ}(\alpha) - \tilde\bbeta].
\end{equation}
Thus, due to the stationarity of $\{\tilde\bbeta_k^{\circ}(\alpha)\}_{k\in\NN}$ and the non-singularity of $\XX_p$, we obtain that uniformly over $k\in\NN$,
\begin{equation}
    \label{eq_l2_stationary_mean0_gd}
    \EE[\tilde\bbeta_k^{\circ}(\alpha) - \tilde\bbeta] = 0.
\end{equation}
As a direct consequence,
\begin{align}
    \III_2 = \Big\|\EE\Big[\sum_{k=1}^n\big(\tilde\bbeta_k^{\circ}(\alpha) - \tilde\bbeta\big)\Big]\Big\|_2 & = \Big\|\sum_{k=1}^n\EE[\tilde\bbeta_k^{\circ}(\alpha) - \tilde\bbeta] \Big\|_2= 0.
\end{align}
In addition, for the part $\III_1$, it follows from Jensen's inequality and (\ref{eq_thm_agd_clt_diff}) that
\begin{align}
    \III_1 & = \Big\|\EE\Big[\sum_{k=1}^n\big(\tilde\bbeta_k(\alpha) - \tilde\bbeta_k^{\circ}(\alpha)\big)\Big]\Big\|_2 \nonumber \\
    & \le \Big(\EE\Big\|\sum_{k=1}^n\big(\tilde\bbeta_k(\alpha)-\tilde\bbeta_k^{\circ}(\alpha)\big)\Big\|_2^2\Big)^{1/2} \nonumber \\
    & \le \Big(\sum_{k=1}^nr_{\alpha,2}^k\Big)\big\|\tilde\bbeta_0 - \tilde\bbeta_0^{\circ}\big\|_2.
\end{align}
By inserting the results of parts $\III_1$ and $\III_2$ back to (\ref{eq_expectation_two_parts_gd}), we obtain
\begin{equation}
    \label{eq_clt_gd_expectation}
    \Big\|\sum_{k=1}^n\EE[\tilde\bbeta_k(\alpha) - \tilde\bbeta]\Big\|_2 \le \Big(\sum_{k=1}^nr_{\alpha,2}^k\Big)\big\|\tilde\bbeta_0 - \tilde\bbeta_0^{\circ}\big\|_2.
\end{equation}
which remains bounded as $n\rightarrow\infty$ when $\alpha\|\XX\|<2$ by Theorem~\ref{thm_gmc_gd}. This completes the proof.
\end{proof}

\subsection{Proof of Corollary~\ref{cor_clt_agd_parallel}}

\begin{proof}[Proof of Corollary~\ref{cor_clt_agd_parallel}]
Recall the stationary GD dropout sequence $\{\tilde\bbeta_k^{\circ}(\alpha)\}_{k\in\NN}$ which follows the unique stationary distribution $\pi_{\alpha}$. Since this sequence satisfies the short-range dependence condition as stated in (\ref{eq_short_range_dep_gd}), it follows from Lemma~\ref{lemma_thm21_shao_wu_2007} and the Cramér-Wold device that any fixed linear combination of the coordinates of $\tilde S_n^{\circ}(\alpha)$ in (\ref{eq_agd_partial_sum_stationary_mean0}) converges to the corresponding linear combination of normal vectors in distribution. Then, the CLT for the averaged GD dropout with multiple learning rates holds by applying the Cramér-Wold device again, that is,
\begin{align}
    n^{-1/2}\mathrm{vec}\big(\tilde S_n^{\circ}(\alpha_1),\ldots,\tilde S_n^{\circ}(\alpha_s)\big) \Rightarrow \N(0,\Sigma^{\mathrm{vec}}).
\end{align}
Then, following the similar arguments in the proof of Theorem~\ref{thm_clt_ave_gd}, we obtain the quenched CLT for $\mathrm{vec}\big(\bar\bbeta_n^{\gd}(\alpha_1)-\tilde\bbeta,\ldots,\bar\bbeta_n^{\gd}(\alpha_s)-\tilde\bbeta\big)$. We omit the details here.
\end{proof}

\subsection{Proof of Theorem~\ref{thm_fclt_agd}}

\begin{proof}[Proof of Theorem~\ref{thm_fclt_agd}]
Recall the stationary GD dropout sequence $\{\tilde\bbeta_k^{\circ}(\alpha)\}_{k\in\NN}$ in (\ref{eq_iter_dropout_stationary0}), where $\tilde\bbeta_k^{\circ}(\alpha)$ follows the stationary distribution $\pi_{\alpha}$ for all $k\in\NN$. Also, recall the centering term $\tilde\bbeta_{\infty}(\alpha)=\EE[\tilde\bbeta_1^{\circ}(\alpha)]$ as defined in (\ref{eq_center_gd}). By (\ref{eq_l2_stationary_mean0_gd}), we have $\EE[\tilde\bbeta_k^{\circ}(\alpha) - \tilde\bbeta] = 0$ uniformly over $k\in\NN$. Hence,
\begin{align}
    \big(\EE\|\tilde\bbeta_{\infty}(\alpha) - \tilde\bbeta\|_2^q\big)^{1/q} = \big(\EE\|\EE[\tilde\bbeta_1^{\circ}(\alpha)-\tilde\bbeta]\|_2^q\big)^{1/q} = 0.
\end{align}
This, along with Assumption~\ref{asm_fclt_agd} and Lemma~\ref{lemma_q_moment_gd} gives, for $q>2$,
\begin{align}
    \label{eq_fclt_gd_order}
    \big(\EE\|\tilde\bbeta_k^{\circ}(\alpha)-\tilde\bbeta_{\infty}(\alpha)\|_2^q\big)^{1/q} & \le  \big(\EE\|\tilde\bbeta_k^{\circ}(\alpha) - \tilde\bbeta\|_2^q\big)^{1/q} + \big(\EE\|\tilde\bbeta_{\infty}(\alpha) - \tilde\bbeta\|_2^q\big)^{1/q} = O(\sqrt{\alpha}).
\end{align}
Moreover, we notice that by Markov's inequality and (\ref{eq_fclt_gd_order}), for any $u\in\RR$ and $\delta>0$, we have
\begin{align}
    \label{eq_fclt_gd_cond1}
    & \quad \sup_{k\ge1}\EE\big[\|\tilde\bbeta_k^{\circ}(\alpha)-\tilde\bbeta_{\infty}(\alpha)\|_2^q\One_{\|\tilde\bbeta_k^{\circ}(\alpha)-\tilde\bbeta_{\infty}(\alpha)\|_2\ge u}\big] \nonumber \\
    & \le \sup_{k\ge1}\EE\Big[\|\tilde\bbeta_k^{\circ}(\alpha)-\tilde\bbeta_{\infty}(\alpha)\|_2^q \cdot \|\tilde\bbeta_k^{\circ}(\alpha)-\tilde\bbeta_{\infty}(\alpha)\|_2^{\delta}/u^{\delta}\Big] \nonumber \\
    & = \sup_{k\ge1}\EE\big[\|\tilde\bbeta_k^{\circ}(\alpha)-\tilde\bbeta_{\infty}(\alpha)\|_2^{q+\delta}\big]/u^{\delta} \nonumber \\
    & = O\big\{\alpha^{(q+\delta)/2}/u^{\delta}\big\},
\end{align}
which converges to 0 as $u\rightarrow\infty$. Therefore, condition (i) in Lemma~\ref{lemma_thm2_karmakar_wu_2020} is satisfied. Since $\{\tilde\bbeta_k^{\circ}(\alpha)\}_{k\in\NN}$ is stationary, following the arguments below Lemma~\ref{lemma_thm2_karmakar_wu_2020}, condition (ii) is not required. Regarding condition (iii), for the constant learning rate $\alpha>0$ satisfying $\alpha\|\XX\|<2$, it follows from Assumption~\ref{asm_fclt_agd} and (\ref{eq_functional_dep_measure_gd}) that the functional dependence measure $\theta_{k,q}(\alpha)\le c\cdot r_{\alpha,q}^k$, for all $q>2$ and $k\in\NN$, where the constant $c>0$ is independent of $k$. Consequently, there exists a constant $\kappa>0$ such that the tail cumulative dependence measure of $\{\tilde\bbeta_k^{\circ}(\alpha)\}_{k\in\NN}$ can be bounded by
\begin{align}
    \Theta_{m,q}(\alpha) = \sum_{k=m}^{\infty}\theta_{k,q}(\alpha) = O\big\{m^{-\chi}\big(\log(m)\big)^{-\kappa}\big\},
\end{align}
where $\chi>0$ is some constant that can be taken to be arbitrarily large. Then, the condition (iii) in Lemma~\ref{lemma_thm2_karmakar_wu_2020} is satisfied. 

Thus, we obtain the invariance principle for the stationary partial sum process $(\tilde S_i^{\circ}(\alpha))_{1\le i\le n}$ defined in (\ref{eq_agd_partial_sum_stationary_mean0}). That is, there exists a (richer) probability space $(\tilde\Omega^{\star},\tilde\A^{\star},\tilde\PP^{\star})$ on which we can define random vectors $\tilde\bbeta_k^{\star}$'s with the partial sum process $\tilde S_i^{\star}=\sum_{k=1}^i(\tilde\bbeta_k^{\star}-\tilde\bbeta_{\infty})$, and a Gaussian process $\tilde G_i^{\star}=\sum_{k=1}^i\tilde{\bm{z}}_k^{\star}$, where $\tilde{\bm{z}}_k^{\star}$'s are independent Gaussian random vectors in $\RR^d$ following $\N(0,I_d)$, such that
\begin{equation}
    (\tilde S_i^{\star})_{1\le i\le n} \overset{\D}{=} (\tilde S_i^{\circ})_{1\le i\le n},
\end{equation}
and
\begin{equation}
    \label{eq_fclt_gd_GS1}
    \max_{1\le i\le n}\big\|\tilde S_i^{\star} - \Sigma^{1/2}(\alpha)\tilde G_i^{\star}\big\|_2 = o_{\PP}(n^{1/q}), \quad \text{in }(\tilde\Omega^{\star},\tilde\A^{\star},\tilde\PP^{\star}),
\end{equation}
where the long-run covariance matrix $\Sigma(\alpha)$ is defined in Theorem \ref{thm_clt_ave_gd}.

Next, recall the partial sum $\tilde S_i^{\tilde\bbeta_0}(\alpha)=\sum_{k=1}^n[\tilde\bbeta_k(\alpha)-\tilde\bbeta_{\infty}(\alpha)] $ as defined in (\ref{eq_agd_partial_sum_mean0}), given an arbitrarily fixed initial point $\tilde\bbeta_0\in\RR^d$. It follows from the triangle inequality that
\begin{align}
    \label{eq_fclt_S_center}
    & \quad \Big(\EE\Big[\max_{1\le i\le n}\big\|\tilde S_i^{\tilde\bbeta_0}(\alpha) - \Sigma^{1/2}(\alpha)\tilde G_i^{\star}\big\|_2\Big]^q\Big)^{1/q} \nonumber \\
    & = \Big(\EE\Big[\max_{1\le i\le n}\big\|\tilde S_i^{\tilde\bbeta_0}(\alpha) - \tilde S_i^{\circ}(\alpha) + \tilde S_i^{\circ}(\alpha) - \Sigma^{1/2}(\alpha)\tilde G_i^{\star}\big\|_2\Big]^q\Big)^{1/q} \nonumber \\
    & \le \Big(\EE\Big[\max_{1\le i\le n}\big\|\tilde S_i^{\tilde\bbeta_0}(\alpha) - \tilde S_i^{\circ}(\alpha)\big\|_2 + \max_{1\le i\le n}\big\|\tilde S_i^{\circ}(\alpha) - \Sigma^{1/2}(\alpha)\tilde G_i^{\star}\big\|_2\Big]^q\Big)^{1/q} \nonumber \\
    & \le \Big(\EE\Big[\max_{1\le i\le n}\big\|\tilde S_i^{\tilde\bbeta_0}(\alpha) - \tilde S_i^{\circ}(\alpha)\big\|_2\Big]^q\Big)^{1/q} + \Big(\EE\Big[\max_{1\le i\le n}\big\|\tilde S_i^{\circ}(\alpha) - \Sigma^{1/2}(\alpha)\tilde G_i^{\star}\big\|_2\Big]^q\Big)^{1/q}.
\end{align}
Therefore, to show the invariance principle for $\big(\tilde S_i^{\tilde\bbeta_0}(\alpha)\big)_{1\le i\le n}$, it suffices to bound the difference part $\max_{1\le i\le n}\|\tilde S_i^{\tilde\bbeta_0}(\alpha) - \tilde S_i^{\circ}(\alpha)\|_2$ in terms of the $q$-th moment. To this end, recall the iterative function $f_D(\bbeta)=\bbeta+\alpha DX^{\top}(\bm{y}-XD\bbeta)$ in (\ref{eq_dropout_function_gd}) that rewrites the GD dropout recursion (\ref{eq_dropout_gd_origin}). We note that
\begin{align}
    \label{eq_fclt_gd_diff}
    & \quad \max_{1\le i\le n}\big\|\tilde S_i^{\tilde\bbeta_0}(\alpha) - \tilde S_i^{\circ}(\alpha)\big\|_2 \nonumber \\
    & = \max_{1\le i\le n}\Big\|\big[f_{D_1}(\tilde\bbeta_0) - f_{D_1}(\tilde\bbeta_0^{\circ})\big] + \cdots + \big[f_{D_i}\circ\cdots\circ f_{D_1}(\tilde\bbeta_0) - f_{D_i}\circ\cdots\circ f_{D_1}(\tilde\bbeta_0^{\circ})\big] \Big\|_2 \nonumber \\
    & \le \max_{1\le i\le n} \Big(\big\|f_{D_1}(\tilde\bbeta_0) - f_{D_1}(\tilde\bbeta_0^{\circ})\big\|_2 + \cdots + \big\|f_{D_i}\circ\cdots\circ f_{D_1}(\tilde\bbeta_0) - f_{D_i}\circ\cdots\circ f_{D_1}(\tilde\bbeta_0^{\circ})\big\|_2\Big) \nonumber \\
    & = \big\|f_{D_1}(\tilde\bbeta_0) - f_{D_1}(\tilde\bbeta_0^{\circ})\big\|_2 + \cdots + \big\|f_{D_n}\circ\cdots\circ f_{D_1}(\tilde\bbeta_0) - f_{D_n}\circ\cdots\circ f_{D_1}(\tilde\bbeta_0^{\circ})\big\|_2.
\end{align}
This, along with the triangle inequality and Theorem~\ref{thm_gmc_gd} yields
\begin{align}
    \label{eq_fclt_gd_gmc_approx}
    & \quad \Big(\EE\Big[\max_{1\le i\le n}\big\|\tilde S_i^{\tilde\bbeta_0}(\alpha) - \tilde S_i^{\circ}(\alpha)\big\|_2\Big]^q\Big)^{1/q} \nonumber \\
    & \le \big(\EE\big\|f_{D_1}(\tilde\bbeta_0) - f_{D_1}(\tilde\bbeta_0^{\circ})\big\|_2^q\big)^{1/q} + \cdots + \big(\EE\big\|f_{D_n}\circ\cdots\circ f_{D_1}(\tilde\bbeta_0) - f_{D_n}\circ\cdots\circ f_{D_1}(\tilde\bbeta_0^{\circ})\big\|_2^q\big)^{1/q} \nonumber \\
    & \le \frac{r_{\alpha,q}(1-r_{\alpha,q}^n)}{1-r_{\alpha,q}}\|\tilde\bbeta_0-\tilde\bbeta_0^{\circ}\|_2 = o(n^{1/q}).
\end{align}
We insert this result back into (\ref{eq_fclt_gd_diff}), which together with (\ref{eq_fclt_S_center}) gives the invariance principle for the partial sum $\tilde S_i^{\tilde\bbeta_0}(\alpha)=\sum_{k=1}^n[\tilde\bbeta_k(\alpha)-\tilde\bbeta_{\infty}(\alpha)] $. 

Finally, let the partial sum $S_i^{\tilde\bbeta_0}(\alpha)=\sum_{k=1}^n[\tilde\bbeta_k(\alpha)-\tilde\bbeta]$ be as defined in Theorem~\ref{thm_fclt_agd}. We shall bound the difference between $S_i^{\tilde\bbeta_0}(\alpha)$ and $\tilde S_i^{\tilde\bbeta_0}(\alpha)$. Since $\tilde\bbeta_{\infty}(\alpha)=\EE_D[\tilde\bbeta_1^{\circ}(\alpha)]$ as defined in (\ref{eq_center_gd}) and $\tilde\bbeta_{\infty}(\alpha)-\tilde\bbeta=\EE_D[\tilde\bbeta_1^{\circ}(\alpha)-\tilde\bbeta]=0$ by (\ref{eq_l2_stationary_mean0_gd}), it follows that
\begin{align}
    \label{eq_fclt_gd_mean_approx}
    \Big(\EE\Big[\max_{1\le i\le n}\big\|\tilde S_i^{\tilde\bbeta_0}(\alpha) - S_i^{\tilde\bbeta_0}(\alpha)\big\|_2\Big]^q\Big)^{1/q}  = \Big(\EE\Big[\max_{1\le i\le n}\big\|\sum_{k=1}^n \tilde\bbeta_{\infty}(\alpha)-\tilde\bbeta\big\|_2\Big]^q\Big)^{1/q} = 0.
\end{align}
Combining this with the invariance principle for $(\tilde S_i^{\tilde\bbeta_0}(\alpha))_{1\le i\le n}$, we obtain the same approximation rate $o_{\PP}(n^{1/q})$ for the partial sum process $(S_i^{\tilde\bbeta_0}(\alpha))_{1\le i\le n}$. This completes the proof.
\end{proof}

\section{Proofs in Section~\ref{subsec_gmc_sgd}}

Recall the SGD dropout sequence $\{\breve\bbeta_k(\alpha)\}_{k\in\NN}$ and the random coefficient $\breve A_k(\alpha)$ in (\ref{eq_dropout_sgd}). To prove that $\sup_{\bm{v}\in\RR^d,\|\bm{v}\|_2=1}\EE\|\breve A_k(\alpha)\bm{v}\|_2^q<1$ is a sufficient condition for the geometric-moment contraction (GMC) of the SGD dropout sequence, we first introduce two useful moment inequalities in Lemma~\ref{lemma_moment_inequality}. 
\subsection{Proof of Lemma~\ref{lemma_gmc_sgd_cond}}

\begin{lemma}[Moment inequality]\label{lemma_moment_inequality}
    Let $q\ge2$. For any two random vectors $\bm{x}$ and $\bm{y}$ in $\RR^d$ with fixed $d\ge1$, the following inequalities holds:
    
    \noindent(i) $\EE\Big|\|\bm{x} + \bm{y}\|_2^q - \|\bm{x}\|_2^q - q\|\bm{x}\|_2^{q-2}\bm{x}^{\top}\bm{y}\Big| \le \EE\big(\|\bm{x}\|_2 + \|\bm{y}\|_2\big)^q - \EE\|\bm{x}\|_2^q - q\EE(\|\bm{x}\|_2^{q-1}\|\bm{y}\|_2)$.

    \noindent(ii) $\EE\Big|\|\bm{x} + \bm{y}\|_2^q - \|\bm{x}\|_2^q - q\|\bm{x}\|_2^{q-2}\bm{x}^{\top}\bm{y}\Big| \le \big[(\EE\|\bm{x}\|_2^q)^{1/q} + (\EE\|\bm{y}\|_2^q)^{1/q}\big]^q - \EE\|\bm{x}\|_2^q - q(\EE\|\bm{x}\|_2^q)^{(q-1)/q}(\EE\|\bm{y}\|_2^q)^{1/q}$.
\end{lemma}
Lemma~\ref{lemma_moment_inequality}(i) immediately follows if we can prove the inequality
\begin{equation}
    \Big|\|\bm{x} + \bm{y}\|_2^q - \|\bm{x}\|_2^q - q\|\bm{x}\|_2^{q-2}\bm{x}^{\top}\bm{y}\Big| \le \big(\|\bm{x}\|_2 + \|\bm{y}\|_2\big)^q - \|\bm{x}\|_2^q - q(\|\bm{x}\|_2^{q-1}\|\bm{y}\|_2),
\end{equation}
which is of independent interest. The right hand side of the inequality in Lemma~\ref{lemma_moment_inequality}(ii) only depends on expectations of either one of the random vectors $\bm{x}$ or $\bm{y}.$ This makes the inequality particularly useful if $\bm{x}$ and $\bm{y}$ are dependent. Lemma~\ref{lemma_moment_inequality}(i) is more favorable in cases where $\bm{x}$ and $\bm{y}$ are independent, or if one vector is deterministic. 

\begin{proof}[Proof of Lemma~\ref{lemma_moment_inequality}(i)]
We can assume that $\|\bm{x}\|_2>0$ and $\|\bm{y}\|_2>0$ as otherwise, the inequality holds trivially. It is moreover sufficient to assume $\bm{x}=w\bm{e}_1$, for a positive number $w$ and $\bm{e}_1\in\RR^d$ a unit vector. Then, we can find two numbers $u,v$ such that
\begin{equation}
    \bm{y} = u\cdot(w\bm{e}_1) + v\bm{e}_2,
\end{equation}
where $\bm{e}_2\in\RR^d$ is a unit vector orthogonal to $\bm{e}_1$. Let $r=\|\bm{y}\|_2 = \sqrt{(uw)^2 + v^2}>0$. We note that $\bm{x}^{\top}\bm{y}=uw^2$ and $\|\bm{x}+\bm{y}\|_2^2=(1+u)^2w^2 + v^2$, which gives
\begin{align}
    \label{eq_lemma_ineq_analytic}
    & \quad \|\bm{x} + \bm{y}\|_2^q - \|\bm{x}\|_2^q - q\|\bm{x}\|_2^{q-2}\bm{x}^{\top}\bm{y} \nonumber \\
    & = \big[(1+u)^2w^2 + v^2\big]^{q/2} - w^q -qw^{q-2}uw^2 \nonumber \\
    & = (w^2 + 2uw^2 +r^2)^{q/2}-w^q-quw^q.
\end{align}
Since $r=\sqrt{(uw)^2 + v^2}$, we can rewrite $uw=r\delta$ for some scalar $\delta$ with $\delta\in[-\delta^*,1]$, where $\delta^*=(w^2+r^2)/(2wr)$, i.e., $w^2-2wr\delta^*+r^2=0$. Here, $|\delta|$ can be viewed as the projection length of $\bm{y}$ on the direction of $\bm{x}$, and the end point $\delta^*$ falls in $[0,1]$. Then, (\ref{eq_lemma_ineq_analytic}) can be rewritten into
\begin{align}
    \varphi(\delta) := (w^2+2wr\delta+r^2)^{q/2} - w^q-qw^{q-1}r\delta.
\end{align}
Recall that $w=\|\bm{x}\|_2>0$ and $r=\|\bm{y}\|_2>0$. The first order derivative of $\varphi(\delta)$ is
\begin{align}
    \varphi'(\delta) & = \frac{q}{2}2wr(w^2+2wr\delta+r^2)^{(q/2)-1} - qw^{q-1}r \nonumber \\
    & = qwr\big[(w^2+2wr\delta+r^2)^{(q/2)-1} - w^{q-2}\big] \nonumber \\
    & = qw^{q-1}r\Big[\Big(1 + \frac{2r\delta}{w} + \frac{r^2}{w^2}\Big)^{(q/2)-1} - 1\Big].
\end{align}
This indicates that, for $q\ge2$, $\varphi'(\delta)\le0$ when $\delta\in[-\delta^*,r/(2w)]$, and $\varphi'(\delta)>0$ when $\delta\in(-r/(2w),1]$. In particular, by Bernoulli's inequality, we can observe that
\begin{align}
    \varphi(-\delta^*) & = - w^q + qw^{q-1}r\delta^* = \big(q/2-1\big)w^q + qw^{q-2}r^2 >0, \nonumber \\
    \varphi(-r/(2w)) & = -qw^{q-2}r^2/2 <0,\nonumber \\
    \varphi(1) & = (w+r)^q - w^q -qw^{q-1}r >0.
\end{align}
Moreover, regarding $\varphi(-\delta^*)$ on $\delta^*\in[0,1]$, we consider a new function $\tilde\varphi(s)=-w^q-qw^{q-1}rs$, which is decreasing on $s\in[-1,0]$. Note that $\tilde\varphi(-1)=-w^q+qw^{q-1}r$. Thus, by comparison, we have $- \varphi(-r/(2w)) <\varphi(-\delta^*) \le \tilde\varphi(-1) <\varphi(1)$. As a direct result, we obtain
\begin{align}
    \sup_{|\delta|\le1}|\varphi(\delta)| = \max\{\varphi(-\delta^*),-\varphi(-r/(2w)),\varphi(1)\} = \varphi(1).
\end{align}
By inserting $w=\|\bm{x}\|_2$ and $r=\|\bm{y}\|_2$ back to $\varphi(\delta)$, we obtain, for any $\bm{x}$ and $\bm{y}$ in $\RR^d$, 
\begin{align}
    \Big|\|\bm{x} + \bm{y}\|_2^q - \|\bm{x}\|_2^q - q\|\bm{x}\|_2^{q-2}\bm{x}^{\top}\bm{y}\Big| \le (\|\bm{x}\|_2+\|\bm{y}\|_2)^q - \|\bm{x}\|_2^q - q\|\bm{x}\|_2^{q-1}\|\bm{y}\|_2.
\end{align}
The desired result holds by taking the expectation on the both sides.
\end{proof}

\begin{proof}[Proof of Lemma~\ref{lemma_moment_inequality}(ii)]
First, we define a function $\phi(t) = \|\bm{x} + t\bm{y}\|_2^q$ on $t\in[0,\infty)$. The first and second order derivatives of $\phi(t)$ are as follows
\begin{align}
    \label{eq_lemma_ineq_derivative}
    \phi'(t)=\frac{d}{dt}\phi(t) & = q\|\bm{x} + t\bm{y}\|_2^{q-2}\big(t\|\bm{y}\|_2^2 + \bm{x}^{\top}\bm{y}\big), \nonumber \\
    \phi''(t)=\frac{d^2}{dt^2}\phi(t) & = q(q-2)\|\bm{x} + t\bm{y}\|_2^{q-4}\big(t\|\bm{y}\|_2^2 + \bm{x}^{\top}\bm{y}\big)^2 + q\|\bm{x} + t\bm{y}\|_2^{q-2}\|\bm{y}\|_2^2 \nonumber \\
    & = q(q-2)\|\bm{x} + t\bm{y}\|_2^{q-4}\big[(\bm{x} + t\bm{y})^{\top}\bm{y}\big]^2 + q\|\bm{x} + t\bm{y}\|_2^{q-2}\|\bm{y}\|_2^2 \nonumber \\
    & \le q(q-1)\|\bm{x} + t\bm{y}\|_2^{q-2}\|\bm{y}\|_2^2,
\end{align}
where the last inequality follows from the Cauchy-Schwarz inequality. In the previous inequality, equality holds when both random vectors $\bm{x}$ and $\bm{y}$ are scalars. Note that $\phi(1)=\|\bm{x} + \bm{y}\|_2^q$, $\phi(0)=\|\bm{x}\|_2^q$, and $\phi'(0)=q\|\bm{x}\|_2^{q-2}\bm{x}^{\top}\bm{y}$. Since
\begin{align}
    \phi(1)-\phi(0) - \phi'(0) & = \int_{t=0}^1\phi'(t)dt - \phi'(0) \nonumber \\
    & = \int_{t=0}^1\Big(\int_{s=0}^t\phi''(s)ds + \phi'(0)\Big)dt - \phi'(0) \nonumber \\
    & = \int_{t=0}^1\int_{s=0}^t\phi''(s)dsdt,
\end{align}
it follows from the upper bound of $\phi''(s)$ in (\ref{eq_lemma_ineq_derivative}) that
\begin{align}
    \|\bm{x} + \bm{y}\|_2^q - \|\bm{x}\|_2^q - q\|\bm{x}\|_2^{q-2}\bm{x}^{\top}\bm{y} \le q(q-1)\int_{t=0}^1\int_{s=0}^t\|\bm{x} + s\bm{y}\|_2^{q-2}\|\bm{y}\|_2^2dsdt.
\end{align}
Taking the expectation yields
\begin{align}
    \label{eq_lemma_ineq_expectation1}
    \EE\Big(\|\bm{x} + \bm{y}\|_2^q - \|\bm{x}\|_2^q - q\|\bm{x}\|_2^{q-2}\bm{x}^{\top}\bm{y}\Big) \le  q(q-1)\int_{t=0}^1\int_{s=0}^t\EE\big(\|\bm{x} + s\bm{y}\|_2^{q-2}\|\bm{y}\|_2^2\big)dsdt.
\end{align}
Furthermore, by Hölder's inequality and the triangle inequality, we obtain
\begin{align}
    \EE\big(\|\bm{x} + s\bm{y}\|_2^{q-2}\|\bm{y}\|_2^2\big) & \le \big(\EE\|\bm{x} + s\bm{y}\|_2^q\big)^{(q-2)/q} \big(\EE\|\bm{y}\|_2^q\big)^{2/q} \nonumber \\
    & \le \Big[\big(\EE\|\bm{x}\|_2^q\big)^{1/q} + s\big(\EE\|\bm{y}\|_2^q\big)^{1/q}\Big]^{q-2}\big(\EE\|\bm{y}\|_2^q\big)^{2/q},
\end{align}
which together with (\ref{eq_lemma_ineq_expectation1}) gives
\begin{align}
    \label{eq_lemma_ineq_gd_result1}
    & \quad \EE\Big(\|\bm{x} + \bm{y}\|_2^q - \|\bm{x}\|_2^q - q\|\bm{x}\|_2^{q-2}\bm{x}^{\top}\bm{y}\Big) \nonumber \\
    & \le q(q-1)\int_{t=0}^1\int_{s=0}^t\Big[\big(\EE\|\bm{x}\|_2^q\big)^{1/q} + s\big(\EE\|\bm{y}\|_2^q\big)^{1/q}\Big]^{q-2}\big(\EE\|\bm{y}\|_2^q\big)^{2/q} dsdt \nonumber \\
    & = q\int_{t=0}^1\Big\{\Big[\big(\EE\|\bm{x}\|_2^q\big)^{1/q} + t\big(\EE\|\bm{y}\|_2^q\big)^{1/q}\Big]^{q-1} - \big(\EE\|\bm{x}\|_2^q\big)^{(q-1)/q}\Big\}\big(\EE\|\bm{y}\|_2^q\big)^{1/q} dt \nonumber \\
    & = \Big[\big(\EE\|\bm{x}\|_2^q\big)^{1/q} + \big(\EE\|\bm{y}\|_2^q\big)^{1/q}\Big]^{q} - \EE\|\bm{x}\|_2^q - q\big(\EE\|\bm{x}\|_2^q\big)^{(q-1)/q}\big(\EE\|\bm{y}\|_2^q\big)^{1/q}.
\end{align}
In addition, recall that by the proof of Lemma~\ref{lemma_moment_inequality}(i), we have
\begin{align*}
    \Big|\|\bm{x} + \bm{y}\|_2^q - \|\bm{x}\|_2^q - q\|\bm{x}\|_2^{q-2}\bm{x}^{\top}\bm{y}\Big| & \le (\|\bm{x}\|_2+\|\bm{y}\|_2)^q - \|\bm{x}\|_2^q - q\|\bm{x}\|_2^{q-1}\|\bm{y}\|_2 \nonumber \\
    & = q(q-1)\int_{t=0}^1\int_{s=0}^t(\|\bm{x}\|_2 + s\|\bm{y}\|_2)^{q-2}\|\bm{y}\|_2^2dsdt.
\end{align*}
Taking expectation on both sides, we obtain
\begin{align}
    \EE\Big|\|\bm{x} + \bm{y}\|_2^q - \|\bm{x}\|_2^q - q\|\bm{x}\|_2^{q-2}\bm{x}^{\top}\bm{y}\Big| & \le q(q-1)\int_0^1\int_0^t\EE\big[(\|\bm{x}\|_2 + s\|\bm{y}\|_2)^{q-2}\|\bm{y}\|_2^2\big] \, dsdt.
\end{align}
It follows from Hölder's inequality and the triangle inequality that
\begin{align}
    \EE\big[(\|\bm{x}\|_2 + s\|\bm{y}\|_2)^{q-2}\|\bm{y}\|_2^2\big] & \le \big[\EE\big(\|\bm{x}\|_2 +s\|\bm{y}\|_2\big)^q\big]^{(q-2)/q}\big(\EE\|\bm{y}\|_2^q\big)^{2/q}  \nonumber \\
    & \le \Big[\big(\EE\|\bm{x}\|_2^q\big)^{1/q} + s\big(\EE\|\bm{y}\|_2^q\big)^{1/q}\Big]^{q-2}\big(\EE\|\bm{y}\|_2^q\big)^{2/q}.
\end{align}
Evaluating the double integral as in  (\ref{eq_lemma_ineq_gd_result1}) yields
\begin{align*}
    & \quad \EE\Big|\|\bm{x} + \bm{y}\|_2^q - \|\bm{x}\|_2^q - q\|\bm{x}\|_2^{q-2}\bm{x}^{\top}\bm{y}\Big| \nonumber \\
    & \le \Big[\big(\EE\|\bm{x}\|_2^q\big)^{1/q} + \big(\EE\|\bm{y}\|_2^q\big)^{1/q}\Big]^{q} - \EE\|\bm{x}\|_2^q - q\big(\EE\|\bm{x}\|_2^q\big)^{(q-1)/q}\big(\EE\|\bm{y}\|_2^q\big)^{1/q}.
\end{align*}
This completes the proof.
\end{proof}

\begin{proof}[Proof of Lemma~\ref{lemma_gmc_sgd_cond}]
Since a dropout matrix $D_k$ is a diagonal matrix with values $0$ and $1$ on the diagonal, $D_k^2=D_k$ and
\begin{align}
    \big\|(I_d-\alpha D_k\XX_k D_k)\bm{v}\big\|_2^2 & = \bm{v}^{\top}\big(I_d - 2\alpha D_k\XX_k D_k + \alpha^2 D_k\XX_k D_k\XX_k D_k\big)\bm{v} \nonumber \\
    & = 1 - 2\alpha \bm{v}^{\top}D_k\XX_k D_k\bm{v} + \alpha^2\bm{v}^{\top}D_k\XX_k D_k^2\XX_k D_k\bm{v} \nonumber \\
    & = 1 - \alpha\bm{v}^{\top}\big[2D_k\XX_k D_k - \alpha D_k\XX_kD_k\XX_kD_k\big]\bm{v} \nonumber \\
    & =: 1-\alpha \bm{v}^{\top}M_k\bm{v},
\end{align}
with
\begin{equation}
    \label{eq_M_k}
    M_k(\alpha) = 2D_k\XX_k D_k - \alpha D_k\XX_kD_k\XX_kD_k.
\end{equation}
Recall the condition on the learning rate $\alpha$ in (\ref{eq_sgd_condtion_q2}), it follows that $\EE[M_k]$ is positive definite (p.d.), which further implies the lower bound $\EE(\alpha \bm{v}^{\top}M_k\bm{v}) \ge \alpha \lambda_{\min}(\EE [M_k]) >0,$ that holds uniformly over all unit vectors $\bm{v}$. As a direct consequence,
\begin{align}
    \label{eq_gmc_sgd_constant_q2}
    \sup_{\bm{v}\in\RR^d,\|\bm{v}\|_2=1}\EE\big\|(I_d-\alpha D_k\XX_k D_k)\bm{v}\big\|_2^2 \le \sup_{\bm{v}\in\RR^d,\|\bm{v}\|_2=1}\big(1- \EE(\alpha \bm{v}^{\top}M_k\bm{v})\big) <1,
\end{align}
proving the result in the case $q=2$.

Next, we shall show that for all $q>2$, we also have 
$$\sup_{\bm{v}\in\RR^d,\|\bm{v}\|_2=1}\EE\|(I_d-\alpha D_k\XX_kD_k)\bm{v}\|_2^q<1.$$
In this case, the techniques in the proof of Lemma~\ref{lemma_gmc_cond} cannot be directly applied due to the randomness of $\XX_k$, and we need to leverage the moment inequalities in Lemma~\ref{lemma_moment_inequality} instead. Specifically, let $\bm{x}$ and $\bm{y}$ in Lemma~\ref{lemma_moment_inequality} be
\begin{equation}
    \label{eq_x_y}
    \bm{x}=\bm{x}(\bm{v})=\bm{v}, \quad \bm{y}=\bm{y}(\bm{v})=D_k\XX_kD_k\bm{v},
\end{equation}
respectively, for $\bm{v}$ a deterministic $d$-dimensional unit vector. It remains to show that for any $q>2$, $\EE\|\bm{x}-\alpha\bm{y}\|_2^q <1$ holds for any arbitrary unit vector $\bm{v}$. By Lemma~\ref{lemma_moment_inequality},
\begin{align}
    \EE\|\bm{x}-\alpha\bm{y}\|_2^q - \|\bm{x}\|_2^q - q\|\bm{x}\|_2^{q-2}\EE(-\bm{x}^{\top}\alpha\bm{y}) \le \EE\big(\|\bm{x}\|_2 + \|\alpha \bm{y}\|_2\big)^q - \|\bm{x}\|_2^q - q\|\bm{x}\|_2^{q-1}\EE\|\alpha \bm{y}\|_2,
\end{align}
which along with $\|\bm{x}\|_2=\|\bm{v}\|_2=1$ further yields, 
\begin{align}
    \label{eq_gmc_sgd_constant1}
    \EE\|\bm{x}-\alpha\bm{y}\|_2^q - 1 + q\alpha\EE(\bm{x}^{\top}\bm{y}) \le \EE\big(1 + \alpha\|\bm{y}\|_2\big)^q - 1 - q\alpha\EE\|\bm{y}\|_2.
\end{align}
Therefore, to prove $\sup_{\bm{v}\in\RR^d,\|\bm{v}\|_2=1}\EE\|\bm{x}-\alpha\bm{y}\|_2^q <1$, it suffices to show
\begin{align}
    \EE\big(1 + \alpha\|\bm{y}\|_2\big)^q - 1 - q\alpha\EE\|\bm{y}\|_2 < q\alpha\EE(\bm{x}^{\top}\bm{y}).
\end{align}
By applying Lemma~\ref{lemma_moment_inequality} again, we have
\begin{equation}
    \label{eq_gmc_sgd_constant2}
    \EE\big(1 + \alpha\|\bm{y}\|_2\big)^q - 1 - q\alpha\EE\|\bm{y}\|_2 \le \big(1 + \alpha(\EE\|\bm{y}\|_2^q)^{1/q}\big)^q - 1 - q\alpha(\EE\|\bm{y}\|_2^q)^{1/q}.
\end{equation}
Thus, we only need to show that for any $d$-dimensional vector  $\bm{v}$,
\begin{align}
    \big(1 + \alpha(\EE\|\bm{y}\|_2^q)^{1/q}\big)^q - 1 - q\alpha(\EE\|\bm{y}\|_2^q)^{1/q} < q\alpha\EE(\bm{x}^{\top}\bm{y}).
\end{align}
Recall the definitions of $\bm{x}$ and $\bm{y}$ in (\ref{eq_x_y}). By Lemma~\ref{lemma_clara}~(i), it follows that $\EE(\bm{x}^{\top}\bm{y}) = p\bm{v}^{\top}\EE[\XX_{k,p}]\bm{v}$, where $\XX_{k,p} = p\XX_k + (1-p)\mathrm{Diag}(\XX_k)$. With $\mu_q=\mu_q(\bm{v}) = (\EE\|\bm{y}(\bm{v})\|_2^q)^{1/q} = (\EE\|D_k\XX_kD_k\bm{v}\|_2^q)^{1/q},$ it suffices to show that
\begin{align}
    \sup_{\bm{v}\in\RR^d,\|\bm{v}\|_2=1}\frac{\big(1 + \alpha\mu_q\big)^q -1-q\alpha\mu_q}{p\bm{v}^{\top}\EE[\XX_{k,p}]\bm{v}} < q\alpha.
    \label{eq.cjdsbh}
\end{align}
Let $q>2$. Consider a function $f:\,\RR_+\mapsto\RR$ with $f(t)=(1+t)^q -1-qt$, which is strictly increasing on $\RR_+$. Note that $f''(t) = q(q-1)(1+t)^{q-2}$. Since $q>2$, it follows that $f(t)=\int_{s=0}^{t}\int_{u=0}^sq(q-1)(1+u)^{q-2}duds$. Therefore,
\begin{align}
    f(t)=(1+t)^q -1-qt \le \frac{q(q-1)}{2}(1+t)^{q-2}t^2.
\end{align}
Thus, the condition \eqref{eq.cjdsbh} is satisfied if
\begin{align}
    \frac{q(q-1)}{2}\alpha^2\sup_{\bm{v}\in\RR^d,\|\bm{v}\|_2=1}\frac{\big(1 + \alpha\mu_q\big)^{q-2}\mu_q^2}{p\bm{v}^{\top}\EE[\XX_{k,p}]\bm{v}} <q\alpha.
\end{align}
As this is true by assumption the proof is complete. 
\end{proof}

\subsection{Proof of Theorem~\ref{thm_gmc_sgd}}

\begin{proof}[Proof of Theorem~\ref{thm_gmc_sgd}]
Let the random coefficient matrix $\breve A_k(\alpha)=I_d-\alpha D_k\XX_kD_k$ be as defined in (\ref{eq_dropout_sgd}). We write $\breve A_k(\alpha)=\breve A_k$ exchangeably in this proof. 

First, we study the case with $q=2$. Consider two SGD dropout sequences $\{\breve\bbeta_k(\alpha)\}_{k\in\NN}$ and $\{\breve\bbeta_k'(\alpha)\}_{k\in\NN}$, given two arbitrarily fixed initial vectors $\breve\bbeta_0,\,\breve\bbeta_0'$. Let $\breve{\bm{\delta}} = \breve\bbeta_0-\breve\bbeta_0'$. Then, it follows from the tower rule that
\begin{align}
    \EE\|\breve\bbeta_k(\alpha) - \breve\bbeta_k'(\alpha)\|_2^2 & = \EE\big[\breve{\bm{\delta}}^{\top}\breve A_1^{\top}\cdots\breve A_k^{\top}\breve A_k\cdots\breve A_1\breve{\bm{\delta}}\big] \nonumber \\
    & = \EE\big[ \EE\big[\breve{\bm{\delta}}^{\top}\breve A_1^{\top}\cdots\breve A_k^{\top}\breve A_k\cdots\breve A_1\breve{\bm{\delta}} \mid \breve A_1,\ldots,\breve A_{k-1} \big] \big] \nonumber \\
    & = \EE\big[ \breve{\bm{\delta}}^{\top}\breve A_1^{\top}\cdots\breve A_{k-1}^{\top}\EE(\breve A_k^{\top}\breve A_k)\breve A_{k-1}\cdots\breve A_1\breve{\bm{\delta}}\big] \nonumber \\
    & \le \|\EE(\breve A_k^{\top}\breve A_k)\| \cdot  \EE\big[ \breve{\bm{\delta}}^{\top}\breve A_1^{\top}\cdots\breve A_{k-1}^{\top}\breve A_{k-1}\cdots\breve A_1\breve{\bm{\delta}}\big] \nonumber \\
    & \le \prod_{i=1}^k\|\EE(\breve A_i^{\top}\breve A_i)\|\cdot\|\breve{\bm{\delta}}\|_2^2.
\end{align}
Recall that for the constant learning rate $\alpha>0$ satisfying the conditions in Lemma~\ref{lemma_gmc_sgd_cond}, we have $\|\EE(\breve A_i^{\top}\breve A_i)\|<1$ uniformly over $i\in\NN$. Thus,
$\prod_{i=1}^k\|\EE(\breve A_i^{\top}\breve A_i)\|<1$. Since the dropout matrices $D_k$'s are i.i.d.\ and are independent of the i.i.d.\ observations $\bm{x}_k$'s, it follows that $\prod_{i=1}^k\|\EE(\breve A_i^{\top}\breve A_i)\| = \|\EE(\breve A_1^{\top}\breve A_1)\|^k <1.$ This gives the desired result for the case with $q=2$.

For $q>2$, we note that
\begin{align}
    \EE\|\breve\bbeta_k(\alpha) - \breve\bbeta_k'(\alpha)\|_2^q  = \EE\|\breve A_k\cdots\breve A_1\breve{\bm{\delta}}\|_2^q = \EE\big(\|\breve A_k\cdots\breve A_1\breve{\bm{\delta}}\|_2^2\big)^{q/2} = \EE\big(\breve{\bm{\delta}}^{\top}\breve A_1^{\top}\cdots\breve A_k^{\top}\breve A_k\cdots\breve A_1\breve{\bm{\delta}}\big)^{q/2}.
\end{align}
Similarly, it follows from the tower rule that
\begin{align}
    \EE\big(\breve{\bm{\delta}}^{\top}\breve A_1^{\top}\cdots\breve A_k^{\top}\breve A_k\cdots\breve A_1\breve{\bm{\delta}}\big)^{q/2} & = \EE\big[ \EE\big[\big(\breve{\bm{\delta}}^{\top}\breve A_1^{\top}\cdots\breve A_k^{\top}\breve A_k\cdots\breve A_1\breve{\bm{\delta}}\big)^{q/2} \mid \breve A_1,\ldots, \breve A_{k-1}\big] \big] \nonumber \\
    & \le \sup_{\bm{v}\in\RR^d,\|\bm{v}\|_2=1}\EE\|\breve A_k\bm{v}\|_2^q \cdot \EE\big(\breve{\bm{\delta}}^{\top}\breve A_1^{\top}\cdots\breve A_{k-1}^{\top}\breve A_{k-1}\cdots\breve A_1\breve{\bm{\delta}}\big)^{q/2} \nonumber \\
    & \le \prod_{i=1}^k\sup_{\bm{v}\in\RR^d,\|\bm{v}\|_2=1}\EE\|\breve A_i\bm{v}\|_2^q\cdot\|\breve{\bm{\delta}}\|_2^q.
\end{align}
Since $\sup_{\bm{v}\in\RR^d,\|\bm{v}\|_2=1}\EE\|\breve A_i\bm{v}\|_2^q<1$ holds uniformly over $i\in\NN$, we obtain the geometric-moment contraction in (\ref{eq_sgd_gmc}) for $q>2$.

Finally, by Corollary~\ref{cor_gmc_mat},  the geometric-moment contraction in (\ref{eq_sgd_gmc}) implies the existence of a unique stationary distribution $\breve\pi_{\alpha}$ of the SGD dropout $\{\breve\bbeta_k(\alpha)\}_{k\in\NN}$. This completes the proof.
\end{proof}

\subsection{Proofs of Lemmas~\ref{lemma_sgd_l2_closed_form}--\ref{lemma_true_l2}}
\begin{lemma}[Closed-form solution of the $\ell^2$ minimizer]
\label{lemma_sgd_l2_closed_form}
    Assume that model (\ref{eq_model_sgd}) is in a reduced form, i.e., $\min_i\EE[\bm{x}_1\bm{x}_1^{\top}]_{ii}>0$. Then, for the minimizer of the $\ell^2$-regularized least-squares loss $\breve\bbeta:=\arg\min_{\bbeta\in\RR^d}\EE\big[(y - \bm{x}^{\top}D\bbeta)^2/2\big]$ as defined in (\ref{eq_goal_sgd}), we have the closed form solution
\begin{equation*}
    \breve\bbeta = (\EE[\XX_{1,p}])^{-1}\EE[y_1\bm{x}_1].
\end{equation*}
\end{lemma}

\begin{proof}[Proof of Lemma~\ref{lemma_sgd_l2_closed_form}]
Recall the $d\times d$ Gram matrix $\XX_k = \bm{x}_k\bm{x}_k^{\top}$ and
\begin{equation*}
    \overline{\XX}_k = \XX_k - \Diag(\XX_k), \quad \XX_{k,p} = p\XX_k + (1-p)\Diag(\XX_k).
\end{equation*}
The closed form solution can be obtained by first computing the gradient of the $\ell^2$-regularized least-squares loss,
\begin{align*}
    (y-\bm{x}^{\top}\bbeta)^2 & = y^2 - 2y\bm{x}^{\top}D\bbeta + \bbeta^{\top}D(\bm{x}\bm{x}^{\top})D\bbeta, \\
    \EE_D[(y-\bm{x}^{\top}\bbeta)^2] & = y^2 - 2py\bm{x}^{\top}\bbeta + p^2\bbeta^{\top}(\bm{x}\bm{x}^{\top})\bbeta + p(1-p)\bbeta^{\top}\mathrm{Diag}(\bm{x}\bm{x}^{\top})\bbeta, \\
    \EE_{(y,\bm{x})}\EE_D[(y-\bm{x}^{\top}\bbeta)^2] & = \EE[y^2] - 2p\EE[y\bm{x}^{\top}]\bbeta + p^2\bbeta^{\top}\EE[\bm{x}\bm{x}^{\top}]\bbeta + p(1-p)\bbeta^{\top}\mathrm{Diag}(\EE[\bm{x}\bm{x}^{\top}])\bbeta, \\
    \nabla_{\bbeta}\EE_{(y,\bm{x})}\EE_D[(y-\bm{x}^{\top}\bbeta)^2] & = - 2p\EE[y\bm{x}] + 2\Big(p^2\EE[\bm{x}\bm{x}^{\top}] + p(1-p)\mathrm{Diag}(\EE[\bm{x}\bm{x}^{\top}])\Big)\bbeta.
\end{align*}
Recall that the i.i.d.\ random noise $\epsilon_k$ is independent of the i.i.d.\ random covariates $\bm{x}_k$. Since model (\ref{eq_model_sgd}) is assumed to be in a reduced form, i.e., $\min_i\EE[\bm{x}_1\bm{x}_1^{\top}]_{ii}>0$, the closed form solution of $\breve\bbeta$ is
\begin{equation*}
    \breve\bbeta = p\Big(p^2\EE[\bm{x}_1\bm{x}_1^{\top}] + p(1-p)\mathrm{Diag}(\EE[\bm{x}_1\bm{x}_1^{\top}])\Big)^{-1}\EE[y_1\bm{x}_1] = (\EE[\XX_{1,p}])^{-1}\EE[y_1\bm{x}_1].
\end{equation*}
\end{proof}

Recall that for any $d\times d$ matrix $A$, $A_p:=pA+(1-p)\mathrm{Diag}(A)$.

\begin{lemma}\label{lemma_L2_convergence}
If the $d\times d$ matrix $\EE[2\XX_k-\alpha\XX_k^2]_p$ is positive definite, then the condition on the learning rate $\alpha$ in (\ref{eq_sgd_condition}) holds for $q=2$.
\end{lemma}
\begin{proof}
By rewriting the condition (\ref{eq_sgd_condition}) with $q=2$, for all the unit vector $\bm{v}\in\RR^d,$ $\|\bm{v}\|_2=1$, we aim to show
\begin{align}
    0 < \alpha < \frac{2\bm{v}^{\top}\EE(D_k\XX_kD_k)\bm{v}}{\EE\|D_k\XX_kD_k\bm{v}\|_2^2}.
\end{align}
Since $D_k^2=D_k\le I_d$, it follows from Lemma~\ref{lemma_clara}(ii) that
\begin{align*}
    & \quad 2\bm{v}^{\top}\EE(D_k\XX_kD_k)\bm{v} - \alpha\bm{v}^{\top}\EE[D_k\XX_kD_k\XX_kD_k]\bm{v} \nonumber \\
    & \ge 2\bm{v}^{\top}\EE(D_k\XX_kD_k)\bm{v} - \alpha\bm{v}^{\top}\EE[D_k\XX_k^2D_k]\bm{v} \nonumber \\
    & = \bm{v}^{\top}\EE[D_k(2\XX_k-\alpha\XX_k^2)D_k]\bm{v} \nonumber \\
    & = p\bm{v}^{\top}\EE[2\XX_k-\alpha\XX_k^2]_p\bm{v} >0.
\end{align*}
As the unit vector $\bm{v}\in\RR^d$ was arbitrary, condition (\ref{eq_sgd_condition}) holds for $q=2$. 
\end{proof}

\begin{lemma}[$\ell^2$-minimizer $\breve\bbeta$ and true parameter $\bbeta^*$]\label{lemma_true_l2}
Assume that $\EE[|\epsilon|^{2q}]+\|\bm{x}\|_2^{2q}]<\infty$. Then, the $q$-th moment of the gradient in (\ref{eq_dropout_sgd_origin}) exists at the true parameter $\bbeta^*$ in model (\ref{eq_model_sgd}), for some $q\ge2$, that is, 
$$\Big(\EE\Big\|\nabla_{\bbeta^*}\frac{1}{2}\big(y-\bm{x}^{\top} D\bbeta^*\big)^2\Big\|_2^q\Big)^{1/q} = \Big(\EE\big\|D\bm{x}\big(y-\bm{x}^{\top}D\bbeta^*\big)\big\|_2^q\Big)^{1/q}<\infty,$$
which further implies the finite $q$-th moment of the stochastic gradient at the $\ell^2$-minimizer $\breve\bbeta$ defined in (\ref{eq_goal_sgd}), that is
$$\Big(\EE\Big\|\nabla_{\breve\bbeta}\frac{1}{2}\big(y-\bm{x}^{\top} D\breve\bbeta\big)^2\Big\|_2^q\Big)^{1/q} = \Big(\EE\big\|D\bm{x}\big(y-\bm{x}^{\top}D\breve\bbeta\big)\big\|_2^q\Big)^{1/q}<\infty,$$
\end{lemma}

\begin{proof}[Proof of Lemma~\ref{lemma_true_l2}]
First, it follows from the triangle inequality that
\begin{align}
    \label{eq_stochastic_gradient_3parts}
    \Big(\EE\big\|D\bm{x}\big(y-\bm{x}^{\top}D\bbeta^*\big)\big\|_2^q\Big)^{1/q} & = \Big(\EE\big\|D\bm{x}\big(\bm{x}^{\top}\bbeta^* + \epsilon -\bm{x}^{\top}D\bbeta^*\big)\big\|_2^q\Big)^{1/q} \nonumber \\
    & \le \big(\EE\big\|D\bm{x}\bm{x}^{\top}\bbeta^*\big\|_2^q\big)^{1/q} + \big(\EE\big\|D\bm{x}\epsilon\big\|_2^q\big)^{1/q} + \big(\EE\big\|D\bm{x}\bm{x}^{\top}D\bbeta^*\big\|_2^q\big)^{1/q}.
\end{align}
By Assumption~\ref{asm_fclt_asgd}, since the dimension of $\bbeta^*$ is fixed, we have 
$$\big(\EE\big\|D\bm{x}\bm{x}^{\top}\bbeta^*\big\|_2^q\big)^{1/q} \le \big(\EE\|\bm{x}\|_2^{2q}\big)^{1/q}\|\bbeta^*\|_2<\infty.$$ 
Due the independence between $\bm{x}$ and $\epsilon$, Assumption~\ref{asm_fclt_asgd} gives
$\big(\EE\big\|D\bm{x}\epsilon\big\|_2^q\big)^{1/q} \le (\EE\|\bm{x}\|_2^q)^{1/q}(\EE\|\epsilon\|_2^q)^{1/q}<\infty.$ Moreover, we obtain
\begin{align}
    \big(\EE\big\|D\bm{x}\bm{x}^{\top}D\bbeta^*\big\|_2^q\big)^{1/q} = \big(\EE\big\|D\bm{x}\big\|_2^{2q}\big)^{1/q}\|\bbeta^*\|_2 \le \big(\EE\big\|\bm{x}\big\|_2^{2q}\big)^{1/q}\|\bbeta^*\|_2 <\infty.
\end{align}
Inserting the inequalities into (\ref{eq_stochastic_gradient_3parts}), we obtain the finite $q$-th moment at the true parameter $\bbeta^*$.

Next, we show that the finite $q$-th moment of the stochastic gradient at $\bbeta^*$ can also imply the finite $q$-th moment at $\breve\bbeta$. Note that
\begin{align}
    \Big(\EE\big\|D_1\bm{x}_1\big(y_1-\bm{x}_1^{\top}D_1\breve\bbeta\big)\big\|_2^q\Big)^{1/q} & \le \Big(\EE\big\|D_1\bm{x}_1\big(y_1-\bm{x}_1^{\top}D_1\bbeta^*\big)\big\|_2^q\Big)^{1/q} + \Big(\EE\big\|D_1\XX_1D_1(\breve\bbeta-\bbeta^*\big)\big\|_2^q\Big)^{1/q}.
\end{align}
We only need to show that the second term is bounded. Since $\XX_{1,p}=p\XX_1 + (1-p)\mathrm{Diag}(\XX_1)$, $\overline{\XX}_1 = \XX_1 - \mathrm{Diag}(\XX_1)$, and $\breve\bbeta=(\EE[\XX_{1,p}])^{-1}\EE[y_1\bm{x}_1]$, it follows that
\begin{align}
    \label{eq_diff_l2_and_true}
    \EE\big\|D_1\XX_1D_1(\breve\bbeta-\bbeta^*\big)\big\|_2^q & = \EE\big\|D_1\XX_1D_1\big((\EE[\XX_{1,p}])^{-1}\EE[y_1\bm{x}_1]-\bbeta^*\big)\big\|_2^q \nonumber \\
    & = \EE\big\|D_1\XX_1D_1\big((\EE[\XX_{1,p}])^{-1}\EE[(\bm{x}_1^{\top}\bbeta^* + \epsilon_1)\bm{x}_1]-\bbeta^*\big)\big\|_2^q \nonumber \\
    & = \EE\big\|D_1\XX_1D_1\big((\EE[\XX_{1,p}])^{-1}\EE[\XX_1]\bbeta^* -\bbeta^*\big)\big\|_2^q \nonumber \\
    & = (1-p)^q\EE\big\|D_1\XX_1D_1(\EE[\XX_{1,p}])^{-1}\EE[\overline{\XX}_1]\bbeta^*\big\|_2^q \nonumber \\
    & \le (1-p)^q\big\|(\EE[\XX_{1,p}])^{-1}\EE[\overline{\XX}_1]\big\|^q \sup_{\bm{v}\in\RR^d,\|\bm{v}\|_2=1} \EE\big\|D_1\XX_1D_1\bm{v}\big\|_2^q\cdot\|\bbeta^*\|_2^q.
\end{align}
The sub-multiplicativity of the operator norm yields
\begin{align}
    \big\|(\EE[\XX_{1,p}])^{-1}\EE[\overline{\XX}_1]\big\| \le \frac{\lambda_{\max}(\EE[\overline{\XX}_1])}{\lambda_{\min}(\EE[\XX_{1,p}])} <\infty,
\end{align}
since $\lambda_{\min}(\EE[\XX_{1,p}])\ge (1-p)\lambda_{\min}(\EE[\mathrm{Diag}(\XX_1)]) = (1-p)\min_i\EE[\XX_1]_{ii} >0$, and $\lambda_{\max}(\EE[\overline{\XX}_1])\le\lambda_{\max}(\EE[\XX_1])<\infty$. Moreover, $\|\bbeta^*\|_2<\infty.$ As we assumed that $\sup_{\bm{v}\in\RR^d,\|\bm{v}\|_2=1}\EE\|D_1\XX_1D_1\bm{v}\|_2^q<\infty$ in Lemma~\ref{lemma_gmc_sgd_cond}, also (\ref{eq_diff_l2_and_true}) is bounded.
\end{proof}

\section{Proofs in Section~\ref{subsec_asymp_sgd}}

\subsection{Proof of Lemma~\ref{lemma_q_moment_sgd}}
\begin{proof}[Proof of Lemma~\ref{lemma_q_moment_sgd}]
The recursion in (\ref{eq_dropout_sgd}) is $\breve\bbeta_k(\alpha) - \breve\bbeta = \breve A_k(\alpha)(\breve\bbeta_{k-1}(\alpha) - \breve\bbeta) + \breve{\bm{b}}_k(\alpha)$, with random matrix $\breve A_k(\alpha) = I_d - \alpha D_k\XX_kD_k$, and random vector $\breve{\bm{b}}_k(\alpha) = \alpha D_k\bm{x}_k(y_k - \bm{x}_k^{\top}D_k\breve\bbeta)$. Recall $\breve\bbeta=\arg\min_{\bbeta\in\RR^d}\EE[(y - \bm{x}^{\top}D\bbeta)^2/2]$ in (\ref{eq_goal_sgd}), where the expectation is taken over both $(y,\bm{x})$ and $D$. By Lemma~\ref{lemma_clara}~(ii), 
\begin{align}
    \label{eq_bk_mean0_sgd}
    \EE_{(y,\bm{x})}\EE_D[\breve{\bm{b}}_k(\alpha)] & = \EE_{(y,\bm{x})}\EE_D[\alpha D_k\bm{x}_k(y_k - \bm{x}_k^{\top}D_k\breve\bbeta)] \nonumber \\
    & = \EE_{(y,\bm{x})}[\alpha pI_dy_k\bm{x}_k - \alpha p\XX_{k,p}\breve\bbeta] \nonumber \\
    & = \alpha p\EE[\XX_{1,p}]\breve\bbeta - \alpha p \EE[\XX_{1,p}]\breve\bbeta \nonumber \\
    & = 0.
\end{align}
Similar to (\ref{eq_iter_dropout_stationary}), we can rewrite the stationary SGD dropout sequence $\breve\bbeta_k^{\circ}(\alpha)$ into
\begin{align}
    \breve\bbeta_k^{\circ}(\alpha) - \breve\bbeta & = \sum_{i=0}^{\infty}\Big(\prod_{j=k-i+1}^k\breve A_j(\alpha)\Big)\breve{\bm{b}}_{k-i}(\alpha) \nonumber \\
    & = \alpha\sum_{i=0}^{\infty}\Big(\prod_{j=k-i+1}^k(I_d-\alpha D_j\XX_jD_j)\Big)D_{k-i}\bm{x}_{k-i}(y_{k-i}-\bm{x}_{k-i}^{\top}D_{k-i}\breve\bbeta) \nonumber \\
    & =: \alpha\sum_{i=0}^{\infty}\breve\M_{i,k}(\alpha).
\end{align}
Recall the filtration $\breve\F_i=\sigma(\bm{\xi}_i,\bm{\xi}_{i-1},\ldots)$ in (\ref{eq_filtration_sgd}) for $i\in\ZZ$, where $\bm{\xi}_i=(y_i,\bm{x}_i,D_i)$. Notice that  $\EE[\breve{\bm{b}}_k(\alpha)]=0$ by (\ref{eq_bk_mean0_sgd}), and therefore we have
\begin{align}
    \EE[\breve\M_{i,k}(\alpha) \mid \breve\F_{k-i+1}] & = \EE\Big[\prod_{j=k-i+1}^k(I_d-\alpha D_j\XX_jD_j) \Bmid \breve\F_{k-i+1}\Big]  \EE\Big[D_{k-i}\bm{x}_{k-i}(y_{k-i}-\bm{x}_{k-i}^{\top}D_{k-i}\breve\bbeta)\Big] \nonumber \\
    & = \EE\Big[\prod_{j=k-i+1}^k(I_d-\alpha D_j\XX_jD_j) \Bmid \breve\F_{k-i+1}\Big]\cdot\EE[\breve{\bm{b}}_{k-i}(\alpha)] \nonumber \\
    & = 0.
\end{align}
Hence, for any $k\in\NN$, $\{\breve\M_{i,k}(\alpha)\}_{i\in\NN}$ is a sequence of martingale differences with respect to the filtration $\breve\F_{k-i}$. Let $t=k-i$. By applying Burkholder's inequality in Lemma~\ref{lemma_burkholder}, we have,
\begin{align}
    \label{eq_sgd_rate_sum_MD}
    \big(\EE\|\breve\bbeta_k^{\circ}(\alpha) - \breve\bbeta\|_2^q\big)^{1/q} & = \alpha\Big(\EE\Big\|\sum_{t=-\infty}^k\breve\M_{k-t,k}(\alpha)\Big\|_2^q\Big)^{1/q} \nonumber \\
    & \lesssim \alpha\Big(\sum_{t=-\infty}^k\big(\EE\|\breve\M_{k-t,k}(\alpha)\|_2^q\big)^{2/q}\Big)^{1/2},
\end{align}
where the constant in $\lesssim$ only depends on $q$. Denote the vector $\breve{\bm{s}}_t=D_t\bm{x}_t(y_t-\bm{x}_t^{\top}D_t\breve\bbeta)$ and the matrix product $\breve A_{(t+1):k}=\breve A_{(t+1):k}(\alpha)=\prod_{j=t+1}^k\breve A_j(\alpha)$ for simplicity. Then, we can write
\begin{align}
    \breve{\bm{b}}_t(\alpha) = \alpha\breve{\bm{s}}_t \quad \text{and} \quad \breve\M_{k-t,k}(\alpha)= \breve A_{(t+1):k}(\alpha)\breve{\bm{s}}_t.
\end{align}

For the case with $q=2$, notice that $\breve A_{(t+1):k}$ is independent of $\breve{\bm{s}}_t$, and by the tower rule, we have
\begin{align}
    \label{eq_sgd_rate_MD_q2}
    \EE\|\breve\M_{k-t,k}(\alpha)\|_2^2 & = \EE\|\breve A_{(t+1):k}\breve{\bm{s}}_t\|_2^2 \nonumber \\
    & = \EE\big[\EE\big[\breve{\bm{s}}_t^{\top}\breve A_{(t+1):k}^{\top}\breve A_{(t+1):k}\breve{\bm{s}}_t \mid \breve\F_t \big] \big] \nonumber \\
    & = \EE\big[\EE\big[\mathrm{tr}(\breve{\bm{s}}_t\breve{\bm{s}}_t^{\top}\breve A_{(t+1):k}^{\top}\breve A_{(t+1):k}) \mid \breve\F_t\big]\big] \nonumber \\
    & = \EE\big[\mathrm{tr}\big(\EE[\breve{\bm{s}}_t\breve{\bm{s}}_t^{\top}\breve A_{(t+1):k}^{\top}\breve A_{(t+1):k} \mid \breve\F_t]\big) \big] \nonumber \\
    & = \EE\big[\mathrm{tr}\big(\EE[\breve{\bm{s}}_t\breve{\bm{s}}_t^{\top}]\breve A_{(t+1):k}^{\top}\breve A_{(t+1):k}\big) \big] \nonumber \\
    & = \mathrm{tr}\big( \EE[\breve{\bm{s}}_t\breve{\bm{s}}_t^{\top}]\EE[\breve A_{(t+1):k}^{\top}\breve A_{(t+1):k}] \big) \nonumber \\
    & \le \big\|\EE[\breve A_{(t+1):k}^{\top}\breve A_{(t+1):k}]\big\| \cdot \EE\|\breve{\bm{s}}_t\|_2^2.
\end{align}
Next, we shall bound the parts $\big\|\EE[\breve A_{(t+1):k}^{\top}\breve A_{(t+1):k}]\big\|$ and $\EE\|\breve{\bm{s}}_t\|_2^2$ separately. First, recall $\breve A_j(\alpha)=I_d-\alpha D_j\XX_jD_j$ in (\ref{eq_dropout_sgd}), which are i.i.d.\ over $j$. By the tower rule with the induction over $j=t+1,t+2,\ldots,k$, we have
\begin{align}
    \big\|\EE[\breve A_{(t+1):k}^{\top}\breve A_{(t+1):k}]\big\| & = \big\|\EE\big[ \EE[\breve A_{(t+1):k}^{\top}\breve A_{(t+1):k} \mid \breve A_{t+1},\breve A_{t+2},\ldots,\breve A_{k-1}] \big]\big\| \nonumber \\
    & \le \big\|\EE[\breve A_k^{\top}(\alpha)\breve A_k(\alpha)]\big\| \cdot \big\|\EE[\breve A_{(t+1):(k-1)}^{\top}\breve A_{(t+1):(k-1)}]\big\| \nonumber \\
    & \le \prod_{j=t+1}^k\big\|\EE[\breve A_j^{\top}(\alpha)\breve A_j(\alpha)]\big\| \nonumber \\
    & = \big\|\EE[\breve A_1^{\top}(\alpha)\breve A_1(\alpha)]\big\|^{k-t}.
\end{align}
Moreover, recall the random matrix $M_k(\alpha)=2D_k\XX_kD_k-\alpha D_k\XX_kD_k\XX_kD_k$ as defined in (\ref{eq_M_k}). For any unit vector $\bm{v}\in\RR^d$, by (\ref{eq_gmc_sgd_constant_q2}) in the proof of Lemma~\ref{lemma_gmc_sgd_cond}, we have
\begin{align}
    \label{eq_sgd_rate_MD_part1_q2}
    \bm{v}^{\top}\EE[\breve A_1^{\top}(\alpha)\breve A_1(\alpha)]\bm{v} & \le 1-\alpha \bm{v}^{\top}\EE[M_1(\alpha)]\bm{v} \nonumber \\
    & \le 1-\alpha \lambda_{\min}\big(\EE[M_1(\alpha)]\big) <1,
\end{align}
as the constant learning rate $\alpha$ satisfies condition~(\ref{eq_sgd_condition}). In fact, condition (\ref{eq_sgd_condition}) also implies that $\EE[M_k(\alpha)]$ is positive definite for each $k\in\NN$, which can be seen by (\ref{eq_sgd_condtion_q2}).

We shall show that the term $\EE\|\breve{\bm{s}}_k\|_2^2 = \EE\|D_k\bm{x}_k(y_k-\bm{x}_k^{\top}D_k\breve\bbeta)\|_2^2$ remains bounded as $k\rightarrow\infty$. In Assumption~\ref{asm_fclt_asgd}, we have assumed that the stochastic gradient in the SGD dropout recursion  $\nabla_{\bbeta}(y-\bm{x}^{\top}D\bbeta)^2/2 = D\bm{x}(y-\bm{x}^{\top}D\bbeta)$, has finite $q$-th moment when $\bbeta=\bbeta^*$ for some $q\ge2$. By Lemma~\ref{lemma_true_l2}, Assumption~\ref{asm_fclt_asgd} also implies the bounded $q$-th moment when $\bbeta=\breve\bbeta$. As a direct consequence, $\EE\|\breve{\bm{s}}_k\|_2^2$ is bounded as $k\rightarrow\infty$. This, along with (\ref{eq_sgd_rate_sum_MD}), (\ref{eq_sgd_rate_MD_q2}) and (\ref{eq_sgd_rate_MD_part1_q2}), yields
\begin{align}
    \big(\EE\|\breve\bbeta_k^{\circ}(\alpha) - \breve\bbeta\|_2^2\big)^{1/2} & \lesssim \alpha\Big(\sum_{t=-\infty}^k\EE\|\breve\M_{k-t,k}(\alpha)\|_2^2\Big)^{1/2} \nonumber \\
    & \le \alpha\Big(\sum_{t=-\infty}^k\big\|\EE[\breve A_{(t+1):k}^{\top}\breve A_{(t+1):k}]\big\| \cdot \EE\|\breve{\bm{s}}_t\|_2^2\Big)^{1/2} \nonumber \\
    & \lesssim \alpha\Big(\sum_{t=-\infty}^k(1-\alpha \lambda_*)^{k-t}\Big)^{1/2} \nonumber \\
    & = \alpha\Big(\sum_{i=0}^{\infty}(1-\alpha \lambda_*)^i\Big)^{1/2} \nonumber \\
    & = O(\sqrt{\alpha}),
\end{align}
where the last equation holds since $\sum_{i=0}^{\infty}(1-\alpha\lambda_*)^i=\frac{1}{1-(1-\alpha\lambda_*)}=O(1/\alpha)$. Here, the constants in $\lesssim$ are independent of $k$ and $\alpha$, and $\lambda_*=\lambda_{\min}\big(\EE[M_1(\alpha)]\big)$ is bounded away from zero since $\EE[M_1(\alpha)]=\EE[2D_1\XX_1D_1-\alpha D_1\XX_1D_1\XX_1D_1]$ is positive definite by condition (\ref{eq_sgd_condtion_q2}).

For the case $q>2$, following similar arguments as in (\ref{eq_sgd_rate_MD_q2}), we obtain
\begin{align}
    \label{eq_sgd_rate_MD}
    \EE\|\breve\M_{k-t,k}(\alpha)\|_2^q & = \EE(\|\breve A_{(t+1):k}\breve{\bm{s}}_t\|_2^2)^{q/2} \nonumber \\
    & = \EE\big(\breve{\bm{s}}_t^{\top}\breve A_{(t+1):k}^{\top}\breve A_{(t+1):k}\breve{\bm{s}}_t\big)^{q/2} \nonumber \\
    & \le \sup_{\bm{v}\in\RR^d,\|\bm{v}\|_2=1}\EE\|\breve A_{(t+1):k}\bm{v}\|_2^q \cdot \EE\|\breve{\bm{s}}_t\|_2^q \nonumber \\
    & \le \big(\sup_{\bm{v}\in\RR^d,\|\bm{v}\|_2=1}\EE\|\breve A_1\bm{v}\|_2^q\big)^{k-t}\cdot \EE\|\breve{\bm{s}}_t\|_2^q.
\end{align}
With $\mu_q(\bm{v})=(\EE\|D_1\XX_1D_1\bm{v}\|_2^q)^{1/q}<\infty$ as defined in Lemma~\ref{lemma_gmc_sgd_cond} and the Equations (\ref{eq_gmc_sgd_constant1}) and (\ref{eq_gmc_sgd_constant2}) in the proof of Lemma~\ref{lemma_gmc_sgd_cond}, we have
\begin{align}
    \EE\|\breve A_1\bm{v}\|_2^q & \le (1+\alpha\mu_q(\bm{v}))^q - q\alpha\mu_q(\bm{v}) - q\alpha p\bm{v}^{\top}\EE[\XX_{1,p}]\bm{v}.
\end{align}
This, together with Taylor expansion around $\alpha=0$ and the inequalities (\ref{eq_sgd_rate_sum_MD}) and (\ref{eq_sgd_rate_MD}) gives finally $\max_k \big(\EE\|\breve\bbeta_k^{\circ}(\alpha) - \breve\bbeta\|_2^q\big)^{1/q}=O(\sqrt{\alpha})$.
\end{proof}

\subsection{Proof of Theorem~\ref{thm_clt_ave_sgd}}

The proofs of the quenched CLT and the invariance principle for averaged SGD dropout follow similar arguments as the ones for averaged GD dropout. The key differences lie in the functional dependence measures, because for SGD settings, both the dropout matrix $D_k$ and the sequential observation $\bm{x}_k$ are random. We shall first introduce some necessary definitions and then proceed with the rigorous proofs.

Recall the generic dropout matrix $D\in\RR^{d\times d}$ and random sample $(y,\bm{x})\in\RR\times\RR^d$. For the SGD dropout sequence $\{\breve\bbeta_k(\alpha)\}_{k\in\NN}$, we define the centering term
\begin{equation}
    \label{eq_center_sgd}
    \breve\bbeta_{\infty}(\alpha) = \lim_{k\rightarrow\infty}\EE[\breve\bbeta_k(\alpha)] = \EE[\breve\bbeta_1^{\circ}(\alpha)],
\end{equation}
where the expectation is taken over both $(y,\bm{x})$ and $D$, and $\breve\bbeta_1^{\circ}(\alpha)$ defined in (\ref{eq_stationary_sgd}) follows the unique stationary distribution $\breve\pi_{\alpha}$. We first use Lemma~\ref{lemma_thm21_shao_wu_2007} to prove the CLT for the partial sum of the stationary sequence $\{\breve\bbeta_k^{\circ}(\alpha)-\breve\bbeta_{\infty}(\alpha)\}_{k\in\NN}$, and then apply the geometric-moment contraction in Theorem~\ref{thm_gmc_sgd} to show the quenched CLT for the partial sum of the non-stationary one $\{\breve\bbeta_k(\alpha)-\breve\bbeta_{\infty}(\alpha)\}_{k\in\NN}$. Finally, we extend the quenched CLT to the partial sum of $\{\breve\bbeta_k(\alpha)-\breve\bbeta_{\infty}(\alpha)\}_{k\in\NN}$ by providing the upper bound of $\breve\bbeta_{\infty}(\alpha)-\breve\bbeta$ in terms of the $q$-th moment for some $q\ge2$.

Similar to Section~\ref{subsec_functional_dependence}, we introduce the \textit{functional dependence measure} in \textcite{wu_nonlinear_2005} for the stationary SGD dropout sequence $\{\breve\bbeta_k^{\circ}(\alpha)\}_{k\in\NN}$. However, the randomness in $\breve\bbeta_k^{\circ}(\alpha)$ is induced from both the dropout matrix $D_k$ and the random sample $(y_k,\bm{x}_k)$. Therefore, we define a new filtration $\breve\F_k$ by
\begin{equation}
    \label{eq_filtration_sgd}
    \breve\F_k = \sigma(\bm{\xi}_k,\bm{\xi}_{k-1},\ldots), \quad k\in\ZZ,
\end{equation}
where the i.i.d. random elements $\bm{\xi}_k=(D_k,(y_k,\bm{x}_k))$, $k\in\ZZ$, are defined in (\ref{eq_sgd_random_pair}). For any random vector $\bm{\zeta}\in\RR^d$ satisfying $\EE\|\bm{\zeta}\|_2<\infty$, define projection operators
\begin{equation}
    \label{eq_projection_operator_sgd}
    \breve\P_k[\bm{\zeta}] = \EE[\bm{\zeta} \mid \breve\F_k] - \EE[\bm{\zeta} \mid \breve\F_{k-1}], \quad k\in\ZZ.
\end{equation}
By Theorem~\ref{thm_gmc_sgd} and (\ref{eq_stationary_sgd}), there exists a measurable function $\breve h_{\alpha}(\cdot)$ such that the stationary SGD dropout sequence $\{\breve\bbeta_k^{\circ}(\alpha)\}_{k\in\NN}$ can be represented by a causal process
\begin{equation}
    \breve\bbeta_k^{\circ}(\alpha) = \breve h_{\alpha}(\bm{\xi}_k,\bm{\xi}_{k-1},\ldots) = \breve h_{\alpha}(\breve\F_k).
\end{equation}
We denote the coupled version of $\breve\F_i$ by 
\begin{equation}
    \label{eq_filtration_couple_sgd}
    \breve\F_{i,\{j\}}=\sigma(\bm{\xi}_i, \ldots, \bm{\xi}_{j+1}, \bm{\xi}_j', \bm{\xi}_{j-1}\ldots),
\end{equation}
and let $\breve\F_{i,\{j\}}=\breve\F_i$ if $j>i$, where $\bm{\xi}_j'$ is an i.i.d.\ copy of $\bm{\xi}_i$. For $q>1$, define the \textit{functional dependence measure} of $\breve\bbeta_k(\alpha)$ as
\begin{equation}
    \label{eq_functional_dep_measure_sgd}
    \breve\theta_{k,q}(\alpha) = \big(\EE\|\breve\bbeta_k^{\circ}(\alpha) - \breve\bbeta_{k,\{0\}}^{\circ}(\alpha)\|_2^q\big)^{1/q}, \quad \text{where }\breve\bbeta_{k,\{0\}}^{\circ}(\alpha) = \breve h_{\alpha}(\breve\F_{k,\{0\}}).
\end{equation}
In addition, if $\sum_{k=0}^{\infty}\breve\theta_{k,q}(\alpha)<\infty$, we define the tail of \textit{cumulative dependence measure} as
\begin{equation}
    \label{eq_cumulative_dep_measure_sgd}
    \breve\Theta_{m,q}(\alpha) = \sum_{k=m}^{\infty}\breve\theta_{k,q}(\alpha), \quad m\in\NN.
\end{equation}
Both $\breve\theta_{k,q}(\alpha)$ and $\breve\Theta_{k,q}(\alpha)$ are useful to study the dependence structure of the stationary SGD dropout iteration $\breve\bbeta_k^{\circ}(\alpha)=f_{D_k,(y_k,\bm{x}_k)}(\breve\bbeta_{k-1}^{\circ}(\alpha))$. To apply Lemma~\ref{lemma_thm21_shao_wu_2007}, we only need to show that the stationary SGD dropout sequence $\{\breve\bbeta_k^{\circ}(\alpha)\}_{k\in\NN}$ is short-range dependent in the sense that  $\breve\Theta_{0,q}(\alpha)<\infty$, for some $q\ge2$.

\begin{proof}[Proof of Theorem~\ref{thm_clt_ave_sgd}]
Consider two initial vectors $\breve\bbeta_0^{\circ},\breve\bbeta_0^{\circ'}$ following the stationary distribution $\breve\pi_{\alpha}$ defined in Theorem~\ref{thm_gmc_sgd}. By the recursion in (\ref{eq_dropout_sgd}), we obtain two stationary SGD dropout sequences $\{\breve\bbeta_k^{\circ}(\alpha)\}_{k\in\NN}$ and $\{\breve\bbeta_k^{\circ'}(\alpha)\}_{k\in\NN}$. It follows from Theorem~\ref{thm_gmc_sgd} that for all $q\ge2$,
\begin{equation}
    \label{eq_thm_asgd_clt_constant}
    \sup_{\breve\bbeta_0^{\circ},\breve\bbeta_0^{\circ'}\in\RR^d,\breve\bbeta_0^{\circ}\neq\breve\bbeta_0^{\circ'}}\frac{\big(\EE\|\breve\bbeta_k^{\circ}(\alpha) - \breve\bbeta_k^{\circ'}(\alpha)\|_2^q\big)^{1/q}}{\|\breve\bbeta_0^{\circ} - \breve\bbeta_0^{\circ'}\|_2} \le \breve r_{\alpha,q}^k, \quad k\in\NN,
\end{equation}
with $\breve r_{\alpha,q}=\big(\sup_{\bm{v}\in\RR^d:\,\|\bm{v}\|_2=1}\EE\|\breve A_1(\alpha)\bm{v}\|_2^q\big)^{1/q}$ as defined in (\ref{eq_sgd_r}) and random matrix $\breve A_1(\alpha)=I_d-\alpha D_1\XX_1D_1$. When the constant learning rate $\alpha>0$ satisfies the condition in (\ref{eq_sgd_condition}), we have $\breve r_{\alpha,q}\in(0,1)$ as shown in Theorem~\ref{thm_gmc_sgd}. Recall the coupled filtration $\breve\F_{i,\{j\}}=\sigma(\bm{\xi}_i, \ldots, \bm{\xi}_{j+1}, \bm{\xi}_j', \bm{\xi}_{j-1}\ldots)$ as defined in (\ref{eq_filtration_couple_sgd}). Then, (\ref{eq_stationary_sgd}) and (\ref{eq_thm_asgd_clt_constant}) show that
\begin{align}
    \label{eq_thm_asgd_clt_function}
    & \quad \big(\EE\|\breve f_{\bm{\xi}_k}\circ\cdots\circ \breve f_{\bm{\xi}_1}(\breve\bbeta_0^{\circ}) - \breve f_{\bm{\xi}_k}\circ\cdots\circ \breve f_{\bm{\xi}_1}(\breve\bbeta_0^{\circ'})\|_2^q\big)^{1/q}\nonumber \\
    & = \big(\EE\|\breve h_{\alpha}(\bm{\xi}_k,\ldots,\bm{\xi}_1,\bm{\xi}_0,\bm{\xi}_{-1},\ldots) - \breve h_{\alpha}(\bm{\xi}_k,\ldots,\bm{\xi}_1,\bm{\xi}_0',\bm{\xi}_{-1}',\ldots)\|_2^q\big)^{1/q} \nonumber \\
    & = \big(\EE\|\breve h_{\alpha}(\breve\F_k) - \breve h_{\alpha}(\breve\F_{k,\{0,-1,\ldots\}})\|_2^q\big)^{1/q} \nonumber \\
    & \le \breve c_q\breve r_{\alpha,q}^k,
\end{align}
for some constant $\breve c_q>0$ that is independent of $k$. Following a similar argument in (\ref{eq_thm_agd_clt_function2}), for all $q\ge2$ and $k\in\NN$, we can bound the functional dependence measure $\breve\theta_{k,q}(\alpha)$ in (\ref{eq_functional_dep_measure_sgd}) as follows
\begin{align}
    \breve\theta_{k,q}(\alpha) & = \big(\EE\|\breve h_{\alpha}(\breve\F_k) - \breve h_{\alpha}(\breve\F_{k,\{0\}})\|_2^q\big)^{1/q} \nonumber \\
    & \le \big(\EE\|\breve h_{\alpha}(\breve\F_k) - \breve h_{\alpha}(\breve\F_{k,\{0,-1,\ldots\}})\|_2^q\big)^{1/q} + \big(\EE\|\breve h_{\alpha}(\breve\F_{k,\{0,-1,\ldots\}}) - \breve h_{\alpha}(\breve\F_{k,\{0\}})\|_2^q\big)^{1/q}\nonumber \\
    & \le \breve c_q' \breve r_{\alpha,q}^k,
\end{align}
for some constant $\breve c_q'>0$ that is independent of $k$. Consequently, the cumulative dependence measure $\breve\Theta_{m,q}(\alpha)$ in (\ref{eq_cumulative_dep_measure_sgd}) is also bounded for all $q\ge2$ and $m\in\NN$, that is,
\begin{equation}
    \breve\Theta_{m,q}(\alpha) = \sum_{k=m}^{\infty}\breve\theta_{k,q}(\alpha) = O(\breve r_{\alpha,q}^m) <\infty.
\end{equation}
The inequality in (\ref{eq_functional_dep_measure_ineq}) derived by \textcite{wu_nonlinear_2005} holds for a general class of functional dependence measures, as long as the inputs of the functional system (i.e., the measurable function $\breve h(\bm{\xi}_k,\bm{\xi}_{k-1},\ldots)$) are i.i.d. elements. Thus, we can apply (\ref{eq_functional_dep_measure_ineq}) to the projection operator $\breve\P_k[\cdot]$ in (\ref{eq_projection_operator_sgd}) and obtain
\begin{align}
    \label{eq_short_range_dep_sgd}
    \sum_{k=0}^{\infty}\big(\EE\|\breve\P_0[\breve\bbeta_k^{\circ}(\alpha)]\|_2^q\big)^{1/q} \le \sum_{k=0}^{\infty}\breve\theta_{k,q}(\alpha)
     = \breve\Theta_{0,q}(\alpha)<\infty,
\end{align}
implying the short-range dependence of the stationary SGD dropout sequence $\{\breve\bbeta_k^{\circ}(\alpha)\}_{k\in\NN}$. Then, it follows from Lemma~\ref{lemma_thm21_shao_wu_2007} that
\begin{align}
    n^{-1/2}\sum_{k=1}^n\big(\breve\bbeta_k^{\circ}(\alpha) - \breve\bbeta_{\infty}(\alpha)\big) \Rightarrow \N(0,\breve\Sigma(\alpha)),
\end{align}
where the long-run covariance matrix $\breve\Sigma(\alpha)$ is defined in Theorem \ref{thm_clt_ave_sgd}. Following similar arguments as in (\ref{eq_thm_agd_clt_diff})--(\ref{eq_thm_agd_clt_diff_rate}), for any initial vector $\breve\bbeta_0\in\RR^d$, we can leverage the geometric-moment contraction in Theorem~\ref{thm_gmc_sgd} and achieve the quenched CLT for the corresponding SGD dropout sequence $\{\breve\bbeta_k(\alpha)\}_{k\in\NN}$, that is,
\begin{equation}
    \label{eq_thm_asgd_clt_center0}
    n^{-1/2}\sum_{k=1}^n\big(\breve\bbeta_k(\alpha) - \breve\bbeta_{\infty}(\alpha)\big) \Rightarrow \N(0,\breve\Sigma(\alpha)).
\end{equation}

Recall the $\ell^2$-minimizer $\breve\bbeta$ in (\ref{eq_goal_sgd}) and the centering term $\breve\bbeta_{\infty}(\alpha) = \lim_{k\rightarrow\infty}\EE(\breve\bbeta_k(\alpha)) = \EE(\breve\bbeta_1^{\circ}(\alpha))$ in (\ref{eq_center_sgd}). We shall prove $\|\sum_{k=1}^n\EE[\breve\bbeta_k(\alpha)-\breve\bbeta]\|_2=o(\sqrt{n})$. For any two initial vectors $\breve\bbeta_0$ and $\breve\bbeta_0^{\circ}$, where $\breve\bbeta_0^{\circ}$ follows the stationary distribution $\breve\pi_{\alpha}$ in Theorem~\ref{thm_gmc_sgd}, while $\breve\bbeta_0$ is an arbitrary initial vector in $\RR^d$, it follows from the triangle inequality that
\begin{align}
    \label{eq_expectation_two_parts_sgd}
    \Big\|\sum_{k=1}^n\EE[\breve\bbeta_k(\alpha) - \breve\bbeta]\Big\|_2 & = \Big\|\EE\Big[\sum_{k=1}^n\big(\breve\bbeta_k(\alpha) - \breve\bbeta_k^{\circ}(\alpha) + \breve\bbeta_k^{\circ}(\alpha)- \breve\bbeta\big) \Big]\Big\|_2 \nonumber \\
    & \le \Big\|\EE\Big[\sum_{k=1}^n\big(\breve\bbeta_k(\alpha) - \breve\bbeta_k^{\circ}(\alpha)\big) \Big]\Big\|_2 + \Big\|\EE\Big[\sum_{k=1}^n\big(\breve\bbeta_k^{\circ}(\alpha) - \breve\bbeta\big)\Big]\Big\|_2 \nonumber \\
    & =: \breve\III_1 + \breve\III_2.
\end{align}
For the part $\breve\III_1$, Jensen's inequality and a similar argument as for (\ref{eq_thm_agd_clt_diff}) yield
\begin{align}
    \breve\III_1 & = \Big\|\EE\Big[\sum_{k=1}^n\big(\breve\bbeta_k(\alpha) - \breve\bbeta_k^{\circ}(\alpha)\big)\Big]\Big\|_2 \nonumber \\
    & \le \Big(\EE\Big\|\sum_{k=1}^n\big(\breve\bbeta_k(\alpha)-\breve\bbeta_k^{\circ}(\alpha)\big)\Big\|_2^2\Big)^{1/2} \nonumber \\
    & \le \Big(\sum_{k=1}^n\breve r_{\alpha,2}^k\Big)\big\|\breve\bbeta_0 - \breve\bbeta_0^{\circ}\big\|_2.
\end{align}
For the part $\breve\III_2$, we recall the random matrix $\breve A_1(\alpha)=I_d - \alpha D_1\XX_1D_1$ in (\ref{eq_dropout_sgd}) with $\XX_1=\bm{x}_1\bm{x}_1^{\top}$. Recall the notation $\XX_{1,p}=p\XX_1+(1-p)\mathrm{Diag}(\XX_1)$ in (\ref{eq_barX_X_kp}). Notice that by Lemma~\ref{lemma_clara}~(ii), we have $\EE_{(y,\bm{x})}\EE_D[\breve A_1(\alpha)] = \EE_{(y,\bm{x})}\EE_D[I_d - \alpha D_1\XX_1D_1]  = \EE_{(y,\bm{x})}[I_d - \alpha p\XX_{1,p}]  = I_d - \alpha p\EE[\XX_{1,p}],$ which along with $\EE[\breve{\bm{b}}_k]=0$ in (\ref{eq_bk_mean0_sgd}) gives $\EE[\breve\bbeta_k^{\circ}(\alpha) - \breve\bbeta] = \big(I_d - \alpha p\EE[\XX_{1,p}]\big)\EE[\breve\bbeta_{k-1}^{\circ}(\alpha) - \breve\bbeta].$ Since $\{\breve\bbeta_k^{\circ}(\alpha)\}_{k\in\NN}$ is stationary and $\EE[\XX_{1,p}]$ is non-singular by the reduced-form condition $\min_i(\EE[\bm{x}_1\bm{x}_1^{\top}])_{ii}>0$ imposed on model (\ref{eq_model_sgd}), it follows that
\begin{equation}
    \label{eq_l2_stationary_mean0_sgd}
    \EE[\breve\bbeta_k^{\circ}(\alpha) - \breve\bbeta] = 0, \quad \text{for all} \ k\in\NN.
\end{equation}
This further yields
\begin{align}
    \breve\III_2 = \Big\|\EE\Big[\sum_{k=1}^n\big(\breve\bbeta_k^{\circ}(\alpha) - \breve\bbeta\big)\Big]\Big\|_2 = \Big\|\sum_{k=1}^n\EE[\breve\bbeta_k^{\circ}(\alpha) - \breve\bbeta]\Big\|_2 = 0.
\end{align}
By applying the results for $\breve\III_1$ and $\breve\III_2$ to (\ref{eq_expectation_two_parts_sgd}), we obtain
\begin{align}
    \label{eq_thm_asgd_clt_expectation_bd}
    \Big\|\sum_{k=1}^n\EE[\breve\bbeta_k(\alpha) - \breve\bbeta]\Big\|_2 \le \Big(\sum_{k=1}^n\breve r_{\alpha,2}^k\Big)\big\|\breve\bbeta_0 - \breve\bbeta_0^{\circ}\big\|_2.
\end{align}
When the constant learning rate $\alpha>0$ satisfies the condition in (\ref{eq_sgd_condition}), $\breve r_{\alpha,2}\in(0,1)$ by Theorem~\ref{thm_gmc_sgd}, and hence (\ref{eq_thm_asgd_clt_expectation_bd}) remains bounded as $n\rightarrow\infty$. By this result and (\ref{eq_thm_asgd_clt_center0}), the desired quenched CLT for the partial sum $\sum_{k=1}^n(\breve\bbeta_k^{\circ}(\alpha) - \breve\bbeta)$ follows.
\end{proof}

\begin{proof}[Proof of Corollary~\ref{cor_clt_asgd_parallel}]
The proof of Corollary~\ref{cor_clt_asgd_parallel} applies the Cramér-Wold device to Theorem~\ref{thm_clt_ave_sgd} and can be derived in the same way as the proof of Corollary~\ref{cor_clt_agd_parallel}. We omit the details here.
\end{proof}

\subsection{Proof of Theorem~\ref{thm_fclt_asgd}}

\begin{proof}[Proof of Theorem~\ref{thm_fclt_asgd}]
Consider a stationary SGD dropout sequence $\{\breve\bbeta_k^{\circ}(\alpha)\}_{k\in\NN}$ following the stationary distribution $\breve\pi_{\alpha}$ in Theorem~\ref{thm_gmc_sgd}. Define the mean-zero stationary partial sum
\begin{equation}
    \breve S_i^{\circ}(\alpha) = \sum_{k=1}^i\big(\breve\bbeta_k^{\circ}(\alpha) - \breve\bbeta_{\infty}\big), \quad i\in\NN.
\end{equation}
Recall that $\breve\bbeta_{\infty}(\alpha)=\EE[\breve\bbeta_1^{\circ}(\alpha)]$ as defined in (\ref{eq_center_sgd}). Due to the stationarity, we have $\EE[\breve\bbeta_k^{\circ}(\alpha) -\breve\bbeta]=0$ for all $k\in\NN$ by (\ref{eq_l2_stationary_mean0_sgd}). This gives
\begin{align}
    \big(\EE\|\breve\bbeta_{\infty}(\alpha) - \breve\bbeta\|_2^q\big)^{1/q} = \big(\EE\|\EE[\breve\bbeta_1^{\circ}(\alpha) - \breve\bbeta]\|_2^q\big)^{1/q}=0.
\end{align}
Since we supposed in Theorem~\ref{thm_fclt_asgd} that Assumption~\ref{asm_fclt_asgd} holds for some $q>2$, it follows from Lemma~\ref{lemma_q_moment_sgd} that, for all $q>2$,
\begin{align}
    \label{eq_fclt_sgd_order}
    \big(\EE\|\breve\bbeta_k^{\circ}(\alpha)-\breve\bbeta_{\infty}(\alpha)\|_2^q\big)^{1/q} & \le  \big(\EE\|\breve\bbeta_k^{\circ}(\alpha) - \breve\bbeta\|_2^q\big)^{1/q} + \big(\EE\|\breve\bbeta_{\infty}(\alpha) - \breve\bbeta\|_2^q\big)^{1/q} = O(\sqrt{\alpha}).
\end{align}
Following a similar argument as for (\ref{eq_fclt_gd_cond1}), we can show that $\breve\bbeta_k^{\circ}(\alpha)-\breve\bbeta_{\infty}(\alpha)$ satisfies condition (i) on the uniform integrability in Lemma~\ref{lemma_thm2_karmakar_wu_2020}. Due to the stationarity of $\{\breve\bbeta_k^{\circ}(\alpha)\}_{k\in\NN}$, condition (ii) in Lemma~\ref{lemma_thm2_karmakar_wu_2020} is not required (see the discussion below Lemma~\ref{lemma_thm2_karmakar_wu_2020} for details). Condition (iii) is also satisfied, since when the constant learning rate $\alpha>0$ satisfies condition (\ref{eq_sgd_condition}), the stationary SGD dropout sequence $\{\breve\bbeta_k^{\circ}(\alpha)\}_{k\in\NN}$ is shown to be short-range dependent by (\ref{eq_short_range_dep_sgd}), i.e., the tail of cumulative dependence measure $\breve\Theta_{m,q}(\alpha) = \sum_{k=m}^{\infty}\breve\theta_{k,q}(\alpha)<\infty$. 

Hence, by Lemma~\ref{lemma_thm2_karmakar_wu_2020}, there exists a (richer) probability space $(\breve\Omega^{\star},\breve\A^{\star},\breve\PP^{\star})$ on which we can define random vectors $\breve\bbeta_k^{\star}$'s with the partial sum process $\breve S_i^{\star}=\sum_{k=1}^i(\breve\bbeta_k^{\star}-\breve\bbeta_{\infty})$, and a Gaussian process $\breve G_i^{\star}=\sum_{k=1}^i\breve{\bm{z}}_k^{\star}$, where $\breve{\bm{z}}_k^{\star}$'s are independent Gaussian random vectors in $\RR^d$ following $\N(0,I_d)$, such that
\begin{equation}
    (\breve S_i^{\star})_{1\le i\le n} \overset{\D}{=} (\breve S_i^{\circ})_{1\le i\le n},
\end{equation}
and
\begin{equation}
    \label{eq_fclt_sgd_GS1}
    \max_{1\le i\le n}\big\|\breve S_i^{\star} - \breve\Sigma^{1/2}(\alpha)\breve G_i^{\star}\big\|_2 = o_{\PP}(n^{1/q}), \quad \text{in }(\breve\Omega^{\star},\breve\A^{\star},\breve\PP^{\star}),
\end{equation}
where the long-run covariance matrix $\breve\Sigma(\alpha)$ is defined in Theorem \ref{thm_clt_ave_sgd}.

Following similar arguments as for (\ref{eq_fclt_S_center})--(\ref{eq_fclt_gd_gmc_approx}), we can leverage the geometric-moment contraction in Theorem~\ref{thm_gmc_sgd} to show the same Gaussian approximation rate, i.e., $o_{\PP}(n^{1/q})$, for the partial sum sequence $\big(\sum_{k=1}^i(\breve\bbeta_k(\alpha)-\breve\bbeta_{\infty}(\alpha))\big)_{1\le i\le n}$, for any arbitrarily fixed initial vector $\breve\bbeta_0\in\RR^d$. Finally, recall the partial sum process $\breve S_i^{\breve\bbeta_0}(\alpha) = \sum_{k=1}^i(\breve\bbeta_k(\alpha)-\breve\bbeta)$. By a similar argument in (\ref{eq_fclt_gd_mean_approx}), the desired Gaussian approximation result for the partial sum process $(\breve S_i^{\breve\bbeta_0}(\alpha))_{1\le i\le n}$ follows.
\end{proof}

\section{Proofs in Section~\ref{sec_online_inference}}

\subsection{Proof of Theorem~\ref{thm_longrun_precision_old}}

\begin{lemma}\label{lemma_square_matrix}
    For any $d\times d$ symmetric matrix $S$, we have $\EE\|S\| \le \sqrt{\mathrm{tr}\EE(S^2)} \le \sqrt{d\|\EE(S^2)\|}.$
\end{lemma}

\begin{proof}
    If $\lambda$ is an eigenvalue of a symmetric matrix $A$, then $\lambda^2$ is an eigenvalue of $A^2$. Denote the $j$-th largest eigenvalue of $A$ by $\lambda_j(A)$, $j=1,\ldots,d$. Then, $\EE\|S\| = \EE\max_{1\le j\le d}\big|\lambda_j(S)\big| = \EE\sqrt{\max_{1\le j\le d}\big[\lambda_j(S)\big]^2} = \EE\sqrt{\max_{1\le j\le d}\lambda_j(S^2)}.$ Since the quadratic matrix $S^2$ is positive semi-definite, it follows from Jensen's inequality that $\EE\sqrt{\max_{1\le j\le d}\lambda_j(S^2)} \le \EE\sqrt{\mathrm{tr}(S^2)} \le \sqrt{\mathrm{tr}\EE(S^2)} \le \sqrt{d\|\EE(S^2)\|}.$
\end{proof}

\begin{proof}[Proof of Theorem~\ref{thm_longrun_precision_old}]
Recall the SGD dropout sequence $\{\breve\bbeta_k(\alpha)\}_{k\in\NN}$ in (\ref{eq_dropout_sgd}), the long-run covariance matrix $\breve\Sigma(\alpha)$ of the averaged SGD dropout iterates in Theorem \ref{thm_clt_ave_sgd}, and the online estimator $\hat\Sigma_n(\alpha)$ in (\ref{eq_longrun_est}). When there is no ambiguity, we omit the dependence on $\alpha$, e.g., $\hat\Sigma_n=\hat\Sigma_n(\alpha)$ and $\breve\Sigma=\breve\Sigma(\alpha)$. We shall bound $\EE\|\hat\Sigma_n-\breve\Sigma\|$. 

Recall the stationary SGD dropout sequence $\{\breve\bbeta_k^{\circ}(\alpha)\}_{k\in\NN}$ in (\ref{eq_dropout_sgd_stationary}), which follows the stationary distribution $\breve\pi_{\alpha}$ in Theorem~\ref{thm_gmc_sgd}. For simplicity, we define $V_n(\alpha)=n\hat\Sigma_n(\alpha)$. By Equation (\ref{eq_longrun_online}), we can write $V_n(\alpha)$ into
\begin{align}
    \label{eq_longrun_est_V}
    V_n(\alpha) & = \Big(\sum_{m = 1}^{\psi(n) - 1} \mathcal{S}_{m}(\alpha)^{\otimes 2} + \mathcal{R}_{n}(\alpha)^{\otimes 2}\Big) + \Big(\sum_{m = 1}^{\psi(n) - 1} |B_{m}|^{2} + |\delta_{\eta}(n)|^{2}\Big) \bar\bbeta_{n}^{\sgd}(\alpha)^{\otimes 2} \nonumber \\
    & \quad - \Big(\sum_{m = 1}^{\psi(n) - 1} |B_{m}| \mathcal{S}_{m}(\alpha) + \delta_{\eta}(n) \mathcal{R}_{n}(\alpha)\Big) \bar\bbeta_{n}^{\sgd}(\alpha)^{\top} \nonumber \\
    & \quad - \bar\bbeta_{n}^{\sgd}(\alpha) \Big(\sum_{m = 1}^{\psi(n) - 1} |B_{m}| \mathcal{S}_{m}(\alpha) + \delta_{\eta}(n) \mathcal{R}_{n}(\alpha)\Big)^{\top}.
\end{align}
Recall the partial sums $\mathcal{S}_{m}(\alpha) = \sum_{k \in B_{m}} \breve\bbeta_k(\alpha)$ and $\mathcal{R}_{n}(\alpha) = \sum_{k = \eta_{\psi(n)}}^{n} \breve\bbeta_k(\alpha)$ in (\ref{eq_longrun_twoparts}). We similarly define
\begin{align}
    V_n^{\circ}(\alpha) & = \sum_{m=1}^{\psi(n)-1}\Big(\sum_{k\in\B_m}\breve\bbeta_k^{\circ}(\alpha)\Big)^{\otimes2} + \Big(\sum_{k=\eta_{\psi(n)}}^n\breve\bbeta_k^{\circ}(\alpha)\Big)^{\otimes2} =: \sum_{m=1}^{\psi(n)-1}\mathcal{S}_m^{\circ}(\alpha)^{\otimes2} + \mathcal{R}_n^{\circ}(\alpha)^{\otimes2}.
\end{align}
By the triangle inequality, we have
\begin{align}
    \label{eq_longrun_est_Sigma_to_V}
    n\EE\|\hat\Sigma_n-\breve\Sigma\| = \EE\|V_n-n\breve\Sigma\| \le \EE\|V_n-V_n^{\circ}\| + \EE\|V_n^{\circ} - n\breve\Sigma\|.
\end{align}
We shall bound these two terms separately.

First, for the term $\EE\|V_n^{\circ}-n\breve\Sigma\|$, we shall use the results in \textcite{xiao_single-pass_2011}. To this end, we need to verify the assumptions on the weak dependence of SGD dropout iterates $\{\breve\bbeta_k\}_{k\in\NN}$ in (\ref{eq_dropout_sgd}) and the growing sizes of blocks $\{B_m\}_{m\in\NN}$ in (\ref{eq_block}). Denote the elements of $d\times d$ matrices $V_n^{\circ}$ and $\breve\Sigma$ respectively by
\begin{align}
    V_n^{\circ} =:(v_{ij,n}^{\circ})_{1\le i,j\le d} \quad \text{and} \quad \breve\Sigma =: (\breve\sigma_{ij})_{1\le i,j\le d}.
\end{align}
Moreover, we write the stationary SGD dropout iterate $\breve\bbeta_k^{\circ}$ in (\ref{eq_dropout_sgd_stationary}) and the $\ell^2$-minimizer $\breve\bbeta$ in (\ref{eq_goal_sgd}) respectively into
\begin{align}
    \breve\bbeta_k^{\circ} = (\breve\beta_{k1}^{\circ},\ldots,\breve\beta_{kd}^{\circ})^{\top} \quad \text{and} \quad \breve\bbeta = (\breve\beta_1,\ldots,\breve\beta_d)^{\top}.
\end{align}
By the short-range dependence of $\{\breve\bbeta_k^{\circ}\}_{k\in\NN}$ in (\ref{eq_short_range_dep_sgd}), it can be shown that for any $1\le i,j\le d$, $\mathrm{Cov}(\breve\beta_{ki}^{\circ},\breve\beta_{0j}^{\circ}) \le c_{ij}\rho_{ij}^k$ for some constants $c_{ij}>0$ and $0\le\rho_{ij}<1$. For the non-overlapping blocks $\{B_m\}_{m\in\NN}$, since the positive integers $\{\eta_m\}_{m\in\NN}$ satisfy $\eta_{m+1}-\eta_m\rightarrow\infty$, it follows that $\eta_{m+1}/\eta_m\rightarrow1$ as $m\rightarrow\infty$, and therefore,
\begin{equation}
    \sum_{m=1}^M(\eta_{m+1}-\eta_m)^2 \asymp \eta_{M+1}(\eta_{M+1}-\eta_M).
\end{equation}
Then, by Theorems~1(i)~and~2(iii) in \textcite{xiao_single-pass_2011}, we obtain, for each $1\le j\le d$,
\begin{align}
    \label{eq_longrun_est_diag_rate}
    \EE(v_{jj,n}^{\circ} - n\breve\sigma_{jj})^2 & \le \Big\{\big[\EE(v_{jj,n}^{\circ} - \EE v_{jj,n}^{\circ})^2\big]^{1/2} + \big[(\EE v_{jj,n}^{\circ} - n\breve\sigma_{jj})^2\big]^{1/2}\Big\}^2 \nonumber \\
    & \lesssim n^{(2/\zeta)\vee(2-1/\zeta)}.
\end{align}
For $\EE(v_{ij,n}^{\circ} - n\breve\sigma_{ij})^2$ with $i\neq j$, the same rate as in (\ref{eq_longrun_est_diag_rate}) holds. To see this, define two new sequences $\{\beta_k^+\}_{k\in\NN}$ and $\{\beta_k^-\}_{k\in\NN}$ with $\beta_k^+ = (\breve\beta_{ki}^{\circ} + \breve\beta_{0j}^{\circ}) - (\breve\beta_i + \breve\beta_j)$, and $\beta_k^- = (\breve\beta_{ki}^{\circ} - \breve\beta_{0j}^{\circ}) - (\breve\beta_i - \breve\beta_j)$. Notice that
\begin{align}
    \breve\sigma_{ij} & = \sum_{k=-\infty}^{\infty}\EE(\breve\beta_{ki}^{\circ} - \breve\beta_i)(\breve\beta_{0j}^{\circ} - \breve\beta_j) \nonumber \\
    & = \sum_{k=-\infty}^{\infty}\EE\frac{[(\breve\beta_{ki}^{\circ} - \breve\beta_i) + (\breve\beta_{0j}^{\circ} - \breve\beta_j)]^2 - [(\breve\beta_{ki}^{\circ} - \breve\beta_i) - (\breve\beta_{0j}^{\circ} - \breve\beta_j)]^2}{4} \nonumber \\
    & = \frac{1}{4}\sum_{k=-\infty}^{\infty}\EE(\beta_k^+)^2 - \frac{1}{4}\sum_{k=-\infty}^{\infty}\EE(\beta_k^-)^2,
\end{align}
which can be viewed as the long-run variances of the sequences $\{\beta_k^+\}_{k\in\NN}$ and $\{\beta_k^+\}_{k\in\NN}$ as indicated by the last line. A similar decomposition can be applied to $v_{ij,n}^{\circ}$. Since the results in \textcite{xiao_single-pass_2011} hold for any linear combination of weak-dependent sequences, again by the short-range dependence of $\{\breve\bbeta_k^{\circ}\}_{k\in\NN}$ in (\ref{eq_short_range_dep_sgd}), we have
\begin{align}
    \label{eq_longrun_est_offdiag_rate}
    \EE(v_{ij,n}^{\circ} - n\breve\sigma_{ij})^2 \lesssim n^{(2/\zeta)\vee(2-1/\zeta)}, \quad \text{for all }1\le i, j\le d.
\end{align}
Since the dimension $d$ is fixed, it follows from Lemma~\ref{lemma_square_matrix} that
\begin{align}
    \label{eq_longrun_est_multi_to_1d}
    \EE\|V_n^{\circ}- n\breve\Sigma\| & \le \sqrt{\mathrm{tr}\EE(V_n^{\circ}- n\breve\Sigma)^2} \nonumber \\
    & = \sqrt{\sum_{1\le i,j\le d}\EE(v_{ij,n}^{\circ} - n\breve\sigma_{ij})^2} \nonumber \\
    & \le \sqrt{d^2\max_{1\le i,j\le d}\EE(v_{ij,n}^{\circ} - n\breve\sigma_{ij})^2} \nonumber \\
    & \lesssim n^{(2/\zeta)\vee(2-1/\zeta)},
\end{align}
where the constants in $\lesssim$ are independent of $n$.

Next, we bound $\EE\|V_n - V_n^{\circ}\|$. Similar arguments as for (\ref{eq_longrun_est_multi_to_1d}) show
\begin{align}
    \label{eq_longrun_est_part2_dto1}
    \EE\|V_n - V_n^{\circ}\| \le \sqrt{d^2\max_{1\le i,j\le d}\EE(v_{ij,n} - v_{ij,n}^{\circ})^2}.
\end{align}
Thus, we only need to show the bound for the one-dimensional case. Now, consider $V_n$, $V_n^{\circ}$, $\mathcal{V}_n$, $H_n$ and $\bar\bbeta_n^{\sgd}$ as scalars. Note that $\EE\|\cdot\|=(\EE[\cdot]^2)^{1/2}$ for $d=1$. By the decomposition in (\ref{eq_longrun_online}) and applying Jensen's inequality, we have
\begin{align}
    \label{eq_longrun_est_3parts}
    (\EE[V_n - V_n^{\circ}]^2)^{1/2} & \le (\EE[\mathcal{V}_n - V_n^{\circ}]^2)^{1/2} + 2(\EE[H_n\bar\bbeta_n^{\sgd}]^2)^{1/2} + K_n(\EE[\bar\bbeta_n^{\sgd}]^4)^{1/2}.
\end{align}
We shall bound the three term respectively. First, recall the contraction constant $\breve r_{\alpha,q}$ defined in (\ref{eq_sgd_r}). By the GMC in Theorem~\ref{thm_gmc_sgd} and Hölder's inequality, it follows that
\begin{align}
    (\EE[\mathcal{V}_n - V_n^{\circ}]^2)^{1/2} & \le \sum_{m=1}^{\psi(n)-1}\big(\EE[\mathcal{S}_m - \mathcal{S}_m^{\circ}]^4\big)^{1/4}\big(\EE[\mathcal{S}_m + \mathcal{S}_m^{\circ}]^4\big)^{1/4}  + \big(\EE[\mathcal{R}_n - \mathcal{R}_n^{\circ}]^4\big)^{1/4}\big(\EE[\mathcal{R}_n + \mathcal{R}_n^{\circ}]^4\big)^{1/4} \nonumber \\
    & \lesssim \sum_{m=1}^{\psi(n)-1}(\breve r_{\alpha,q})^{\eta_m}(\eta_{m+1}-\eta_m)^{1/2} + (\breve r_{\alpha,q})^{\eta_{\psi(n)}}(n-\eta_{\psi(n)}+1)^{1/2} \nonumber \\
    & = O(1),
\end{align}
where the constants in $\lesssim$ and $O(\cdot)$ are independent of $n$. Similarly, since the dimension $d$ is fixed and $\EE[\breve\bbeta_k^{\circ}]=\breve\bbeta$ by (\ref{eq_l2_stationary_mean0_sgd}) with the $\ell^2$-minimizer $\breve\bbeta$ defined in (\ref{eq_goal_sgd}), we can show that
\begin{align}
    (\EE[\bar\bbeta_n^{\sgd}]^4)^{1/2}
    \lesssim \Big(\EE\Big[\frac{1}{n}\sum_{k=1}^n\breve\bbeta_k^{\circ}\Big]^4\Big)^{1/2} + \Big(\EE\Big[\frac{1}{n}\sum_{k=1}^n(\breve\bbeta_k - \breve\bbeta_k^{\circ})\Big]^4\Big)^{1/2} \asymp \frac 1 n,
\end{align}
and $(\EE[H_n]^4)^{1/2} \lesssim n(n-\eta_{\psi(n)})^2$. Combining these results with the fact that $K_n\asymp n(n-\eta_{\psi(n)})$ yields
\begin{align}
    K_n(\EE[\bar\bbeta_n^{\sgd}]^4)^{1/2} \asymp n-\eta_{\psi(n)} \quad \text{and} \quad (\EE[H_n\bar\bbeta_n^{\sgd}]^2)^{1/2} \lesssim n-\eta_{\psi(n)}.
\end{align}
Inserting all these expressions back into (\ref{eq_longrun_est_3parts}), we obtain
\begin{align}
    (\EE[V_n-V_n^{\circ}]^2)^{1/2} \lesssim n-\eta_{\psi(n)} \asymp n^{1-(1/\zeta)}.
\end{align}
Consequently, by (\ref{eq_longrun_est_part2_dto1}), we have $\EE\|V_n-V_n^{\circ}\| \lesssim n^{1-(1/\zeta)}$ for the multi-dimensional case. Since $\zeta>1$, compared to the rate of $\EE\|V_n^{\circ}-n\breve\Sigma\|\lesssim n^{(2/\zeta)\vee(2-1/\zeta)}$ in (\ref{eq_longrun_est_multi_to_1d}), the latter dominates. This, along with (\ref{eq_longrun_est_Sigma_to_V}) gives the desired result.
\end{proof}

\end{appendix}

\end{document}